\documentclass[10pt,a4paper]{article}

\usepackage[nocompress]{cite}
\usepackage[pdftex]{graphicx}
\usepackage[small,rm,RM]{subfigure}
\usepackage[cmex10]{amsmath}
\usepackage{array}
\usepackage{caption}
\usepackage{amssymb}
\usepackage{amsthm}
\usepackage{amsmath}
\usepackage{multirow}
\usepackage[table,pdftex]{xcolor}
\usepackage{authblk}
\usepackage{appendix}

\newcommand{\etal}{et al. }
\newcommand{\superscript}[1]{\ensuremath{^{\textrm{#1}}}}

\newtheorem{corollary}{Corollary}
\newtheorem{lemma}{Lemma}
\newtheorem{proposition}{Proposition}
\newtheorem{remark}{Remark}

\makeatletter 
\newtheorem*{rep@theorem}{\rep@title} 
\newcommand{\newreptheorem}[2]{% 
\newenvironment{rep#1}[1]{% 
 \def\rep@title{#2 \ref{##1}}% 
 \begin{rep@theorem}}% 
 {\end{rep@theorem}}} 
\makeatother 

\newreptheorem{lemma}{Lemma} 
\newreptheorem{corollary}{Corollary} 
\newreptheorem{remark}{Remark}

\newcommand{\footnoteremember}[2]{
\footnote{#2}
\newcounter{#1}
\setcounter{#1}{\value{footnote}}
}
\newcommand{\footnoterecall}[1]{
\footnotemark[\value{#1}]
}

\begin{document}

\title{WESD - Weighted Spectral Distance for Measuring Shape Dissimilarity}

\author[1]{Ender~Konukoglu\thanks{Corresponding author: email: ender.konukoglu@gmail.com}}
\author[1]{Ben~Glocker}
\author[1]{Antonio~Criminisi}
\author[2]{Kilian~M.~Pohl}
\affil[1]{Microsoft Research Cambridge, UK}
\affil[2]{University of Pennsylvania, USA}

\maketitle

\begin{abstract}
This article presents a new distance for measuring shape dissimilarity between objects. Recent publications introduced the use of eigenvalues of the Laplace operator as compact shape descriptors. Here, we revisit the eigenvalues to define a proper distance, called Weighted Spectral Distance (WESD), for quantifying shape dissimilarity. The definition of WESD is derived through analysing the heat-trace. This analysis provides the proposed distance an intuitive meaning and mathematically links it to the intrinsic geometry of objects. We analyse the resulting distance definition, present and prove its important theoretical properties. Some of these properties include: i) WESD is defined over the entire sequence of eigenvalues yet it is guaranteed to converge, ii) it is a pseudometric, iii) it is accurately approximated with a finite number of eigenvalues, and iv) it can be mapped to the $[0,1)$ interval. Lastly, experiments conducted on synthetic and real objects are presented. These experiments highlight the practical benefits of WESD for applications in vision and medical image analysis. 
\end{abstract}

% Introduction
\section{Introduction}\label{sec:intro}
Quantifying shape differences between objects is an important task for
various areas in computer science, medical imaging and engineering. In
manufacturing, for example, one may wish to characterize the
difference in shape of two fabricated tools. In radiology, a doctor
frequently diagnoses a disease based on anatomical and pathological
shape changes over time. In computer vision, discriminative shape
models are used for automated object recognition,~\cite{Zhang2004,Iyer2005}.

In order to define measurements of shape dissimilarity, scientists rely on descriptors of objects that capture information on their geometry \cite{Zhang2004}. These descriptors can be in the form of parametrized models (e.g. point clouds, surface patches, space curves, medial axis transforms) or in the form of geometric properties (e.g. volume, surface area to volume ratio, curvature maps). Once a descriptor is formulated the distance between two shapes can be defined as the difference between the associated descriptors. The exact definition of the distance however, is a critical issue. In order to define an intuitive and theoretically sound distance, one should ensure that it takes into account the nature of the descriptor. For instance, the descriptor might be an infinite sequence of positive values, in which case we should be careful not to define a distance that diverges for every non-identical pair of shapes.

Shape descriptors based on the eigensystems of Laplace and
Laplace-Beltrami operators, called {\it spectral signatures}, have
recently gained popularity in computational shape
analysis~\cite{Reuter2006,Levy2006,Rustamov2007,Sun2009,Reuter2009,Bronstein2011}. These descriptors leverage the fact that the eigenvalues and the
eigenfunctions of Laplace operators contain information on the
intrinsic geometry of
objects~\cite{Weyl1912,Kac1966,CourantHilbert}. A visual analogy
useful for an intuitive understanding is to think of an object (e.g.\
in 2D) as the membrane of a drum. In this case the eigenvalues
correspond to the fundamental frequencies of vibration of the membrane
during percussion, and the eigenfunctions correspond to its
fundamental patterns of vibration. Both the eigenvalues and the
eigenfunctions depend on the shape of the drum head and thus can be
used as shape descriptors for the object.

Despite recent progress
by~\cite{Reuter2006,Levy2006,Rustamov2007,Sun2009,Reuter2009,Bronstein2011},
designing meaningful shape distances based on spectral signatures
remains challenging. Difficulties arise from the nature of the
eigensystems. The eigenfunctions of a shape mostly provide {\em
  localized} information on the geometry of small
neighborhoods. Aggregating such local information into an overall
shape dissimilarity measure is non-trivial. On the other hand, the
eigenvalues provide information about the {\em overall} shape, so they
are ideal for defining global distances. However, they form a
diverging sequence making it difficult to define a theoretically sound
metric. Here, we tackle this latter problem and propose a new shape
distance based on the eigenvalues, which is technically sound,
intuitive and practically useful.

In the remainder of this section, we first review in further detail
the literature on spectral signatures and shape distances related to
eigenfunctions and eigenvalues. Then, we provide a brief overview of
our new shape distance.

\subsection{Eigenfunctions}
The eigenfunctions of an object constitute an infinite set of
functions. Each function depends on the shape of the object and is
different than the rest of the set. Figure~\ref{fig:eigensystem}
illustrates this for two example objects where a few eigenfunctions
are shown. The values these functions attain at each point capture the
local geometry around the point, i.e. of its neighborhood. Inspired from this
geometric information, methods define {\it local} shape
signatures~\cite{Levy2006,Jain2006,Rustamov2007,Sun2009} for each
point on an object by evaluating a subset of eigenfunctions at that
specific location. Global shape distances are then defined using such
local signatures. Such distance definitions rely on correspondences. 
These correspondences
should hold both in terms of points and the subset of
eigenfunctions used in the signatures, a condition hard to satisfy in
practice~\cite{Jain2006}. Explicitly searching for such
correspondences leads to expensive
algorithms~\cite{Jain2006,Bronstein2010a,Memoli2010,Bronstein2009,Bronstein2010}. On
the other hand, computing distances between distributions of local
signatures obtained by aggregating all the points, as
in~\cite{Rustamov2007,Sun2009,Lai2010}, might implicitly construct
false correspondences. In summary, defining a global distance based on
local signatures is not an easy task.
\begin{figure}
  \center \subfigure{
    \includegraphics[width=0.70\linewidth]{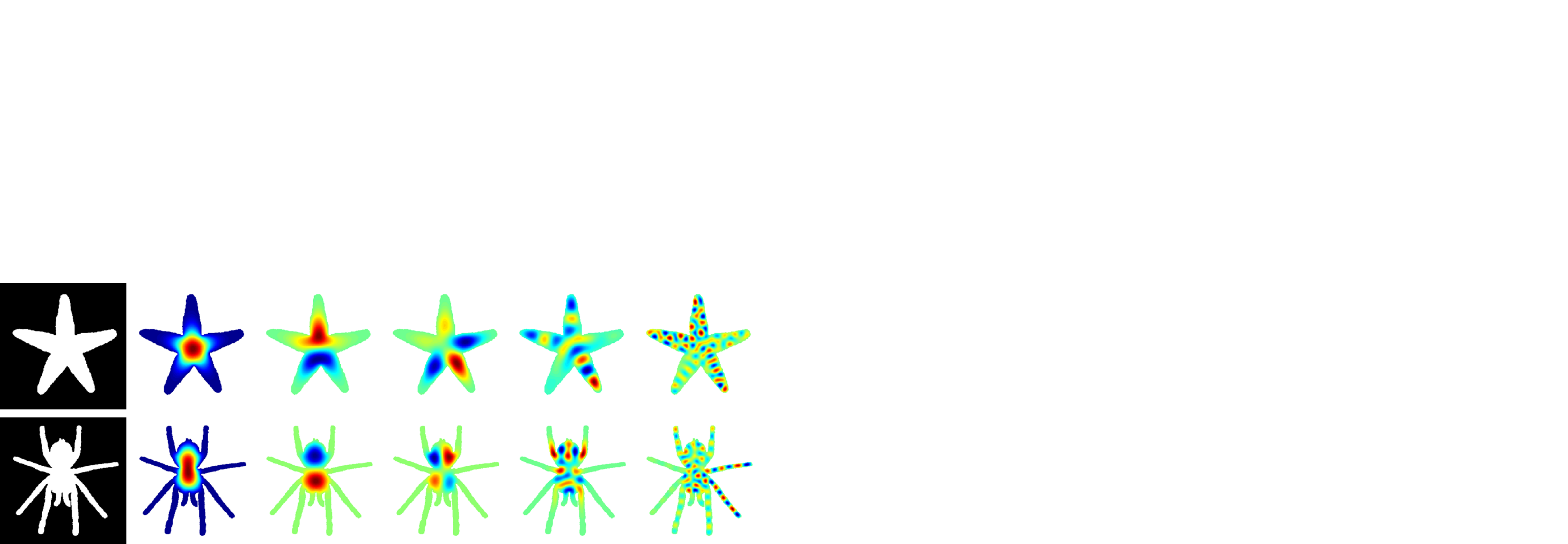}}
  \caption{\label{fig:eigensystem}Starfish and tarantula. The objects
    represented as binary maps are shown on the left, followed by the
    1st, 2nd, 5th, 20th, and 100th eigenfunction. The values increase
    from blue (negative) to red (positive) with green being zero.}
\end{figure}
  
Instead of extracting local information from an eigenfunction, one can
also think of capturing its global pattern by looking at regions where
its values are all positive or all negative. Such regions are called {\em
  nodal domains}. Different eigenfunctions induce different patterns
and, in turn, have different number of nodal domains, called {\em
  nodal counts}~\cite{CourantHilbert}. For a given object, the ordered
sequence of nodal counts contain information on its overall
geometry~\cite{Gnutzmann2005,Gnutzmann2006}. Inspired by these
observations, authors in~\cite{Lai2009} used this sequence as a {\em
  global} shape signature. They further defined the associated shape
distance between two objects as the Euclidean norm of the vector
difference between their nodal count sequences. However, it is not
intuitively clear what the nodal counts represent. Furthermore, the
entire sequence is diverging so that, in practice, one first chooses a
finite subset and then computes the distance for that subset. These
difficulties make it hard to define an intuitive and sound shape
distance based on nodal counts.
\subsection{Eigenvalues}
Signatures based on eigenvalues, on the other hand, have a clearer
geometric interpretation. The set of eigenvalues contains information
on the overall geometry of the object. Specifically, the ordered
sequence is analytically related to the intrinsic geometry by the {\it
  heat-trace},~\cite{Pleijel1954,McKean1967,Smith1981,Protter1987,Vassilevich2003}. Hence,
more intuitive distances can be constructed using the eigenvalues. However, similar to the sequence of nodal counts, the eigenvalue sequence is also divergent. This makes the distance definition theoretically challenging. Inspired by the sequence's link to the geometry, Reuter~\etal in~\cite{Reuter2006}, used the smallest $N$ eigenvalues as a shape signature, called {\em shape-DNA}. As the associated shape distance, the authors proposed the Euclidean norm of the vector difference between the shape-DNAs of objects. Although this is a very good first attempt the divergent nature of eigenvalue sequence results in important theoretical limitations for this distance, as also pointed out in~\cite{Memoli2010}. The main problems are i) defining a distance on the entire sequence does not yield a proper metric, ii) the differences between the higher components of two sequences dominate the final distance value, even though these components do not necessarily provide more information on the geometry, and iii) the distance value is sensitive to the choice of the signature size $N$. These theoretical problems also cause practical drawbacks as we demonstrate later.

This article proposes a new shape distance, called Weighted Spectral Distance (WESD), using the sequence of eigenvalues of the Laplace operator. We derive WESD from the functional relationship between the eigenvalues and the geometric invariants as given by the heat-trace. {This derivation provides WESD a clear geometric intuition as a shape distance. It also links WESD to the distance defined by Reuter~\etal in~\cite{Reuter2006} as well as to the local signature defined in \cite{Rustamov2007}. The resulting formulation of WESD differs from other previously proposed scores based on eigenvalues, whether in shape analysis or other fields \cite{Jurman2010}, both in its formulation and in the fact that it is defined over the entire sequence. This latter point, as we will show later, alleviates the critical importance of the choice of the signature.} We furthermore analyse and prove theoretical properties of WESD showing that it does not share some of the fundamental problems the distance proposed in~\cite{Reuter2006} has. Specifically, we prove that WESD: i) converges despite the fact that it is defined over the entire eigenvalue sequence, ii) can be mapped to the $[0, 1)$ interval, iii) is accurately approximated with a finite number of eigenvalues and the truncation error has an analytical upper bound and iv) is a pseudometric.  These theoretical properties also yield important practical advantages such as being less sensitive to the signature size (truncation parameter) $N$, providing a principled way of choosing this parameter, providing more stable low-dimensional shape embedding and simplicity in combining with other distances as WESD can be normalised. Applying to synthetic and real objects, we further demonstrate the benefits of WESD in comparison to the other eigenvalue-based distance defined in~\cite{Reuter2006}. 

The remainder of this article is structured as follows. Section~\ref{sec:solo} presents a brief overview of the Laplace operator, the eigenvalue sequence and its role in shape analysis. In Section~\ref{sec:wsd_theory} we define WESD and derive its theoretical properties. Section~\ref{sec:wsd_experiments} presents an extensive set of experimental analysis on 2D objects extracted from synthetic binary maps, shape-based retrieval results for 3D objects using the SHREC dataset~\cite{Lian2011}, low dimensional embeddings of real 3D data such as subcortical structures in brain scans and 4D analysis of binary maps extracted from cardiac images.
% Section 2
\section{Spectrum of Laplace Operator}~\label{sec:solo} This section
provides a brief background on the Laplace operator, its eigenvalue
sequence, called {\it spectrum}, and its role in shape analysis. We
first relate an object's intrinsic geometry to the spectrum of the
corresponding Laplace operator. We then provide some details on the
previously proposed shape-DNA~\cite{Reuter2006} and discuss the
associated issues. For further details we refer the reader
to~\cite{CourantHilbert,Protter1987,Vassilevich2003}
and~\cite{Reuter2006}.

We denote an object as a closed bounded domain
$\Omega\subset\mathbb{R}^d$ with piecewise smooth boundaries. In the
case of binary maps, $\Omega$ would correspond to the foreground representing the object. For a given $\Omega$, the
Laplace operator on this object is defined with respect to a twice
differentiable real-valued function $f$ as
\begin{equation}
  \nonumber \Delta_{\Omega} f \triangleq \sum_{i=1}^d \frac{\partial^2}{ \partial x_i^2} f,\  \ \forall \mathbf{x} \in \Omega
\end{equation}
where $\mathbf{x} = [x_1,...,x_d]$ is the spatial coordinate. The
eigenvalues and the eigenfunctions of $\Delta_{\Omega}$ are defined as the solutions of the Helmholtz equation with Dirichlet type boundary
conditions\footnote{Other boundary conditions yield different
  eigensystems. Here we are only interested in the Dirichlet
  type. Please refer to~\cite{CourantHilbert} for the other
  types.},~\cite{CourantHilbert},
\begin{equation}
  \nonumber \Delta f + \lambda f = 0 \ \forall \mathbf{x}\in\Omega, \ f(\mathbf{x}) = 0, \ \forall \mathbf{x} \in \partial\Omega,
\end{equation}
where $\partial\Omega$ denotes the boundary of the object and
$\lambda\in\mathbb{R}$ is a scalar. There are infinitely many pairs
$\{(\lambda_n,f_n)\}_{n=1}^\infty$ satisfying this equation and they
form the set of eigenvalues and eigenfunctions respectively. The
ordered set of eigenvalues is a positive diverging sequence such that
$0<\lambda_1\leq\lambda_2\leq\dots\leq\lambda_n\leq\dots$. This
infinite sequence is called the Dirichlet spectrum of $\Delta_\Omega$,
which we refer simply as the ``spectrum''. In addition, each component
of the spectrum is called a ``mode", e.g. $\lambda_n$ is the called
$n^{th}$ mode of the spectrum

The spectrum contains information on the intrinsic geometry of
objects. Weyl in~\cite{Weyl1912} showed the first spectrum-geometry
link by proving that the asymptotic behavior of the eigenvalues is
given as
\begin{equation}
  \nonumber \lambda_n \sim 4\pi^2\left( \frac{n}{B_d V_{\Omega}} \right)^{2/d},\ n\rightarrow\infty,
\end{equation}
where $V_{\Omega}$ is the volume of $\Omega$ and $B_d$ is the volume
of the unit ball in $\mathbb{R}^d$. {Later works,
as~\cite{Pleijel1954,McKean1967,Smith1981,Protter1987}, extended this result by studying the properties of the Green's function of the Laplace operator, and showed that a more accurate spectrum-geometry link is given by the {\it heat-trace}, which in $\mathbb{R}^d$ is given as}
\begin{eqnarray}
  \label{eqn:heat_trace} Z(t) &\triangleq& \sum_{n=1}^{\infty}e^{-\lambda_n t} = \sum_{s = 0}^\infty a_{s/2}t^{-d/2+ s/2},\ \ t>0.
\end{eqnarray}
The coefficients of the polynomial expansion, $a_{s/2}$, are the
components carrying the geometric information. These coefficients are
given as sums of volume and boundary integrals of some local
invariants of the
shape,~\cite{McKean1967,Smith1981,Vassilevich2003}. For instance, as
given in~\cite{McKean1967}, the first three coefficients are:
\begin{eqnarray}
  \nonumber a_{0} &=& \frac{1}{(4\pi)^{d/2}} V_{\Omega}\\
  \nonumber a_{1/2} &=& -\frac{1}{4(4\pi)^{d/2-1/2}} S_{\Omega},\\
  \nonumber a_{1} &=& -\frac{1}{6(4\pi)^{d/2}}\int_{\partial\Omega}\kappa d\partial\Omega,
\end{eqnarray}
where $S_{\Omega}$ is the surface area (circumference in 2D) and
$\kappa$ is the mean (geodesic) curvature on the boundary of
$\Omega$. The functional relationship between the eigenvalue sequence
and the coefficients $a_{s/2}$ can be seen in
Equation~(\ref{eqn:heat_trace}). This connection relates the spectrum
to the intrinsic geometry, which is the reason why Laplace spectrum is
important for the computational study of shapes.

In addition to the spectrum-geometry link, the eigenvalues of the
Laplace operator have two other properties which make them useful for
shape analysis,~\cite{CourantHilbert}. These are: 1) the Laplace
operator is invariant to isometric transformations and 2) the
spectrum depends continuously on the deformations applied to the
boundary of the object. The advantage of the first property is obvious since isometric transformations do not alter the shape. In addition to this, the second property states that there is a continuous link between the differences in eigenvalues and the difference in shape, which makes eigenvalues ideal for measuring shape differences. 

Unfortunately, it has also been shown that there exists isospectral non-congruent objects, i.e. objects with different shape but the same spectrum~\cite{Gordon1992}. Therefore, theoretically the Laplace
spectrum does not uniquely identify shapes. However, as stated
in~\cite{Reuter2006}, practically this does not cause a problem mostly
because the constructed isospectral non-congruent objects in 2D and 3D are rather extreme examples with nonsmooth boundaries.

The spectral signature, shape-DNA, proposed in~\cite{Reuter2006} is
inspired from the properties given above. For a given shape $\Omega$,
its shape-DNA is the first $N$ modes of the spectrum of the Laplace
operator defined on $\Omega$: $[\lambda_1, \lambda_2, \dots, \lambda_N]$. In addition to the properties the shape-DNA inherits from the eigenvalues, the authors also proposed several normalisations to obtain almost scale invariance\footnote{We use the term ``almost" because scale invariance is an application dependent concept and the definition of scale difference between arbitrary objects is a mathematically vague notion. A further discussion of scale invariance is outside the scope of this article and we refer the reader to~\cite{Reuter2006}.}. The normalisations used in the experiments in~\cite{Reuter2006,Niethammer2007,Reuter2009,Lian2011} are given as $\lambda_n \rightarrow \lambda_n V_{\Omega}^{2/d}$ and $\lambda_n \rightarrow \lambda_n / \lambda_1$.

In~\cite{Reuter2006}, the authors also defined a shape distance based
on shape-DNA. Either using the original or its scale invariant
version, this distance is given as
\begin{eqnarray}
  \label{eqn:r_distance} \rho_{SD}^N(\Omega_\lambda,\Omega_\xi) \triangleq  \left[\sum_{n=1}^N \left( \lambda_n - \xi_n \right)^2 \right]^{1/2},
\end{eqnarray}
where $\Omega_\xi$ denotes the object with the spectrum
$\{\xi_n\}_{n=1}^\infty$. Using
$\rho_{SD}^N(\Omega_\lambda,\Omega_\xi)$, the authors were able to
distinguish between distinct shapes~\cite{Lian2011}, construct shape manifolds based on the pairwise distances and perform statistical
comparisons~\cite{Niethammer2007,Reuter2009}.

However, as also pointed out in~\cite{Memoli2010}, due to the diverging nature of the spectrum, $\rho_{SD}^N$ suffers from three essential drawbacks limiting its usability: i) differences at higher modes of the spectrum have higher impacts on the final distance value even though they are not necessarily more informative about the intrinsic geometry, ii) the distance is extremely sensitive to the signature size $N$, while the choice of this parameter is arbitrary, and iii) the distance cannot be defined over the entire spectrum because it does not yield a proper metric in that case. Therefore, defining a sound and intuitive distance based on the spectrum is still an open question for which we propose a solution in the next section.
% Section 3
\section{Weighted Spectral Distance - WESD}\label{sec:wsd_theory}
This section presents the proposed spectral distance, WESD, the analysis of the heat-trace leading to its definition and its theoretical properties. The structure of presentation aims to separate the definition of the distance, which is essential for its practical implementation, from the details related to its derivation and theoretical properties. In this light, we first present the definitions and mention the associated properties with appropriate references to the following subsections, which contain further details.

We define the Weighted Spectral Distance - WESD - for two closed bounded domains with piecewise smooth boundaries, $\Omega_\lambda,\Omega_\xi\subset \mathbb{R}^d$ as
\begin{equation}\label{eqn:wsd_infinite}
\rho(\Omega_\lambda,\Omega_\xi) \triangleq \left[ \sum_{n=1}^{\infty} \left(\frac{|\lambda_n - \xi_n|}{\lambda_n\xi_n}\right)^p \right]^{1/p},
\end{equation}with $p\in\mathbb{R}$ and $p>d/2$. Unlike the distance given in Equation~(\ref{eqn:r_distance}), WESD is defined over the entire eigenvalue sequence and the factor $p$ is not fixed to 2. In addition, the difference at each mode contributes to the overall distance proportional to $|\lambda_n - \xi_n|/\lambda_n\xi_n$ instead of $|\lambda_n - \xi_n|$. The additional $\lambda_n\xi_n$ factor (seeming like a simple addition to Equation~\ref{eqn:r_distance}) actually arises from analysing the relation between the n\superscript{th} mode of the spectrum and the heat-trace, which will be presented in Section~\ref{sec:heat_trace_analysis}. This analysis also provides WESD with a geometric intuition. Furthermore, for $p>d/2$ the infinite sum in the definition is guaranteed to converge to a finite value for any pair of shapes. Hence, WESD  exists. In addition to its existence, WESD also satisfies the triangular inequality making it a pseudometric. These points are proven in Section~\ref{sec:theo_analysis}. Moreover, the pseudometric WESD has a multi-scale aspect with respect to $p$. In Section~\ref{sec:multi} we show that adjusting $p$ controls the sensitivity of WESD with respect shape differences at finer scales, i.e. with respect to geometric differences at local level such as thin protrusions or small bumps. Thus, for higher values of $p$ the distance becomes less sensitive to finer scale differences.

In addition to WESD, we define the normalised score for shape dissimilarity {\it nWESD} as
\begin{equation}\label{eqn:nwsd}
\overline{\rho}(\Omega_\lambda,\Omega_\xi) \triangleq \frac{\rho(\Omega_\lambda,\Omega_\xi)}{\mathbf{W}(\Omega_\lambda,\Omega_\xi)}\in[0,1),
\end{equation}
which maps $\rho(\Omega_\lambda,\Omega_\xi)$ to the $[0,1)$ interval using the shape-dependent normalisation factor
\begin{equation}
\nonumber\mathbf{W}(\Omega_\lambda,\Omega_\xi) \triangleq \left\{C + K\cdot\left[ \zeta\left(\frac{2p}{d}\right) - 1 - \left(\frac{1}{2}\right)^\frac{2p}{d}\right] \right\}^\frac{1}{p}.
\end{equation}
The factors $C$ and $K$ are the shape based coefficients defined in Corollary~\ref{corr:conv}, and $\zeta(\cdot)$ is the Riemann zeta function~\cite{Whittaker}.
Being confined to $[0,1)$, nWESD allows us to i) compare dissimilarities of different pairs of shapes and ii) easily use the shape dissimilarity in combination with scores quantifying other type of differences between objects such as volume overlap in case of matching or Jacard's index in case of accuracy assessment. 

One important issue in defining a distance or a score using the entire eigenvalue sequence is computational limits. In practice we can only compute a finite number of eigenvalues and therefore, can only approximate such distances. Considering this, here we define the finite approximations of WESD and nWESD using the smallest $N$ eigenvalues as
\begin{eqnarray}
\label{eqn:wsd_finite}\rho^N(\Omega_\lambda,\Omega_\xi) &\triangleq& \left[ \sum_{n=1}^N \left( \frac{|\lambda_n-\xi_n}{\lambda_n\xi_n} \right)^p \right]^{1/p} \\
\label{eqn:nwsd_finite}\overline{\rho}^N(\Omega_\lambda,\Omega_\xi)&\triangleq& \frac{\rho^N(\Omega_\lambda,\Omega_\xi)}{\mathbf{W}(\Omega_\lambda,\Omega_\xi)}\in[0,1),
\end{eqnarray}where $N$ is a truncation parameter. 
%Previous works, such %as~\cite{Reuter2006,Rustamov2007,Lai2009,Lai2010} also defined shape %distances based on finite number of modes. However, the effects of %using finite modes have not been carefully analysed.
{Previous works, such as~\cite{Reuter2006,Rustamov2007,Lai2009,Jurman2010,Lai2010}, also define distances based on finite number of modes. However, their view on the distance definition was first to construct finite shape signatures and then to define a distance on the signatures. Therefore, the signature size was a critical component of the definition itself. Furthermore, the effects of the choice of the signature size on the distance values have not been carefully analysed in these works. The view presented here defines the distance directly using the entire sequence without constructing a finite signature. This alleviates the importance of the signature size on the distance. The finite computation given in Equations~\ref{eqn:wsd_finite} and~\ref{eqn:nwsd_finite} are viewed as approximations to the distance and $N$ as the truncation parameter. In this conceptually different setting, unlike previous works, we provide in Section~\ref{sec:choosingN} a careful analysis of the choice of $N$ on the spectral distance.} Specifically, we prove that $\lim_{N\rightarrow\infty}|\rho(\Omega_\lambda,\Omega_\xi)-\rho^N(\Omega_\lambda,\Omega_\xi)| = 0$ and $\lim_{N\rightarrow\infty}|\overline{\rho}(\Omega_\lambda,\Omega_\xi)-\overline{\rho}^N(\Omega_\lambda,\Omega_\xi)| = 0$. Furthermore, we provide a theoretical upper bound for these errors that shows how fast they decrease in the worst case leading to a principled strategy for choosing $N$.

Section~\ref{sec:scale} ends the section by focusing on the invariance of WESD and nWESD to global scale (relative size) differences between objects. Specifically, we discuss how an ``approximate" scale invariance can be attained for WESD and nWESD by following the same strategy proposed in~\cite{Reuter2006}.
\subsection{Analysis of the Heat-Trace and Derivation of WESD}\label{sec:heat_trace_analysis}
{We derive WESD by analysing the mathematical link between the spectrum of an object and its geometry. This link is given by the heat-trace defined in Equation~(\ref{eqn:heat_trace}). Let us consider the heat-trace as a function of both $t$ and the spectrum, $Z(t,\lambda_1,\lambda_2,...)$. The main question we answer is how much the $Z(t,\cdot)$ function changes when we change the $n^{th}$ mode of the spectrum from $\lambda_n$ to $\xi_n$. Considering the polynomial expansion equivalent to $Z(t,\cdot)$ given in Equation~(\ref{eqn:heat_trace}), one can see that the change in the value $Z(t,\cdot)$ is directly related to the changes in the coefficients $a_{s/2}$ and so to changes in the integrals over the local invariants. By analysing the influence of the change in the n\superscript{th} mode on $Z(t,\cdot)$, we actually analyse the influence of this change on the integrals over local geometric invariants. Following this line of thought, we quantify the influence of the change from $\Lambda_n$ to $\xi_n$ on $Z(t,\cdot)$ in terms of $\lambda_n$ and $\xi_n$. This can be done by defining
\begin{eqnarray}
\nonumber\Delta_Z^{n} &\triangleq& \int_0^\infty \left|
Z(t,\dots,\lambda_{n-1},\lambda_n,\lambda_{n+1},\dots) \right.\\
\nonumber &&\ \ \ \ \ \ \ \left. - Z(t,\dots,\lambda_{n-1},\xi_n,\lambda_{n+1},\dots)\right|dt,
\end{eqnarray}which is simply the $L_1$-norm of the difference between the functions that is linked to the difference between $\lambda_n$ and $\xi_n$. Replacing $Z(t,\cdot)$ with its definition leads to
\begin{eqnarray}\label{eqn:heat_trace_integral}
\Delta_Z^{n} &=& \int_0^\infty \left|e^{-\lambda_n t} - e^{-\xi_n t}\right|dt
\end{eqnarray}
Without loss of generality let us assume $\xi_n \geq \lambda_n$. Then 
\begin{equation}
\nonumber e^{-\lambda_n t} \geq e^{-\xi_n t}\ \text{for }t>0.
\end{equation}
We can then evaluate the integral in Equation~(\ref{eqn:heat_trace_integral}) as
\begin{eqnarray}
\nonumber\Delta_Z^{n} = \int_0^{\infty} e^{-\lambda_n t}- e^{-\xi_n t}dt = \frac{|\lambda_n - \xi_n|}{\lambda_n\xi_n}.
\end{eqnarray}
$\Delta_Z^{n}$ captures the influence of the difference at the n\superscript{th} mode on $Z(t,\cdot)$. Now, aggregating these influences across all modes leads to the definition of WESD
\begin{equation}
\nonumber\rho(\Omega_\lambda,\Omega_\xi) = \left[\sum_{n=1}^\infty \left(\Delta_Z^{n}\right)^p\right]^{1/p} = \left[\sum_{n=1}^\infty \left(\frac{|\lambda_n - \xi_n|}{\lambda_n \xi_n}\right)^p\right]^{1/p}.
\end{equation}}

{Surprisingly, the formulation of WESD, which results from the analysis presented above, also has very beneficial properties that makes it theoretically sound and useful in practical applications. These properties will be analysed in the following.}

{Before delving into this analysis though let us make two remarks. The first relates $\rho_{SD}(\cdot,\cdot)$ (Equation~(\ref{eqn:r_distance})) to the analysis of the heat-trace presented above.
\begin{remark}\label{remark:Reuter}
Let us define 
\begin{eqnarray}
\nonumber\Delta_Z^{n,m} &\triangleq& \left|\int_0^\infty \frac{d^m}{dt^m}Z\left(t,...,\lambda_{n-1},\lambda_n,\lambda_{n+1},... \right) \right.\\
\nonumber && \ \ \ \ \left. -  \frac{d^m}{dt^m}Z\left(t,...,\lambda_{n-1},\xi_n,\lambda_{n+1},... \right) dt \right|,
\end{eqnarray}
Then $\Delta_Z^{n,0} = \Delta_Z^n$. Evaluating this integral yields $\Delta_Z^{n,m} = \left| \lambda_n^{m-1} - \xi_n^{m-1}\right|$. By setting $m=2$ $\rho_{SD}(\cdot,\cdot)$ can be derived as follows
\begin{equation}
\nonumber \rho_{SD}(\Omega_\lambda,\Omega_\xi) = \left[ \sum_{n=0}^\infty \left(\Delta_Z^{n,2}\right)^2\right]^{1/2}.
\end{equation}
This derivation not only relates WESD to $\rho_{SD}(\cdot,\cdot)$ but also provides the link between $\rho_{SD}(\cdot,\cdot)$ and the heat-trace. 
\end{remark}}

{The second remark notes the link between WESD and {\em Global Point Signatures} ($\textrm{GPS}$), a local shape descriptor, presented in \cite{Rustamov2007}. 
\begin{remark}\label{remark:Rustamov}
$\textrm{GPS}$, as presented in \cite{Rustamov2007}, is defined for each point in an object $\Omega_\lambda$ as the infinite series $\textrm{GPS}_{\Omega_\lambda}(\mathbf{x})\triangleq\left\{\Phi_{\lambda,n}(\mathbf{x})\right\} \triangleq\left\{ \lambda_n^{-1/2}f_n(\mathbf{x}) \right\}_{n=1}^\infty$, where $\mathbf{x}\in\Omega_\lambda$ and $f_n(\mathbf{x})$ is the n\superscript{th} eigenfunction. $\textrm{GPS}$ has a connection to WESD arising from the following element-wise integrals
\begin{equation}
\nonumber \int_{\Omega_\lambda}\Phi_{\lambda,n}^2(\mathbf{x})d\mathbf{x} = \int_{\Omega_\lambda}\left[ \lambda_n^{-1/2} f_n(\mathbf{x})\right]^2d\mathbf{x} = \lambda_n^{-1},
\end{equation}
where the equality arises from the fact that eigenfunctions form an orthonormal basis in $\Omega_\lambda$ \cite{CourantHilbert}, i.e. $\int_{\Omega_\lambda}f_n(\mathbf{x})f_m(\mathbf{x})d\mathbf{x} = \delta(n-m)$ with $\delta(\cdot)$ being the Dirac's delta. Considering this integral, WESD can also be regarded as a distance between $\textrm{GPS}$' of two objects as
\begin{eqnarray}
\nonumber \left\{\sum_{n=1}^\infty \left[\int_{\Omega_\lambda}\Phi_{\lambda,n}^2(\mathbf{x})d\mathbf{x} - \int_{\Omega_\xi}\Phi_{\xi,n}^2(\mathbf{x})d\mathbf{x}\right]^p\right\}^\frac{1}{p}
= \rho(\Omega_\lambda,\Omega_\xi).
\end{eqnarray}
This link also provides an alternative view on the normalisation factor $\lambda_n^{-1/2}$ used in $\textrm{GPS}$. In \cite{Rustamov2007} author justifies this normalisation factor by noting that for an object the {\em Green's function} can be written as an inner product in the GPS domain, see Section 4 in \cite{Rustamov2007}. This is later used to argue the geometric meaning of GPS as authors point out the use of Green's function in different shape processing tasks. Our link between GPS and WESD provides an alternative view on the normalisation factor as it connects this local signature to the heat-trace $Z(t)$.
\end{remark}}
\subsection{Existence of the Pseudometric
  WESD}\label{sec:theo_analysis}
WESD is defined as the limit of an infinite series as given in
Equation~(\ref{eqn:wsd_infinite}). For such a distance to be a proper one, actually a pseudometric in this case, the limit of the infinite series should exist for any two spectra. In the case of WESD, this is not evident because it is defined over the entire spectra and each spectrum is a divergent sequence. The first corollary presented below proves that when $p>d/2$ WESD indeed satisfies this condition, i.e. the infinite series converges. The corollary further provides an upper bound for this limit, which is used to construct nWESD. We would like to note that for the ease of presentation, the proofs for all the following corollaries and lemmas are given in Appendix B in the supplemental material.
\begin{corollary}\label{corr:conv}
  Let $\Omega_\lambda\subset\mathbb{R}^d$ and
  $\Omega_\xi\subset\mathbb{R}^d$ be any two closed domains with
  piecewise smooth boundaries and $\{\lambda\}_{n=1}^\infty$ and
  $\{\xi\}_{n=1}^\infty$ be their Laplace spectrum. Then the weighted
  spectral distance
  \begin{equation}
    \nonumber\rho(\Omega_\lambda, \Omega_\xi) = \left[ \sum_{n=1}^\infty  \left(\frac{|\lambda_n - \xi_n|}{\lambda_n \xi_n}\right)^p\right]^{1/p}
  \end{equation}converges for $p>\frac{d}{2}$. Furthermore, 
  \begin{equation}
    \label{eqn:n_factor}\rho(\Omega_\lambda, \Omega_\xi) < \left\{ C + K\cdot\left[ \zeta\left(\frac{2p}{d}\right) - 1 - \left(\frac{1}{2}\right)^\frac{2p}{d}\right] \right\}^\frac{1}{p},
  \end{equation}where $\zeta(\cdot)$ is the Riemann zeta function and the coefficients $C$ and $K$ are given as
  \begin{eqnarray}
    \nonumber C &\triangleq& \sum_{i=1,2} \left[\frac{d+2}{d\cdot4\pi^2}\cdot\left(\frac{B_d\hat{V}}{i}\right)^\frac{2}{d} - \frac{1}{\mu}\cdot\left(\frac{d}{d+4}\right)^{i-1}\right]^p\\
    \nonumber K &\triangleq& \left[ \frac{d+2}{d\cdot4\pi^2}\cdot\left(B_d\hat{V}\right)^\frac{2}{d} - \frac{1}{\mu}\cdot\frac{d}{d + 2.64}\right]^p\\
    \nonumber \hat{V} &\triangleq& \max(V(\Omega_\lambda),V(\Omega_\xi)),\ \ \mu\triangleq\max(\lambda_1,\xi_1),
  \end{eqnarray}
  where $V(\cdot)$ denotes the volume (or area in 2D) of an object.
\end{corollary}
The Inequality~(\ref{eqn:n_factor}) states that WESD has a
shape-dependent upper bound. We thus can map the WESD to the $[0,1)$
interval through normalising it with this upper bound. The nWESD
score, given in Equation~\ref{eqn:nwsd} is constructed based on this
strategy. Since its existence is established next we prove that WESD
is a pseudometric, i.e. satisfies the other criteria to be a pseudometric, such as the triangle inequality.
\begin{corollary}\label{corr:pseudometric}
  $\rho(\Omega_\lambda,\Omega_\xi)$ is a pseudometric for $d\geq2$.
\end{corollary}

We note that WESD is not a metric because the spectrum is invariant to
isometries, which is a desirable property for shape analysis. However,
in addition to this, the spectrum is also invariant to isospectral
non-congruent shapes. This is not desirable but does not cause
problems in practice as discussed in Section \ref{sec:solo} and also
confirmed in our experiments.
%%%%%%%%%%%%%%%%%%%%%%%%%%%%%%%%%%%%%%%
\subsection{On the multi-scale aspect of WESD}~\label{sec:multi} 
The previous section highlighted the role of $p$ on the convergence
properties of WESD and therefore on its existence. We now demonstrate that $p$ also provides WESD a multi-scale characteristic. The sensitivity of WESD to the shape differences at finer scales depends on the value of $p$. Specifically, we show that the higher the value $p$ the less sensitive WESD is to finer scale details and its sensitivity increases as $p$ gets lower.

The multi-scale aspect of WESD arises from the relationship between
the Laplace operators and heat diffusion processes~\cite{EvansBook}. We first present an intuitive summary of this relationship, which is about the multi-scale aspect of $Z(t)$ and $t$ in particular. For a more mathematical treatment we refer the reader to~\cite{Sun2009}. As stated in~\cite{Sun2009} and~\cite{Memoli2010}, $t$ can be interpreted as the time variable in a heat diffusion process within an object. A useful visual analogy to consider here is the Laplacian smoothing of a surface where $t$ would correspond to the amount of smoothing. Similar to the surface smoothing, as $t$ increases, the local geometric details of an object, such as sharp ridges or steep valleys, lose further their influence on the $Z(t)$ value. As a result $Z(t)$ becomes somewhat insensitive to these local geometric details, in
other words shape details at finer scales. From an alternative view, the value of $Z(t)$ loses its information content with regards to local geometric details. This effect intuitively summarizes the multi-scale characteristic of the heat-trace with respect to $t$.

Having explained the multi-scale aspect of $Z(t)$, we now analyse how this aspect is reflected upon the eigenvalues. To do so let us define the {\em influence ratio} $\mathcal{D}(n,t)\triangleq \frac{e^{-\lambda_n t}}{Z(t)}$. This ratio captures the influence of the $n^{th}$ mode on the heat-trace. In other words, the higher the ratio, the higher the influence of $\lambda_n$ on the value of $Z(t)$ at that specific $t$. The following lemma compares the influence ratios of different modes and how this comparison depends on $t$.
\begin{lemma}~\label{lemma:influenceratio} Let
  $\Omega_\lambda\subset\mathbb{R}^d$ represent an object with
  piecewise smooth boundary and $\mathcal{D}(l,t)\triangleq
  \frac{e^{-\lambda_l t}}{Z(t)}$ be the corresponding influence ratio
  of mode $l$ at $t$. Then for any two spectral indices
  $m>n>0$
  \begin{equation}\nonumber\mathcal{D}(n,t) > \mathcal{D}(m,t),\ \
    \forall t>0\end{equation}
  and particularly for two $t$ values such that $t_1 > t_2$
  \begin{equation} \nonumber
    \frac{\mathcal{D}(m,t_1)}{\mathcal{D}(n,t_1)} <
    \frac{\mathcal{D}(m,t_2)}{\mathcal{D}(n,t_2)}.\end{equation}
\end{lemma}

The first inequality of the lemma indicates that the lower modes in the spectrum have more influence on the value of $Z(t)$ than the higher modes. The second inequality shows that the influence of the
higher modes become more prominent as $t$ decreases. Considering that
for lower $t$ values $Z(t)$ is more informative with regards to shape
details at finer scales, Lemma~\ref{lemma:influenceratio} suggests
that the higher modes are more important for finer scales than for coarser scales. We illustrate this observation on a synthetic example shown in Figure~\ref{fig:multi_scale} with the pair (a) + (b) being an example showing coarser shape differences and the pair (a) + (c) showing finer differences. The plots given in Figure~\ref{fig:multi_scale} (d) and (e) show the corresponding spectral differences observed at modes between 1 and 150. {Between (a) and (b) the shape differences are at the coarse level. According to Lemma~\ref{lemma:influenceratio} these differences should show up at the very first modes. On the other hand, between (a) and (c) the differences are at a finer scale and furthermore the objects are very similar at the coarse level. Lemma~\ref{lemma:influenceratio} states that these differences therefore, should show up at higher modes and the differences at the lower modes should be low. Satisfying these expectations, the differences at the first few modes shown in plot (d) have relatively large values compared to the ones in plot (e).} Furthermore, the amplitude of the differences at higher modes are generally larger in plot (e) than in plot(d), especially after 100.
\begin{figure}[!h]
  \begin{center}
    \subfigure[]{\includegraphics[width=0.23\linewidth]{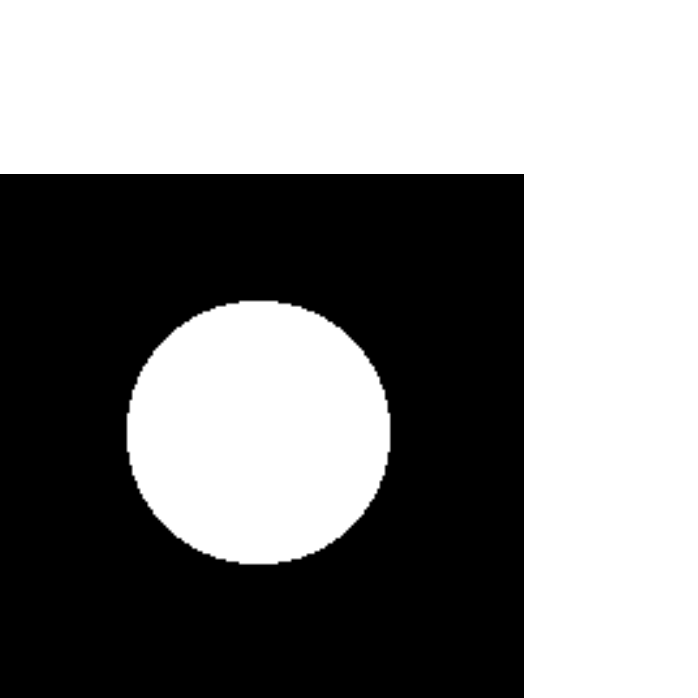}}
    \subfigure[]{\includegraphics[width=0.23\linewidth]{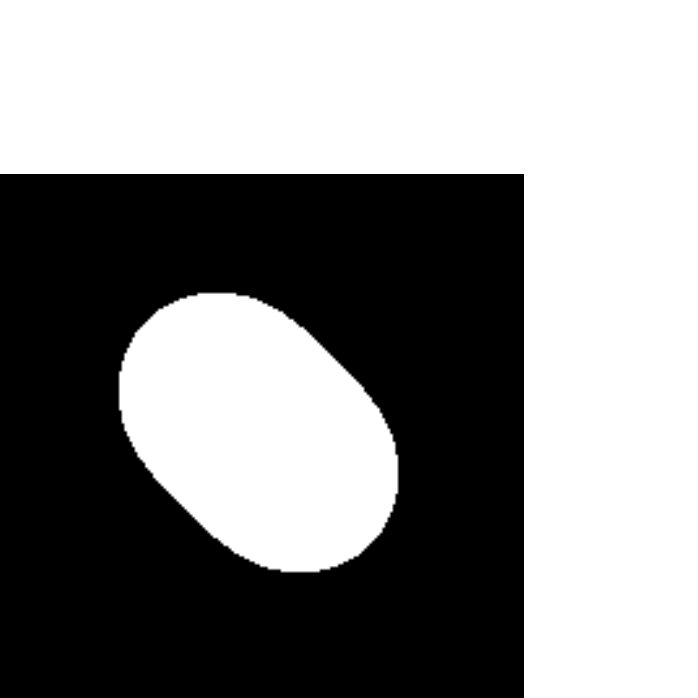}}
    \subfigure[]{\includegraphics[width=0.23\linewidth]{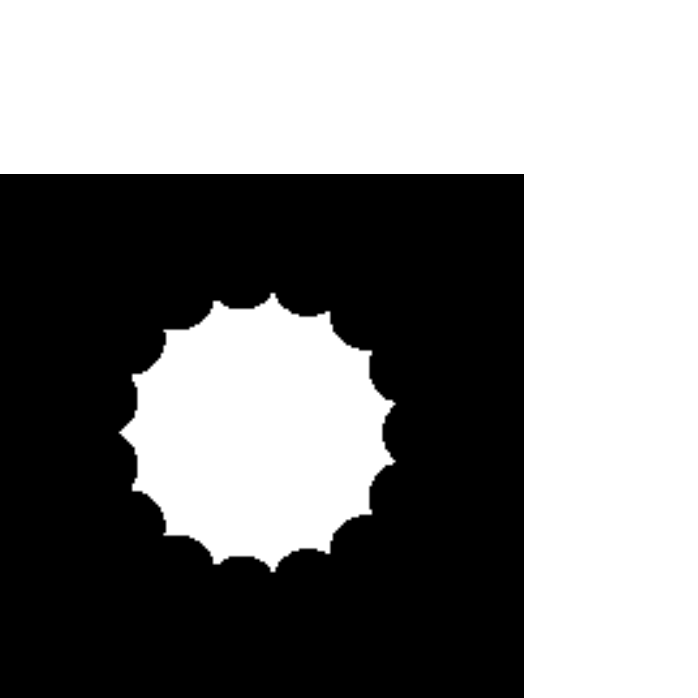}}
    \subfigure[between (a) and (b)] {\includegraphics[width=0.48\linewidth]{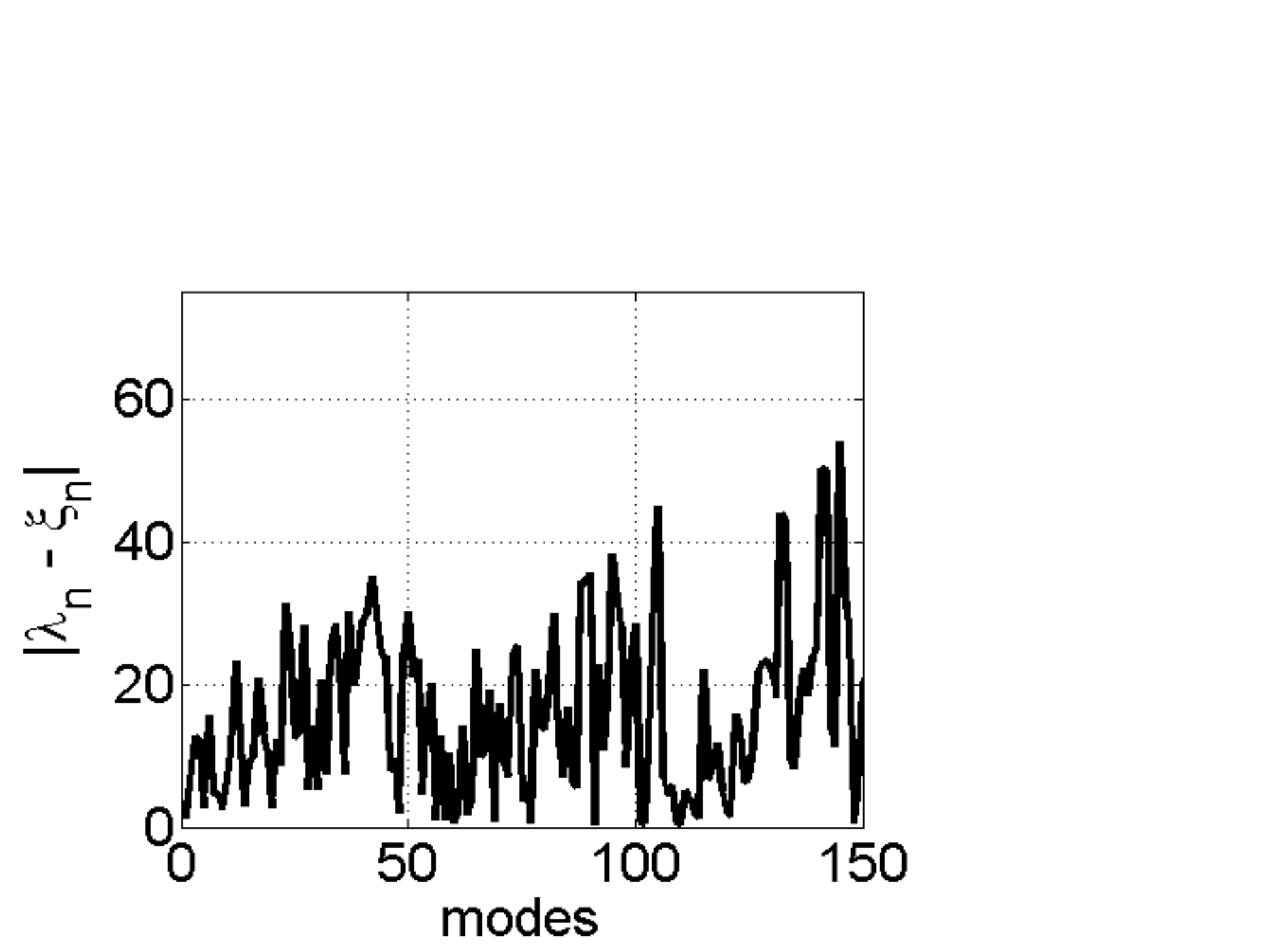}}
    \subfigure[between (a) and (c)] {\includegraphics[width=0.48\linewidth]{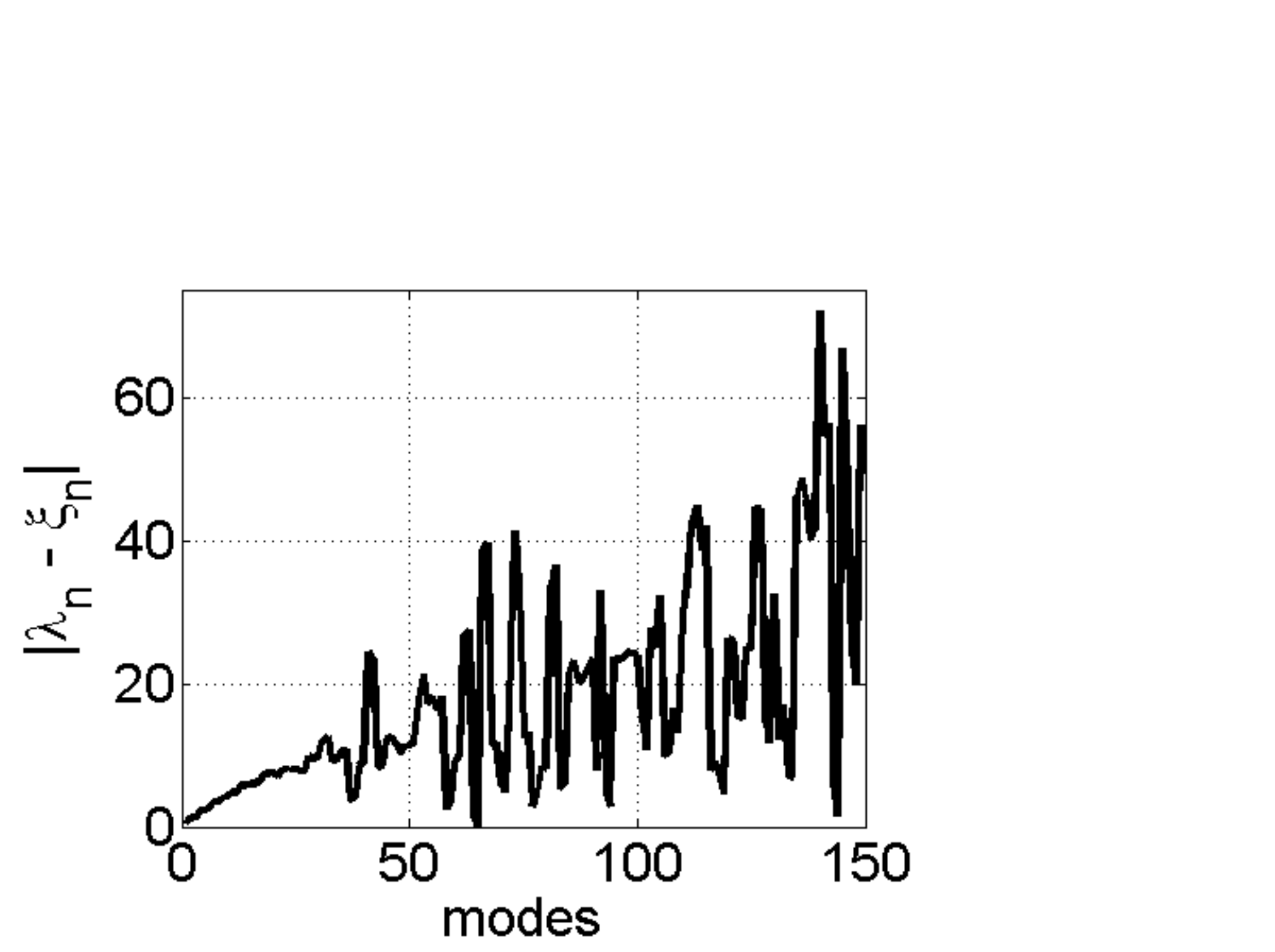}}
  \end{center}
  \caption{\label{fig:multi_scale}Multi-scale characteristics of different spectral modes:(a), (b) and (c) show three synthetic shapes. In (d) we plot the absolute differences between the corresponding modes of (a) and (b) with respect to the spectral index. In (e) we plot the same difference for the shapes in (a) and (c). The shape difference between (a) and (b), which is at a coarser level, is already captured at the lower spectral modes. The difference between (a) and (c) results in lower differences in lower spectral modes because these objects are more similar at a coarser level. At the higher spectral modes though the difference between (a) and (c) becomes more prominent since these objects differ more substantially at the finer scales. The plots in (d) and (e) demonstrate that the higher modes for a given object are more important for finer scale shape details.}
\end{figure}

In order now to connect these findings to WESD and $p$ let us present
the following corollary, which studies the influence of $p$ on the
components inside the infinite sum defining the distance.
\begin{corollary}~\label{corr:multires} Let $\Omega_\lambda$ and
  $\Omega_\xi$ be two objects with piecewise smooth boundaries. Then
  for any two scalars with with $p>d/2$, $q>d/2$, $p \geq q$ and for all $n$ with   $|\lambda_n - \xi_n|>0$ there exists a $M>n$ so that $\forall m\geq M$
  \begin{equation}
    \nonumber \frac{\left(\frac{|\lambda_m - \xi_m|}{\lambda_m\xi_m}\right)^p}{\left(\frac{|\lambda_n - \xi_n|}{\lambda_n\xi_n}\right)^p} \leq \frac{\left(\frac{|\lambda_m - \xi_m|}{\lambda_m\xi_m}\right)^q}{\left(\frac{|\lambda_n - \xi_n|}{\lambda_n\xi_n}\right)^q}
  \end{equation}
\end{corollary}

Thus, the relative contributions of the higher spectral modes on $\rho(\Omega_\lambda,\Omega_\xi)$ with respect to the contributions of the lower modes depend on the value of $p$.  Specifically, the higher spectral modes become more influential as $p$ decreases. Combining this finding with the result of Lemma~\ref{lemma:influenceratio}, we follow that as $p$ increases WESD gives less importance to differences at higher spectral modes and therefore becomes less sensitive to the shape differences at finer scales. This provides WESD with a multi-scale aspect with respect to $p$ and also provides us the intuition for choosing a proper value for $p$.
\subsection{Finite Approximations of WESD and
  nWESD}\label{sec:choosingN}
One of the important practical questions regarding spectral distances is the number of modes to be included in the calculation of the distance. The computation of eigenvalues and eigenfunctions can be expensive and inaccurate especially for the higher modes. Therefore, spectral distances require the user to set a finite number of modes to be used. {This parameter is often referred to as the signature size. Having defined the distance over the entire sequence, we refer to it as the {\em truncation parameter}. This actually provides a different perspective on the number of modes used to compute the distance. In previous works, such as~\cite{Reuter2006, Rustamov2007, Lai2009}, the value of this parameter, viewed as the signature size, is often set arbitrarily and its effect on the distances have not been carefully analysed. Here, viewing it as a truncation parameter, we study its influence. Specifically, we formulate the difference between using the entire spectra to only using a finite number of modes as an {\em approximation/truncation error}. So we analyse how this error changes with respect to the truncation parameter.} We specifically show in the next corollary that the errors in approximating WESD and nWESD by the first $N$ modes converges to zero as $N$ increases. Furthermore, we provide an upper bound for both errors as a function of $N$.
\begin{corollary}\label{corr:trunc}
  Let $\rho^N(\Omega_\lambda,\Omega_\xi)$ be the truncated
  approximation of $\rho(\Omega_\lambda,\Omega_\xi)$ based on the first $N$ modes and $\overline{\rho}^N(\Omega_\lambda, \Omega_\xi)$ of $\overline{\rho}(\Omega_\lambda,\Omega_\xi)$. Then $\forall p>d/2$
  \begin{equation}\nonumber\lim_{N\rightarrow\infty}|\rho -\rho^N| =
    0\end{equation}and
  \begin{equation}\nonumber\lim_{N\rightarrow\infty}|\overline{\rho} -
    \overline{\rho}^N| = 0.\end{equation} Furthermore, for a given
  $N\geq 3$ the truncation errors $|\rho-\rho^N|$ and
  $|\overline{\rho}-\overline{\rho}^N|$ can be bounded by
  \begin{eqnarray}
    \label{eqn:errbound_wesd} \left|\rho - \rho^N\right| &<& \left\{C + K\cdot\left[ \zeta\left(\frac{2p}{d}\right) - 1 - \left(\frac{1}{2}\right)^\frac{2p}{d}\right]\right\}^\frac{1}{p}\\
    \nonumber & & - \left\{C + K\cdot\left[ \sum_{n=3}^N \left(\frac{1}{n}\right)^\frac{2p}{d} \right]\right\}^\frac{1}{p}\\
    \label{eqn:errbound_nwesd} \left|\overline{\rho} - \overline{\rho}_N\right| &<& 1 - \left\{\frac{ C + K\cdot\left[ \sum_{n=3}^N \left(\frac{1}{n}\right)^\frac{2p}{d} \right] }{C + K\cdot\left[ \zeta\left(\frac{2p}{d}\right) - 1 - \left(\frac{1}{2}\right)^\frac{2p}{d}\right]}\right\}^\frac{1}{p}
  \end{eqnarray}
\end{corollary}
The above corollary has important practical implications. First of all, the sensitivities of $\rho^N$ and $\overline{\rho}^N$ with respect to $N$ decreases as $N$ increases. For any application relying on the shape distances, such as constructing low dimensional embeddings, this reduced sensitivity is particularly important as it provides stability with respect to $N$ both for the distance and for the application using the distance. We note that the opposite is true for $\rho_{SD}^N$, which is one of the main disadvantages of this distance.

{In addition, Corollary~\ref{corr:trunc} can guide the choice for the number of modes $N$ and the norm type $p$. Specifically, the error upper bounds given in Equations~\ref{eqn:errbound_wesd} and ~\ref{eqn:errbound_nwesd} provide the worst case errors for a given pair of shapes without the need to compute the eigenvalues. So for instance, once a number of modes are computed then based on the distance value obtained so far and the worst case error computed using the upper bounds, one can decide whether to compute more modes or not. Furthermore. these upper bounds are shape-specific as they depend on $C$ and $K$. One can go one step further and define a shape-independent {\em residual ratio} for $N\geq3$ and $p>d/2$ as}
\begin{equation}\label{eqn:residual_ratio}
  R(N,p) \triangleq 1 - \left[\frac{\sum_{n=3}^N\left(\frac{1}{n}\right)^\frac{2p}{d}}{\zeta\left( \frac{2p}{d}\right) - 1 - \left( \frac{1}{2}\right)^\frac{2p}{d}}\right]^\frac{1}{p}.
\end{equation}
that satisfies
$R(N,p) > \overline{\rho} - \overline{\rho}^N$, for which the
proof is given in Proposition 1 in Appendix B. Based on this,
$R(N,p)$ can be used to select the parameters $N$ and $p$ as it
quantifies the quality of the approximation for a given pair of
$(N,p)$ in terms of the error upper bounds.

\begin{figure}[!tb]
  \begin{center}
    \subfigure{\includegraphics[width=0.40\linewidth]{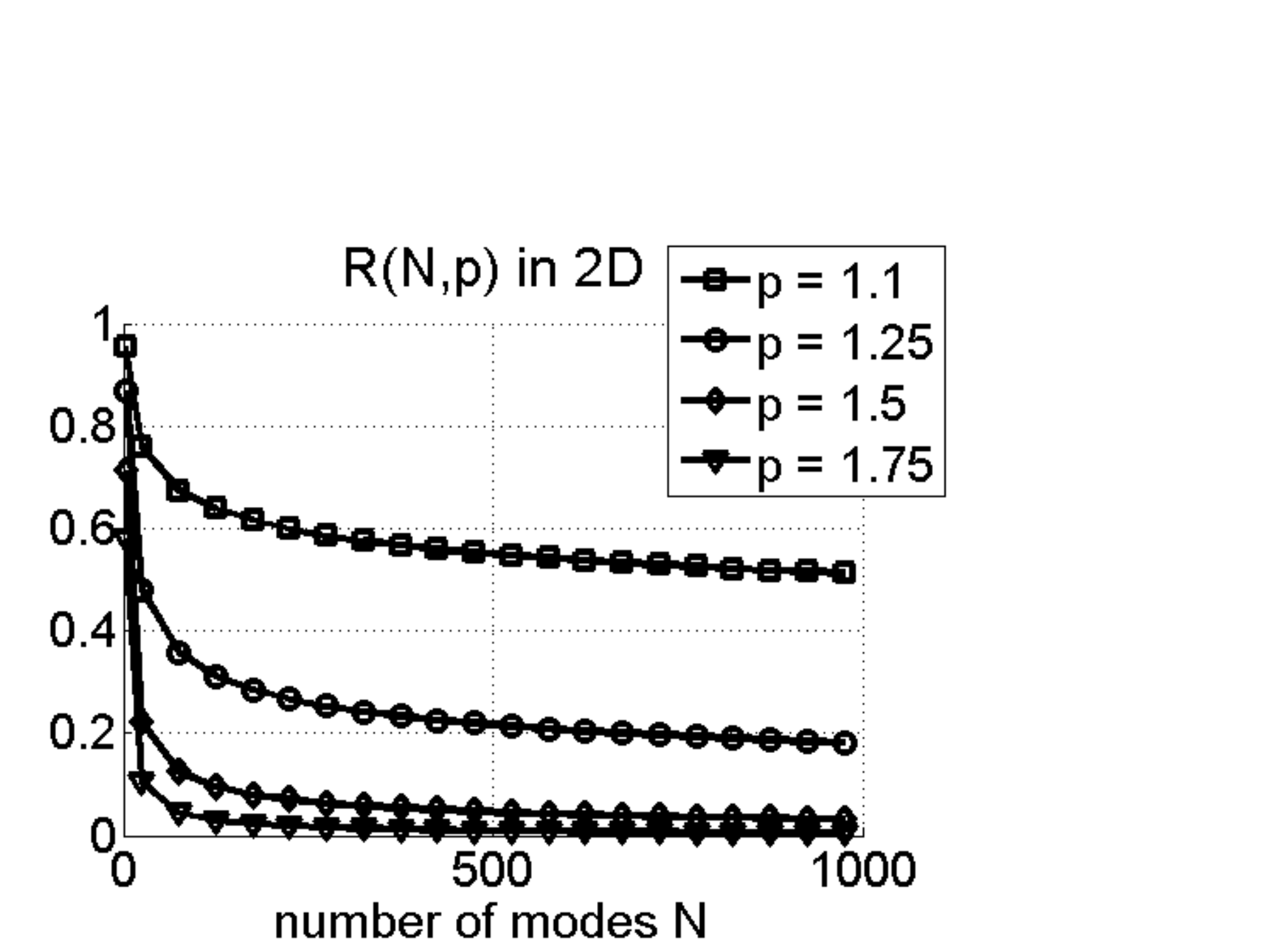}}
    \subfigure{\includegraphics[width=0.40\linewidth]{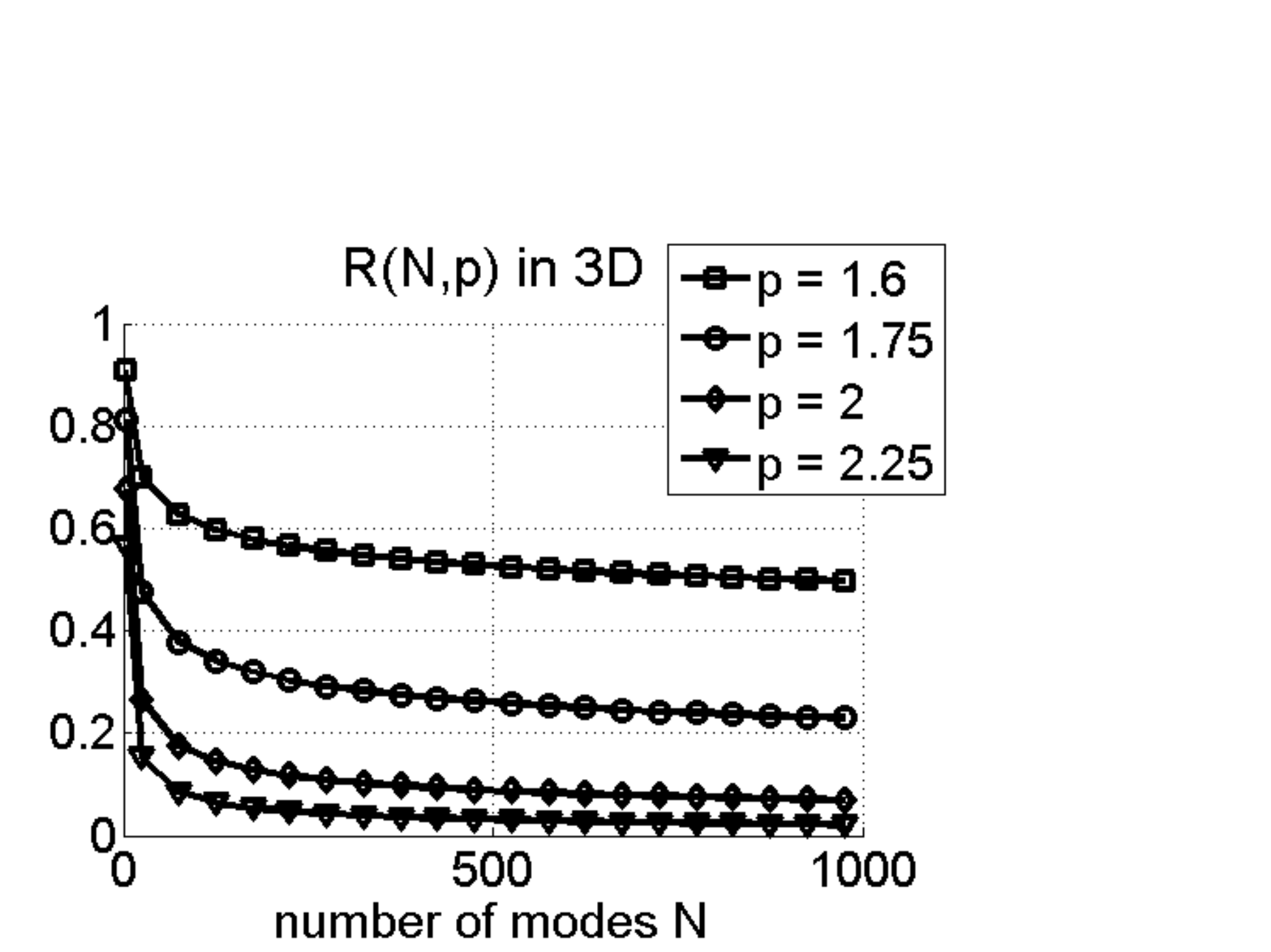}}
  \end{center}
  \caption{\label{fig:upp_bounds}Choosing $N$: The figures plot the residual ratio $R(N,p)$ versus $N$ for different $p$ values in 2D (left) and in 3D (right). As expected the error upper bound drops with increasing $N$. The rate of decrease also becomes faster with increasing $p$. This inverse relation suggests the trade-off between $N$ and the sensitivity of WESD to finer scale shape differences since WESD becomes less sensitive as $p$ increases, see Section~\ref{sec:multi}.}
\end{figure}
In Figure~\ref{fig:upp_bounds}, we plot $R(N,p)$ versus $N$
for different settings of $p$ and $d = 2,3$. Besides the obvious point
that the error upper bound decreases for increasing $N$ we also notice
that i) the behavior in 2D and 3D are similar and ii) the rate of
decrease of the error upper bound is much faster for higher
$p$. Considering the multi-scale aspect of WESD captured in $p$, this
behavior is interesting. It demonstrates that the choice of $p$ and
$N$ are correlated and suggests a trade-off between the rate of
decrease of the truncation error and the sensitivity of WESD to shape
differences at finer scales. In theory, the choice of these parameters
depends on the application and the expected shape differences. If one
expects coarse scale differences then choosing a large $p$ and small
$N$ might be sufficient. However, if one is interested in finer scale
differences then a small $p$ value will be required, which in turn will require a large $N$ value to have a decent approximation. {The important aspect of $R(N,p)$ is that it is universal, i.e. it does not depend on the objects. So it can be used in any type of application to choose the parameter pair $N,p$ and to have a rough estimate of the computational costs for computing the distance WESD. We note once again, the specific values should be chosen based on the application and the shapes at hand.}
\subsection{Invariance to global scale differences}~\label{sec:scale}
We end this section studying the impact of global scale differences on WESD and how invariance to such differences can be attained. We would like to note that the notion of global scale in this section refers to the relative size of an object, which is not to be confused with the notion of multi-scale used in Section~\ref{sec:multi}. The spectrum of an object depends on the object's size, i.e. a global scale change alters all the eigenvalues by a constant multiplicative factor~\cite{CourantHilbert}. As a result, the global scale difference between two objects contributes to the spectral shape distance WESD. In some applications this contribution might not be desirable, for instance in an object recognition task, where objects in the same category have varying sizes. Therefore, it is a useful property of a shape distance to allow invariance to global scale differences.  

{Reuter~\etal~\cite{Reuter2006} proposed different approximations for normalising the effects of scale differences on the spectrum. In particular, the authors use two different normalisations in their experiments in \cite{Niethammer2007,Reuter2009,Lian2011}. Both normalisations directly act on the eigenvalues. The first one normalises the eigenvalues with respect to the volume (area in 2D or surface area for Riemannian manifolds) and is given as $\lambda_n\rightarrow\lambda_nV_{\Omega_\lambda}^{2/d}$. The second one normalises the eigenvalue with respect to the first eigenvalue in the sequence as $\lambda_n \rightarrow \lambda_n / \lambda_1$. Both of these strategies can be used when computing distances with WESD. Furthermore, since these strategies do not alter the mathematical characteristics of the entire spectrum the theoretical properties of WESD and nWESD hold either way. For our experiments we adopt the first strategy, volume normalisation, using the volume as defined in Euclidean geometry. When using the volume normalised eigenvalues, the only change that applies to the technical details presented so far is $\hat{V}$ in Equation~\ref{eqn:n_factor} becomes $\hat{V} = 1$. The rest applies directly without any modification.}

We would also like to note that estimating the global scale difference between two arbitrary objects is not always a well-posed problem. It is especially hard when the objects are of different category, e.g. an octopus and a submarine. Furthermore, the scale normalisation is application dependent and it might not be desirable for all applications. In Section~\ref{sec:cardiac} we present such an example where we analyse the temporal change of the left ventricle shape during a heart cycle. In this case, the volume change is essential for analysing the heart of the same patient so that scale invariance is not appropriate. 

% Section 4
\section{Experiments}\label{sec:wsd_experiments}
This section presents a variety of experiments on synthetic and real data highlighting the strengths and weaknesses of WESD and nWESD. We start by briefly explaining the details of the numerical implementation of WESD used in the experiments presented here. Then in Section~\ref{sec:wsd_synthetic}, the proposed distances are applied to synthetically generated objects demonstrating that
\begin{enumerate}
\item[(i)] Ordering objects with respect to their shapes using nWESD results in a visually coherent series (Section~\ref{sec:wsd_synthetic_1}), 
\item[(ii)] WESD is useful for constructing low dimensional embeddings, in particular it yields stable embeddings with respect to the signature size $N$, (Section~\ref{sec:wsd_synthetic_2}) and 
\item[(iii)] WESD is a suitable distance for shape retrieval, which is shown through experiments on the SHREC dataset~\cite{Lian2011} (Section~\ref{sec:shrec}). 
\end{enumerate}
Lastly, in Section \ref{sec:wsd_med_experiments} WESD is applied to real objects extracted from 3D medical images. We focus on two examples from a wide variety of applications WESD and nWESD can be beneficial to: population studies of brain structures and analysis of 4D cardiac images. 
% Numerical implementation and our consideration of applications %involving images. 
\noindent{\subsection{Implementation Details}
There are two different aspects in the implementation of WESD: the numerical computation of the Laplace spectra and the parameter settings. First, any numerical method tailored towards computing the eigenvalues of the Laplace operator can be used to compute WESD. Examples of such method are listed in \cite{Ames,Reuter2006}. Our specific implementation represents objects simply as binary images with the foreground defining $\Omega$. Using the Cartesian grid of the image, it discretizes $\Delta_\Omega$ through finite difference scheme (see also Chapter 2 of \cite{Ames}). This step yields a sparse matrix of which we compute the eigenvalues via Arnoldi's method presented in \cite{Arnoldi1954} and implemented in MATLAB\textsuperscript{\textregistered}. We choose this specific implementation as 1) it is simple 2) it does not introduce any additional parameters and 3) when working with images it avoids any extra preprocessing steps, such as surface extraction or mesh construction.}

{With regards to the second implementation aspect, we set the parameters $N$ and $p$ empirically. Based on Section~\ref{sec:choosingN}, we set $p=1.5$ in 2D and $p=2$ in 3D. These values result in a relatively fast diminishing upper bound of the truncation error with respect to $N$ (see Figure~\ref{fig:upp_bounds}) while being sensitive to shape differences at finer scales. In both 2D and 3D, we chose $N=200$ for the number of modes as the truncation error seemed to vanish at that point. Furthermore, in addition to the theoretical considerations on the effects of $N$ and $p$ on WESD given in Section~\ref{sec:choosingN}, in Sections~\ref{sec:wsd_synthetic_2} and \ref{sec:shrec} we experimentally study the effects of these parameters on applications using WESD, specifically on constructing low dimensional embeddings and shape retrieval.}
% ------------------------------
\subsection{Synthetic Data}\label{sec:wsd_synthetic}
We conduct three experiments: first two are on 2D objects and the last one is on 3D objects. {For all of the experiments, we use the scale invariant versions of the spectra obtained by normalising the eigenvalues with the object's volume as described in Section~\ref{sec:scale}. As a result the distances WESD and nWESD become ``almost'' invariant to global scale differences.}
\subsubsection{Ordering of Shapes}\label{sec:wsd_synthetic_1}
\begin{figure}[!t]
\begin{center}
\subfigure[]{\includegraphics[width=0.70\linewidth]{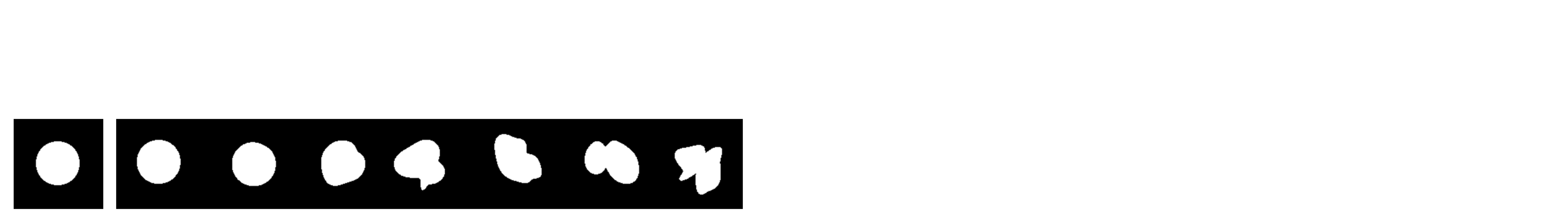}}
\subfigure[]{\includegraphics[width=0.70\linewidth]{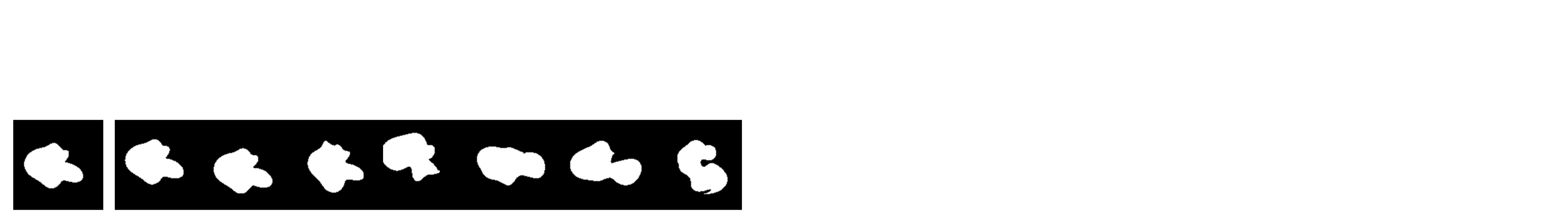}}
\subfigure[]{\includegraphics[width=0.98\linewidth]{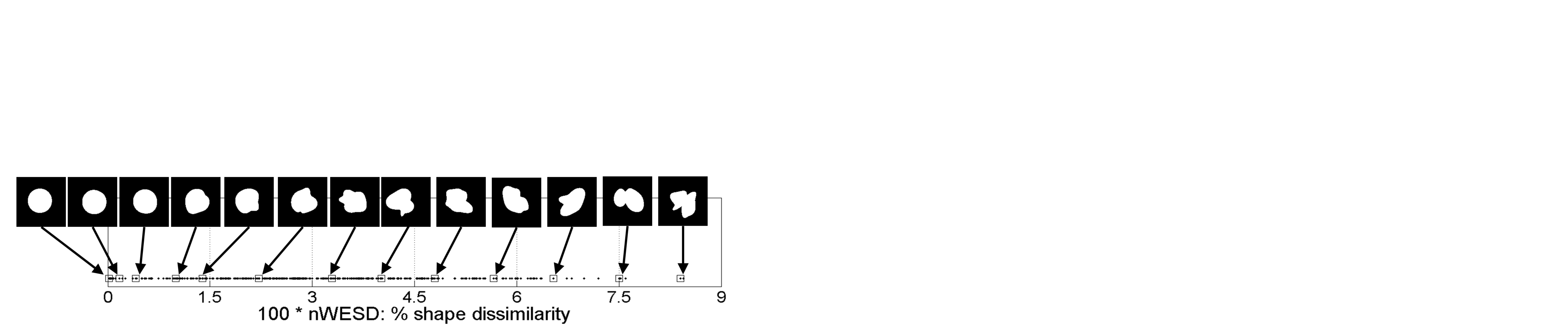}}
\subfigure[]{\includegraphics[width=0.98\linewidth]{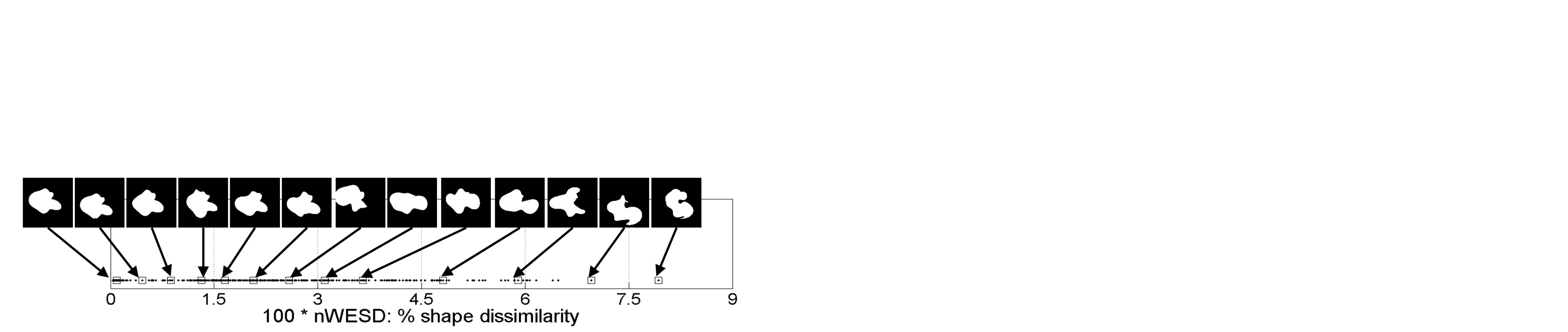}}
\end{center}
\caption{\label{fig:artif_1}Shape-based ordering of objects: We generate two artificial datasets each consisting of a reference object and its random deformations. Samples from the datasets are shown in (a) and (b). The binary images to the very left show the reference objects for each dataset. We then ordered all the deformed objects with respect to the nWESD scores between the object and the reference. The graphs in (c) and (d) plot these orderings. Based on visual inspection the ordering is quite reasonable.}
\end{figure}
For the first experiment we created two synthetic datasets. Each dataset consists of a reference object and random deformations of this reference. These deformed versions are generated by transforming the reference via random deformations of varying magnitude and amount of nonlinearity. As a result the datasets contain objects that are very similar to the reference ones and objects that are substantially different. Figures~\ref{fig:artif_1}(a) and (b) show some examples from these datasets where the binary images to the very left show the reference objects. In the first dataset, the reference object is a disc and in total there are 500 random deformations of this reference disc. The first 400 are generated via non-linear deformations while the last 100 are isometric transformations. In the second dataset, the reference is a slightly more complicated object (see Figure~\ref{fig:artif_1}(b)) and in total there are 400 random transformations of this reference. The first 300 are generated by non-linear deformations and the last 100 produced via isometric transformations. All objects are discretized as binary maps with a size of $200\times200$ pixels. The numerical computations are performed on these image grids as discussed earlier. 

We computed the nWESD scores ($\overline{\rho}^N$ with $p=1.5$ and $N=200$) between the reference and the deformed objects in each dataset. Based on these scores, we then ordered the deformed objects according to their similarity in shape to the reference. Figures~\ref{fig:artif_1} (c) and (d) show examples of the resulting orderings. We notice that the orderings are visually meaningful , i.e. the further the deformed objects visually deviate from the references, the higher their nWESD score is. Furthermore, all the objects generated via isometric transformations yielded scores close to zero as a result of the invariance of the proposed scores to this type of transformation. 
\subsubsection{Low Dimensional Embeddings}\label{sec:wsd_synthetic_2}
In the second experiment we focus on creating low dimensional embeddings. We compare the embeddings constructed by WESD with the ones constructed using $\rho_{SD}^N$ (Equation~(\ref{eqn:r_distance})), the distance proposed in~\cite{Reuter2006}. We do so based on the TOSCA dataset (toolbox for surface comparison and analysis),~\cite{Bronstein2008,Bronstein2008a}. This dataset contains binary segmentations of 5 human, 5 centaurs and 5 horses as shown in Figure~\ref{fig:hmc}(a). We compute the pairwise affinity matrices between objects via $\rho_{SD}^N$ and WESD ($\rho^N$ with $p=1.5$). We then apply the ISOMAP algorithm~\cite{Tenenbaum2000} to these matrices, which maps the 15 objects to a 2D plane based on the pairwise shape distances. We repeat this experiment for affinity matrices computed using different number of spectral modes, i.e. $N=50,100,200$, to demonstrate the effect of the signature size (truncation parameter) on both distances. 
\begin{figure}[!h]
\center
\subfigure[]{\includegraphics[width=0.45\linewidth]{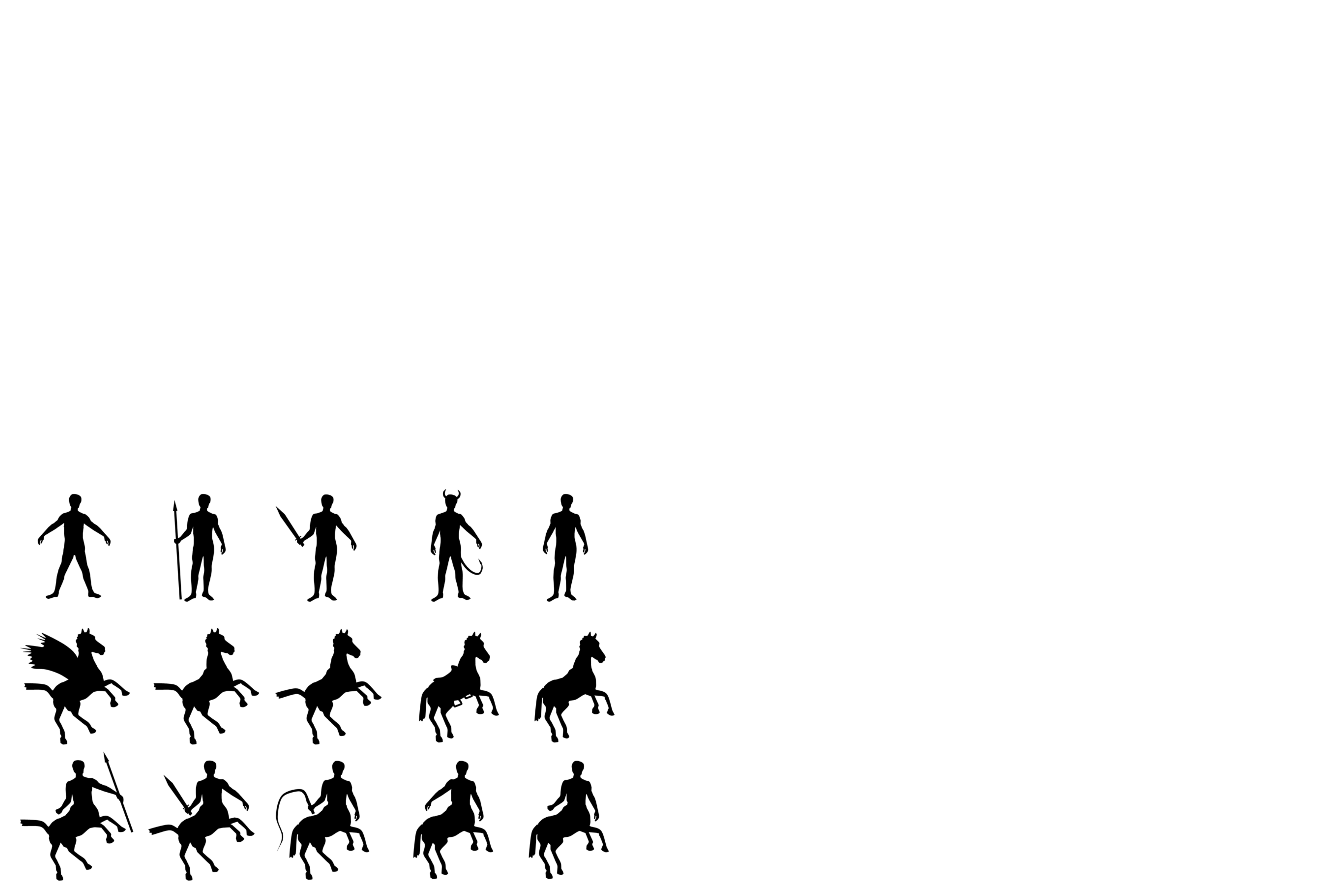}}\\
\subfigure[$\rho_{SD}^N$, N=50]{\includegraphics[width=0.35\linewidth]{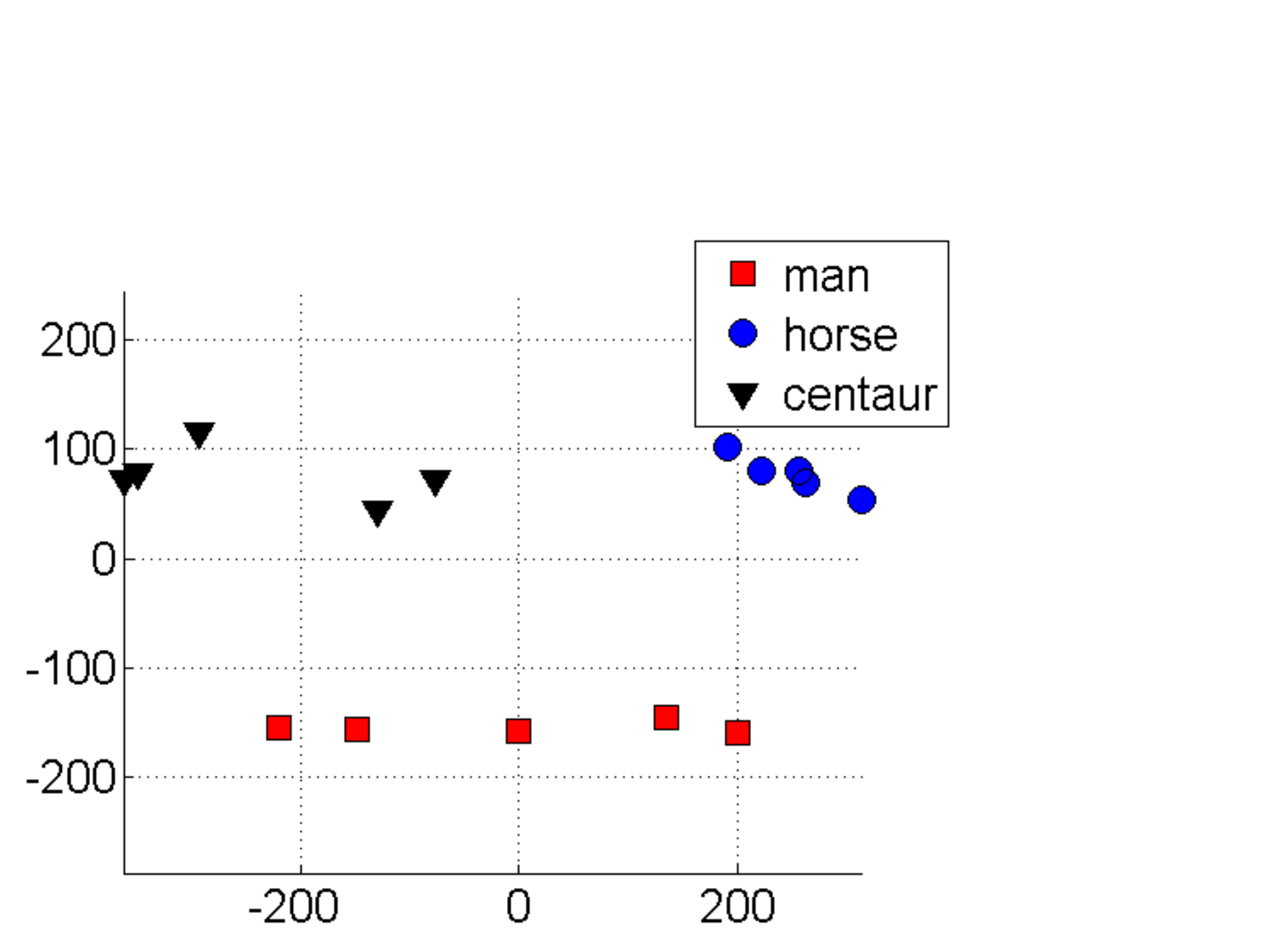}}
\subfigure[$\rho^N$, N=50]{\includegraphics[width=0.35\linewidth]{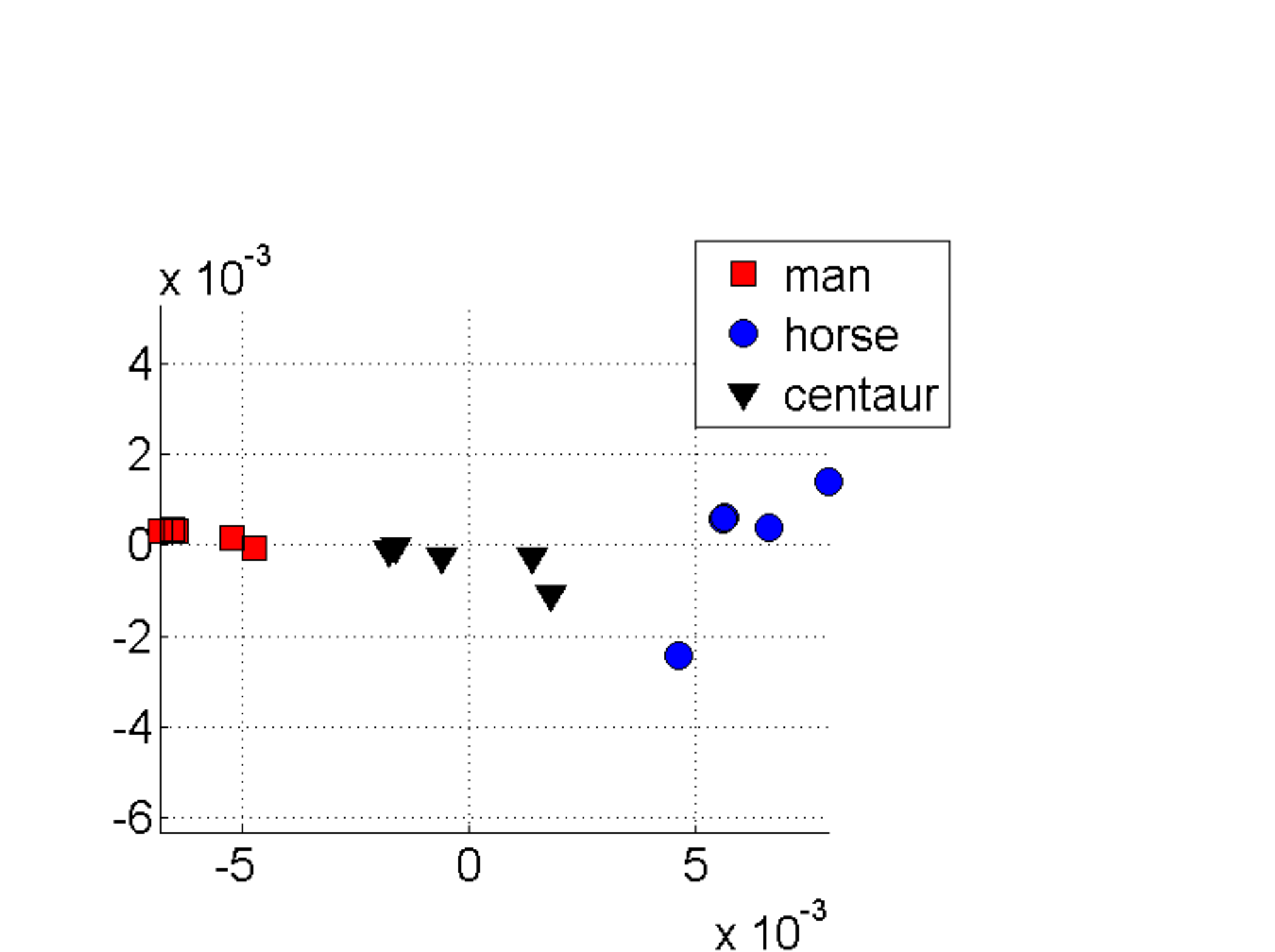}}
\subfigure[$\rho_{SD}^N$, N=100]{\includegraphics[width=0.35\linewidth]{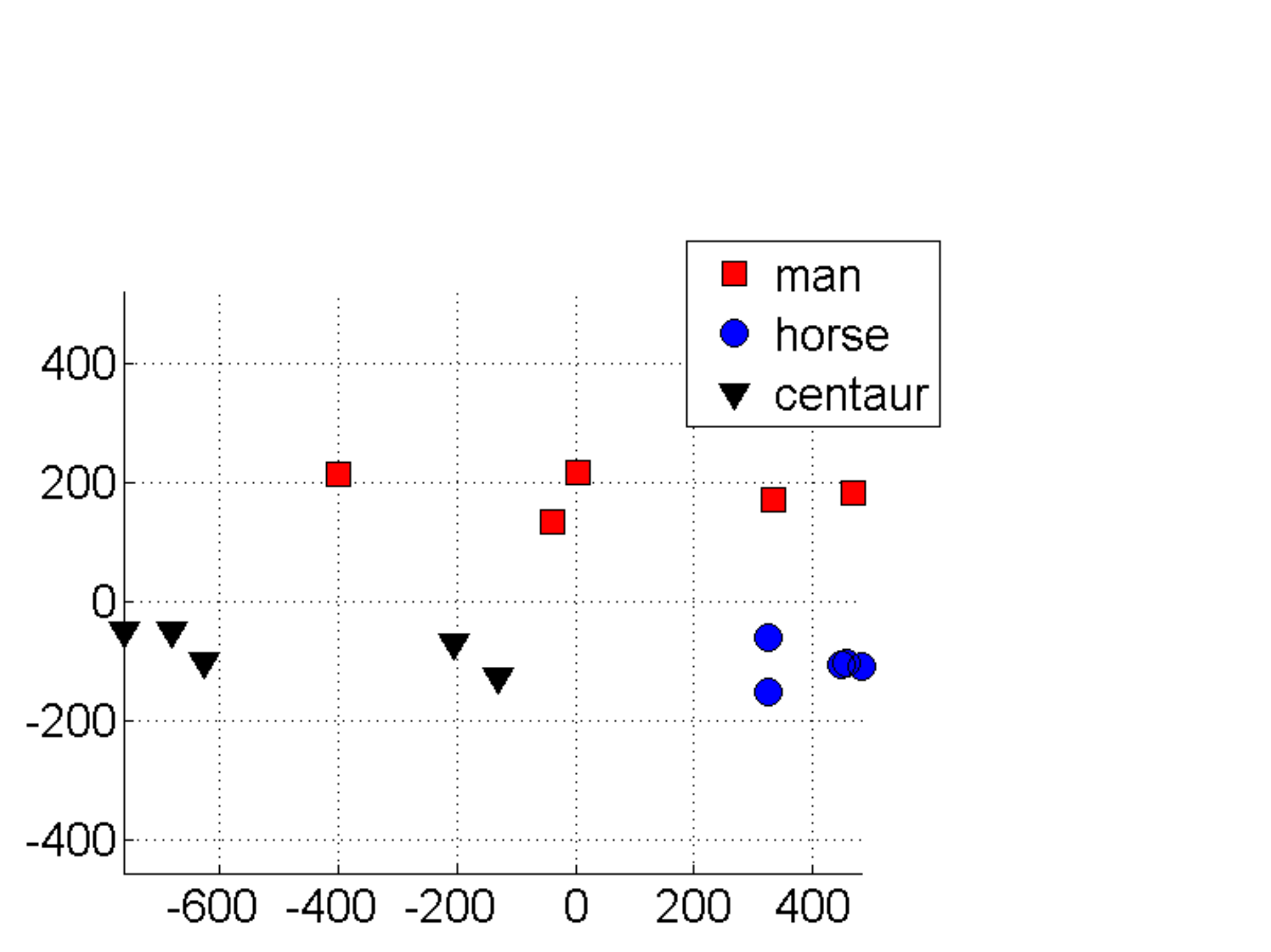}}
\subfigure[$\rho^N$, N=100]{\includegraphics[width=0.35\linewidth]{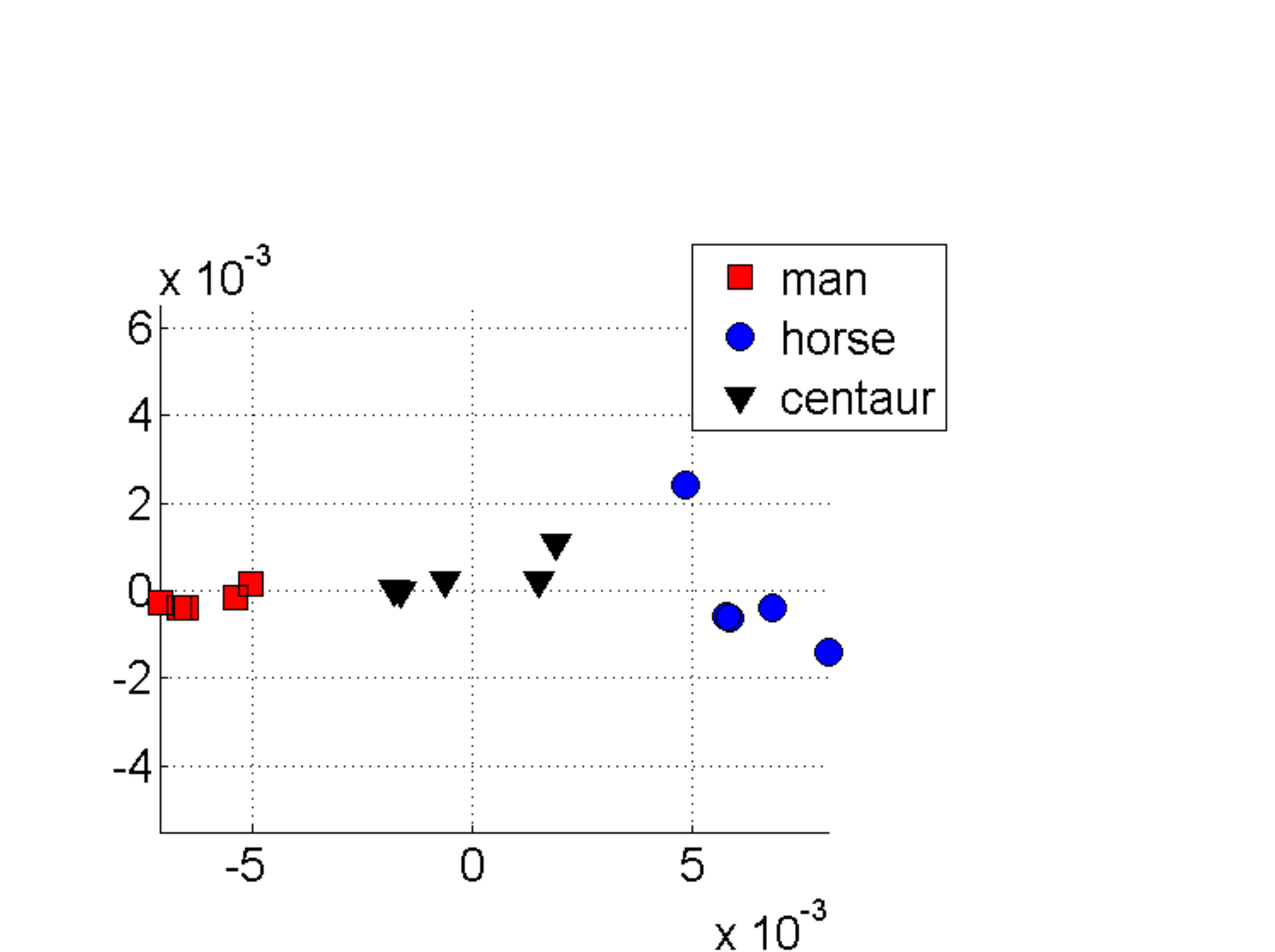}}
\subfigure[$\rho_{SD}^N$, N=200]{\includegraphics[width=0.35\linewidth]{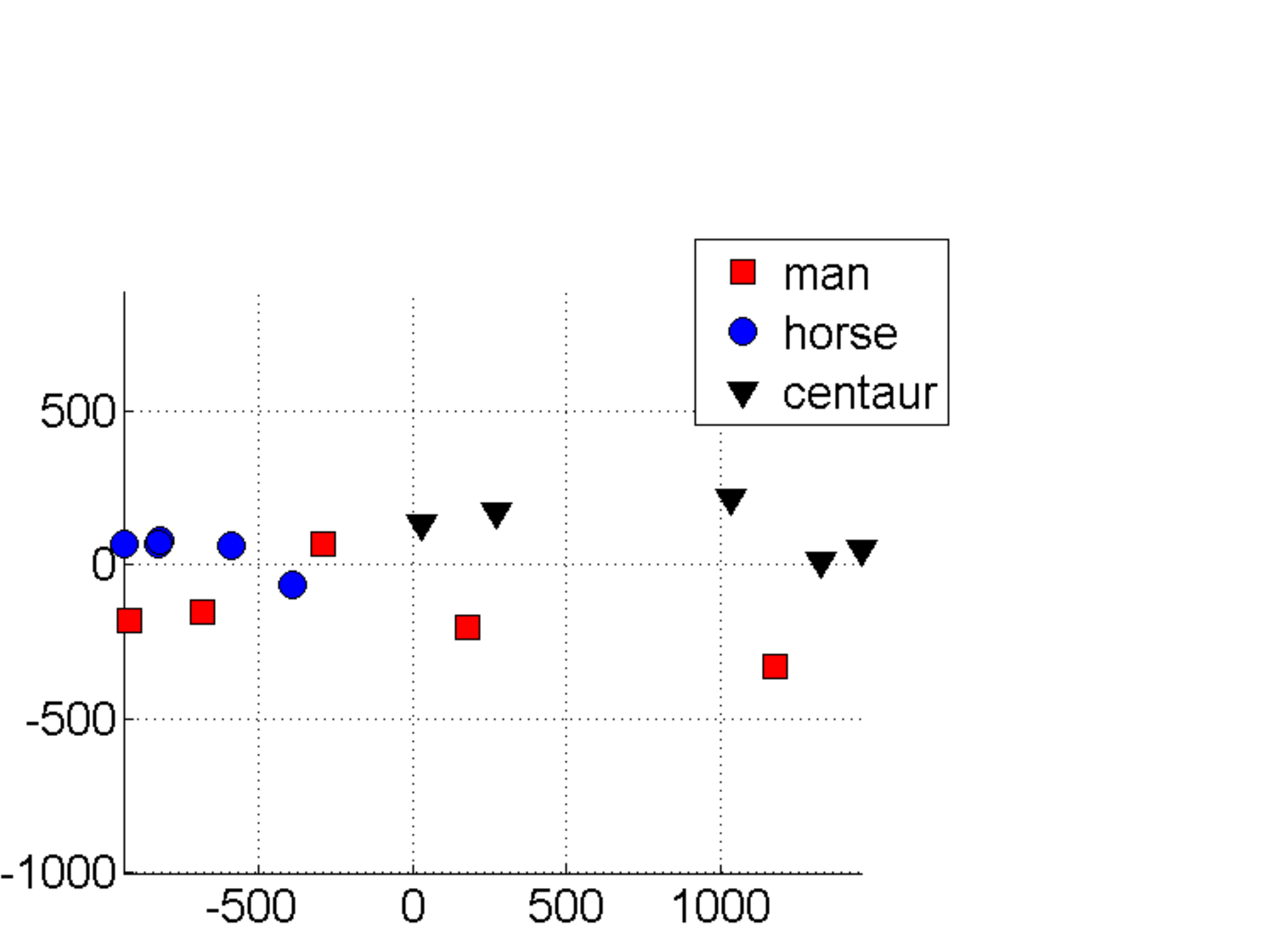}}
\subfigure[$\rho^N$, N=200]{\includegraphics[width=0.35\linewidth]{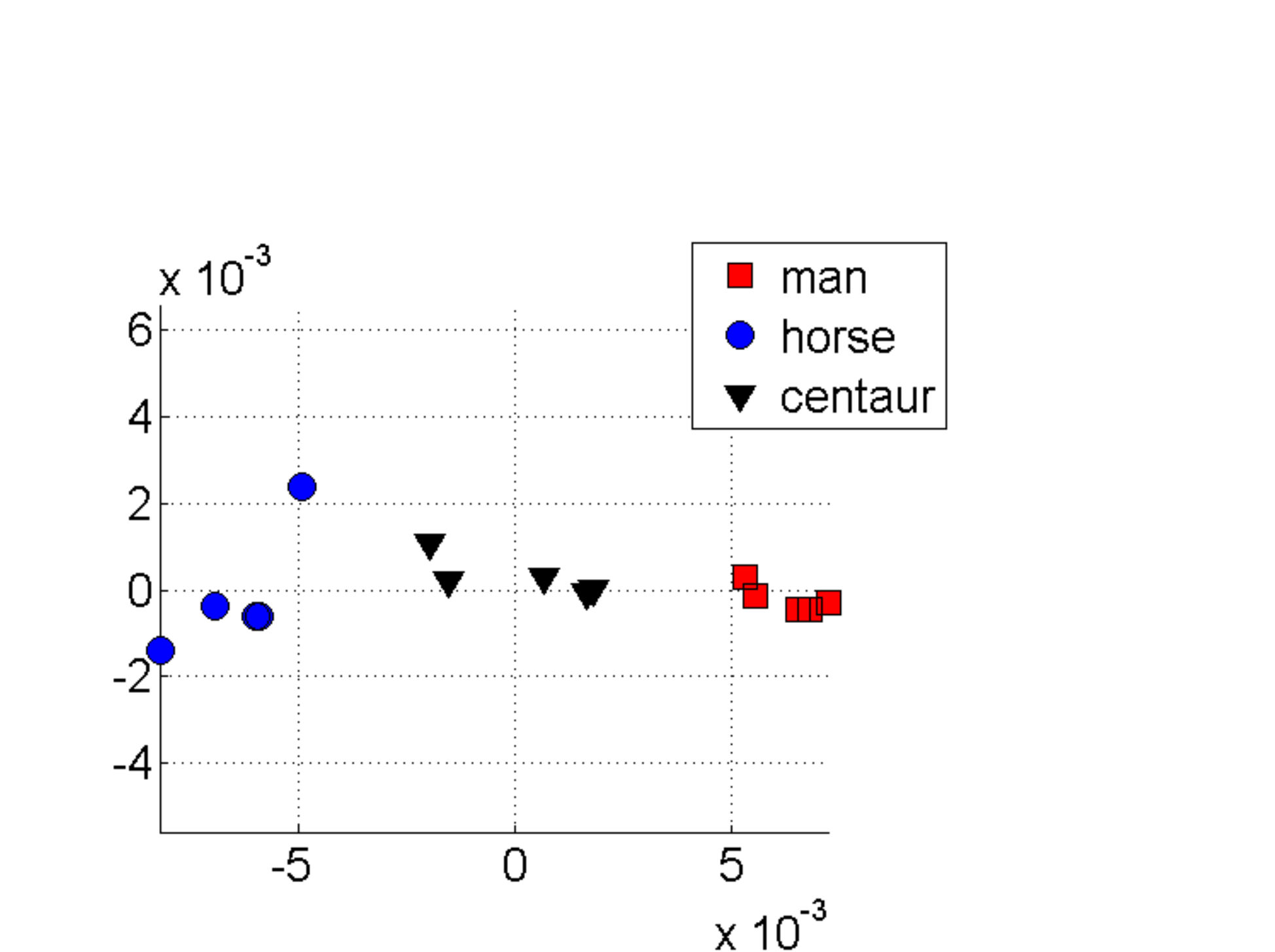}}
\caption{\label{fig:hmc}Low dimensional embeddings: (a) The 15 objects used in this experiment. The graphs plot the 2D embeddings of the objects based on the affinity matrices constructed by $\rho_{SD}^N$ and WESD ($\rho^N$). Each row presents the results based on different number of eigenvalues: 50, 100 and 200 from top to bottom respectively. The structures of the 2D embedding based on $\rho_{SD}^N$ are quite different for different $N$. WESD however, produces embeddings that are similar. This demonstrates the stability of the embedding with respect to $N$ when WESD is used.}
\end{figure}

The plots in Figures~\ref{fig:hmc}(b),(d) and (f) present the resulting 2D embeddings of the dataset using $\rho_{SD}^N$. The embeddings are substantially different for different $N$. This variation arises due to high sensitivity of $\rho_{SD}^N$ towards the signature size $N$. Moreover, the embeddings obtained using higher $N$ are less satisfactory in terms of separating the three different object classes. This is actually as expected since the spectral modes with higher indices dominate the value of $\rho_{SD}^N$ even though they are not informative with regards to the overall geometry and thus, negatively impact the outcome. The plots in Figures~\ref{fig:hmc}(c), (e) and (g) present the embeddings obtained using WESD. Apart from simple coordinate flips (arising from ISOMAP's indifference to signs) the embeddings obtained at different $N$ are very similar. This shows that the construction of the low dimensional embedding is stable with respect to $N$ when WESD is used. This is a direct consequence of the convergent behavior of WESD discussed in Sections~\ref{sec:theo_analysis} and~\ref{sec:choosingN}. As illustrated by this experiment, this property has very important practical implications. 
%%%%%%%%%%%
\subsubsection{{Shape-based Retrieval of 3D Objects}}\label{sec:shrec}

\begin{figure*}[!t]
\begin{center}
\subfigure[]{
{\footnotesize
\begin{tabular}{c|c|c|c|c|c|}
\cline{2-6}
& NN & FT & ST & E & DCG \\ \hline
\multicolumn{1}{|c|}{WESD ($\rho^N$)} & \multirow{2}{*}{0.9933} & \multirow{2}{*}{0.9020} & \multirow{2}{*}{0.9305} & \multirow{2}{*}{0.6900} & \multirow{2}{*}{0.9706} \\
\multicolumn{1}{|c|}{p=3.15, N=100} & & & & & \\ \hline
\multicolumn{1}{|c|}{WESD ($\rho^N$)} & \multirow{2}{*}{0.9933} & \multirow{2}{*}{0.8923} & \multirow{2}{*}{0.9238} & \multirow{2}{*}{0.6824} & \multirow{2}{*}{0.9691} \\
\multicolumn{1}{|c|}{p=2.0, N=100} & & & & & \\ \hline
\multicolumn{1}{|c|}{$\rho_{SD}^N$} & \multirow{2}{*}{0.9967} & \multirow{2}{*}{0.8896} & \multirow{2}{*}{0.9521} & \multirow{2}{*}{0.6959} & \multirow{2}{*}{0.9748}\\
\multicolumn{1}{|c|}{N=12, norm1} & & & & & \\\hline
\multicolumn{1}{|c|}{$\rho_{SD}^N$} & \multirow{2}{*}{0.9917} & \multirow{2}{*}{0.9153} & \multirow{2}{*}{0.9569} & \multirow{2}{*}{0.7047} & \multirow{2}{*}{0.9783} \\ 
\multicolumn{1}{|c|}{N=12, normA} & & & & & \\\hline
\multicolumn{1}{|c|}{$\rho_{SD}^N$} & \multirow{2}{*}{0.9933} & \multirow{2}{*}{0.8683} & \multirow{2}{*}{0.9431} & \multirow{2}{*}{0.6895} & \multirow{2}{*}{0.9705} \\ 
\multicolumn{1}{|c|}{N=15, norm1} & & & & & \\\hline
\end{tabular}}}
\subfigure[$p=3.15$ and $N=100$]{\includegraphics[width=0.65\linewidth]{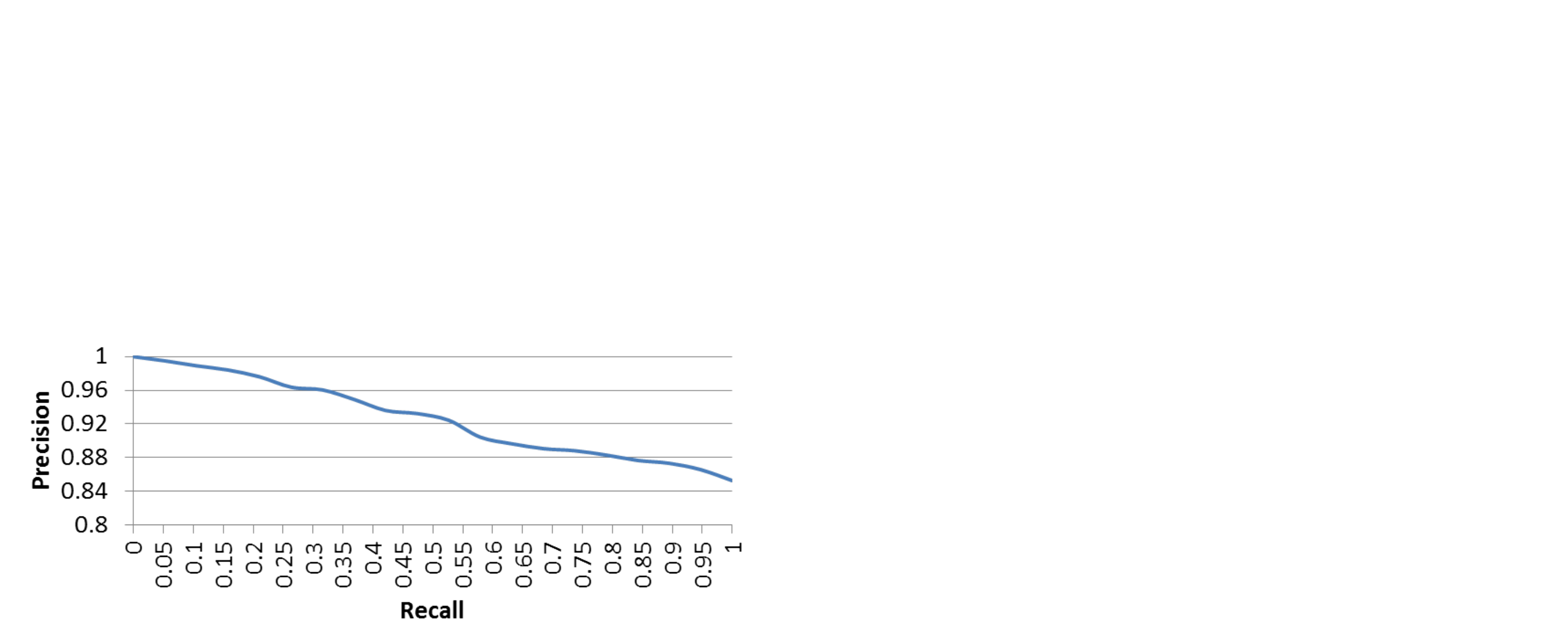}}
\subfigure[$p=3.15$]{\includegraphics[width=0.37\linewidth]{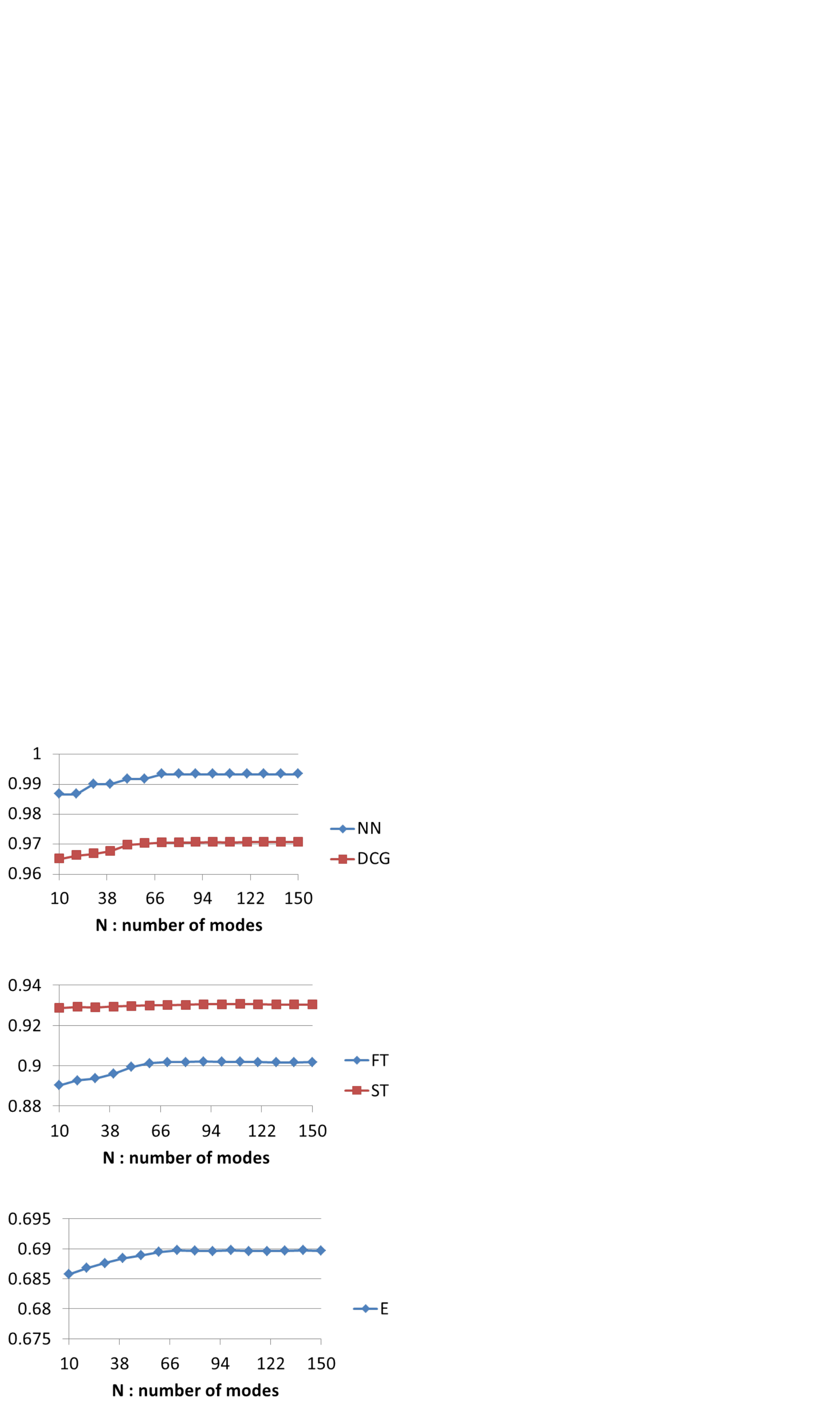}}
\subfigure[$N=100$]{\includegraphics[width=0.37\linewidth]{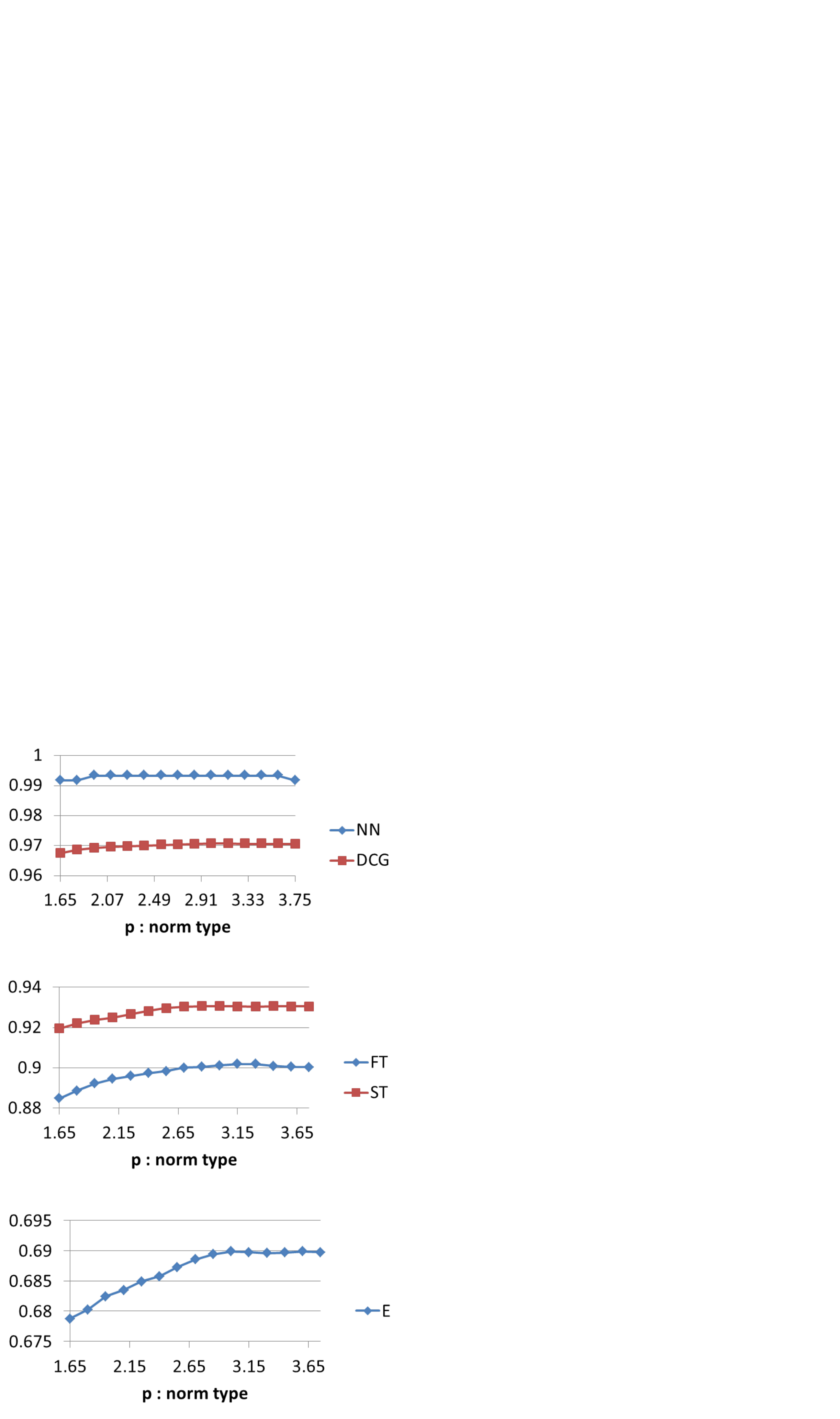}}
\end{center}
\caption{\label{fig:shrec}{Shape-based Object Retrieval Results on SHREC Dataset. a) Retrieval scores obtained by WESD for two different sets of $N$ and $p$ values along with the scores obtained by the distance proposed in \cite{Reuter2006} (values taken from \cite{Lian2011}.) b) Precision-Recall curves obtained for shape retrieval via WESD for the entire dataset. c) Effect of the signature size $N$ on the retrieval scores obtained by WESD for a fixed $p=3.15$. d) Effect of the norm type $p$ on the same scores for a fixed $N=100$.}}
\end{figure*}

{In this last experiment with synthetic data, we focus on the application of {\em shape-based object retrieval}, i.e. given a test object identifying other ``similar'' objects within a dataset using shape information. Similarity in this context can be defined in various ways but the definition used here is semantic similarity, meaning that objects that are of the same semantic category (e.g. human bodies, aeroplanes, etc) are similar and objects of different categories are not. Shapes of similar objects have similar traits and properties. Shape distances used for retrieval purposes should be able to capture these traits yielding the lowest values between similar object pairs. Here, WESD's value for shape-based retrieval is evaluated using the publicly available dataset SHREC presented in \cite{Lian2011}\footnoteremember{shrecfootnote}{Available at  http://www.itl.nist.gov/iad/vug/sharp/contest/2011/NonRigid/}.}

{SHREC dataset consists of 600 3D non-rigid objects from 30 different categories, i.e. 20 objects per category. Objects from the same category differ with substantial non-linear deformations, which makes retrieval in this dataset challenging. To evaluate the retrieval accuracy of WESD, first each object was converted from its original watertight surface mesh discretization to a 3D binary image using the Iso2mesh software package\footnote{http://iso2mesh.sourceforge.net/cgi-bin/index.cgi}.
Then pairwise shape distances across the entire dataset were computed using WESD and the $600\times600$ affinity matrix was constructed, where each entry is a pairwise distance. This affinity matrix was then evaluated using the software provided with the dataset\footnoterecall{shrecfootnote}. The evaluation consists of a variety of retrieval accuracy scores such as Nearest Neighbor (NN), First-Tier (FT), Second-Tier (ST), E-Measure (E), Discounted Cumulative Gain (DCG) and Precision-Recall curve. The first two rows of the table in Figure~\ref{fig:shrec}(a) list these scores obtained using WESD for two different settings of the $p$ and $N$ values. Additionally, the last three rows of the same table show the results obtained using $\rho_{SD}^N$ (Equation~\ref{eqn:r_distance},~\cite{Reuter2006}), as listed in \cite{Lian2011}\footnote{We note that for these latter results a slightly different notation is used here than in \cite{Lian2011} to conform to the overall notation of this article.}, using two different types of scale normalisation (norm1: normalising with respect to the first eigenvalue, normA: area normalisation, see Section~\ref{sec:scale} for further details). These accuracy scores show that WESD and $\rho_{SD}^N$ perform very similar in retrieval from the SHREC dataset. Furthermore, Figure~\ref{fig:shrec}(b) shows the precision-recall curve of WESD ($p=3.15$ and $N=100$) for the entire dataset. The curve is very similar to the best curve obtained using $\rho_{SD}^N$ shown in \cite{Lian2011}. Once again, this confirms that both distances perform similarly.}

{Lastly, the graphs shown in Figure\ref{fig:shrec}(c) and (d) provide an analysis of the retrieval results with respect to the parameters $N$ and $p$. Graphs in Figure~\ref{fig:shrec}(c) plot the change of different retrieval scores with respect to the number of modes used $N$, i.e. signature size, keeping $p$ fixed at $3.15$. Graphs in Figure~\ref{fig:shrec}(d) plot the changes with respect to the norm type $p$ keeping $N$ fixed at $100$. These graphs show that as $N$ increases the scores seem to increase slowly and then converge. On the other hand, $p$ has a stronger effect on the results than $N$, particularly on FT, ST and E scores. However, the changes in the scores with respect to changes in $N$ or $p$ are rather small especially compared to the relatively larger fluctuation of the FT score of $\rho_{SD}^N$ with respect to the two sample $N$ values provided in the table in Figure~\ref{fig:shrec}(a).}
 
{The experiment presented above showed that the retrieval power of WESD is similar to that of the distance $\rho_{SD}^N$ proposed by Reuter \etal \cite{Reuter2006}. The soundness and theoretical properties of WESD do not come at the expense of lower retrieval power. On the contrary, WESD is able to leverage the descriptive power of the spectra while its properties guarantee that it does not suffer from similar drawbacks as other distances, such as sensitivity to signature size.}
\subsection{Real Data}~\label{sec:wsd_med_experiments}
\begin{figure*}[!h]
\begin{minipage}[b]{0.48\linewidth}
\subfigure[Four Hippocampi]{\includegraphics[width=0.9\linewidth]{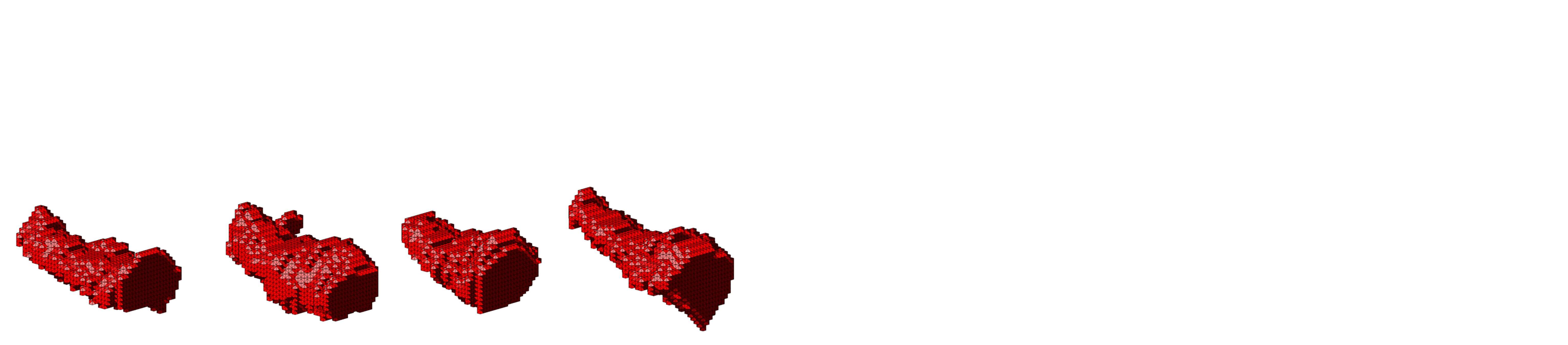}}
\subfigure[Four Caudate Nuclei]{\includegraphics[width=0.9\linewidth]{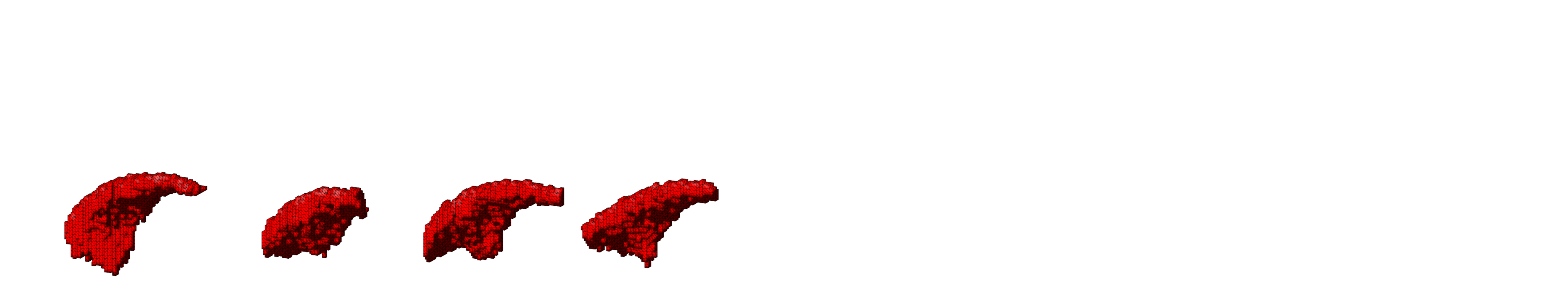}}
\subfigure[Four Putamen]{\includegraphics[width=0.9\linewidth]{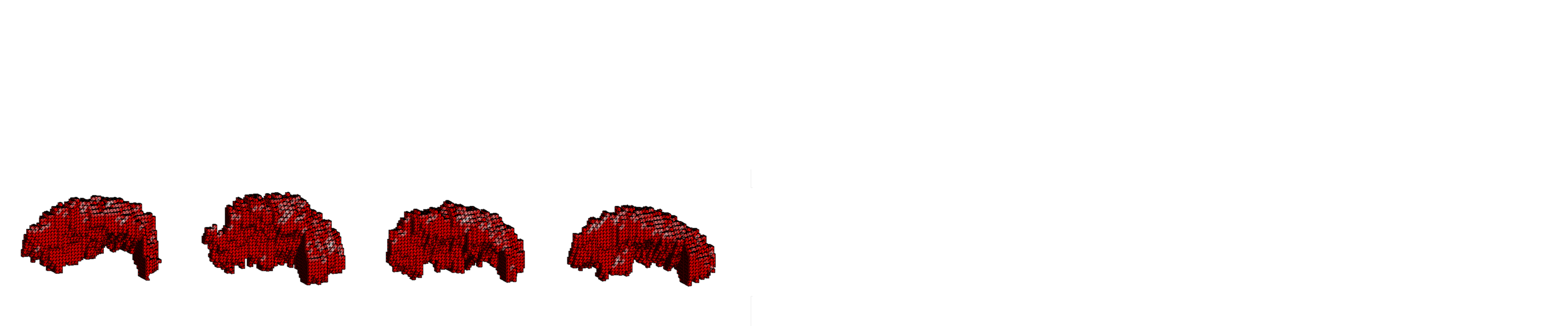}}
\end{minipage}
\begin{minipage}[b]{0.52\linewidth}
\subfigure[$\rho_{SD}^N$ - no preprocessing]{\includegraphics[width=0.48\linewidth]{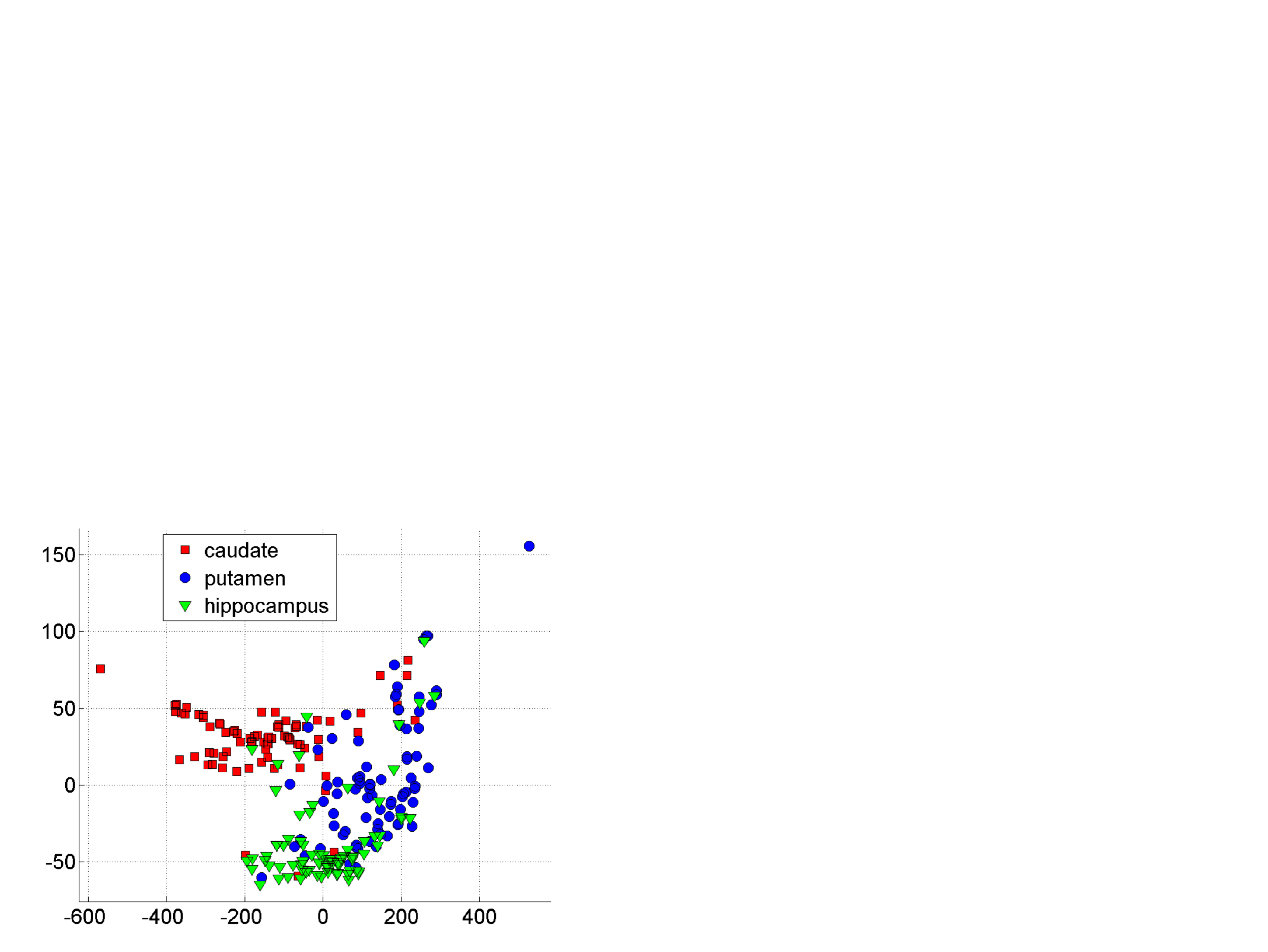}}
\subfigure[WESD - no preprocessing]{\includegraphics[width=0.48\linewidth]{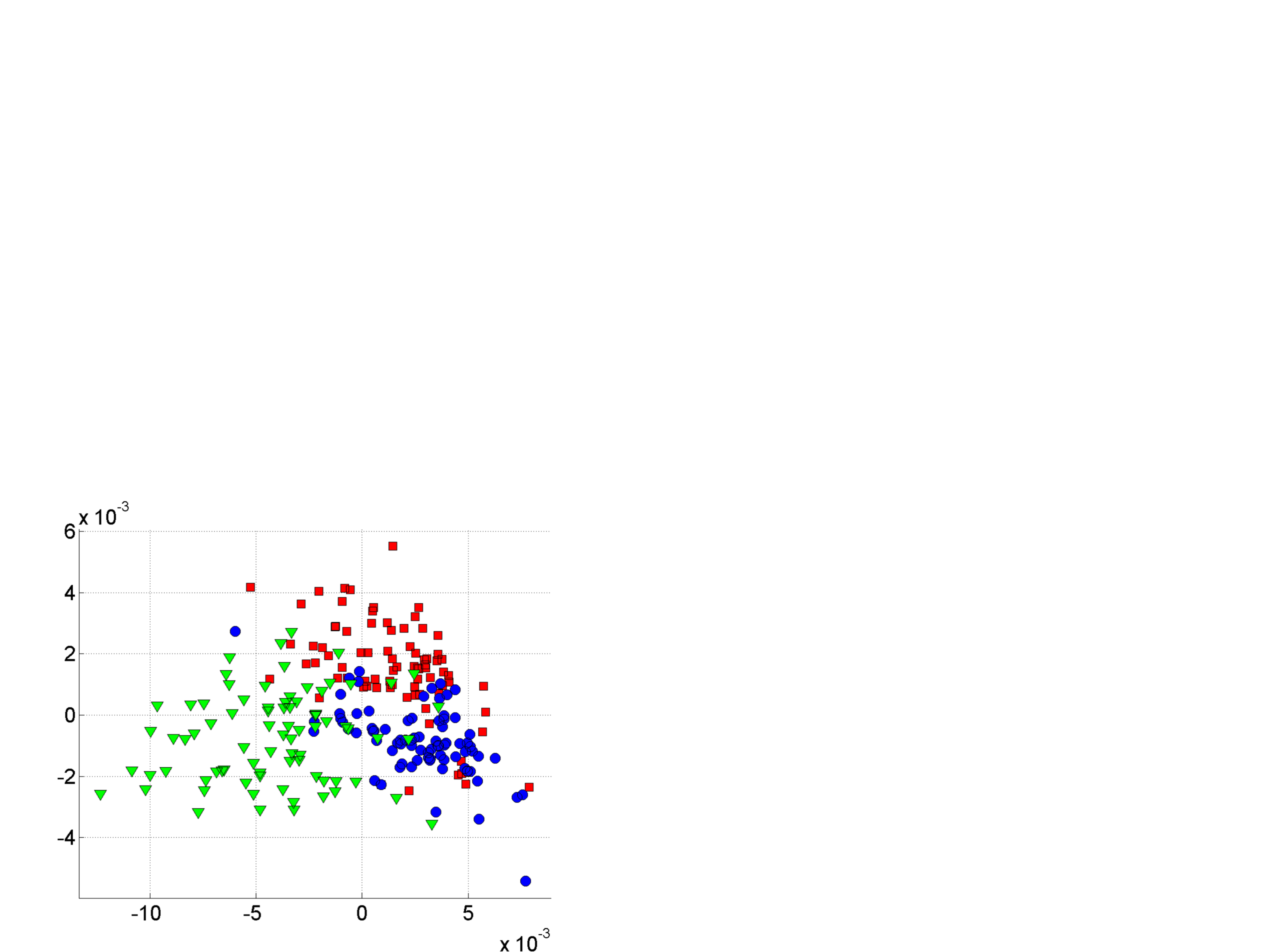}}
\subfigure[$\rho_{SD}^N$ - surface smoothing]{\includegraphics[width=0.48\linewidth]{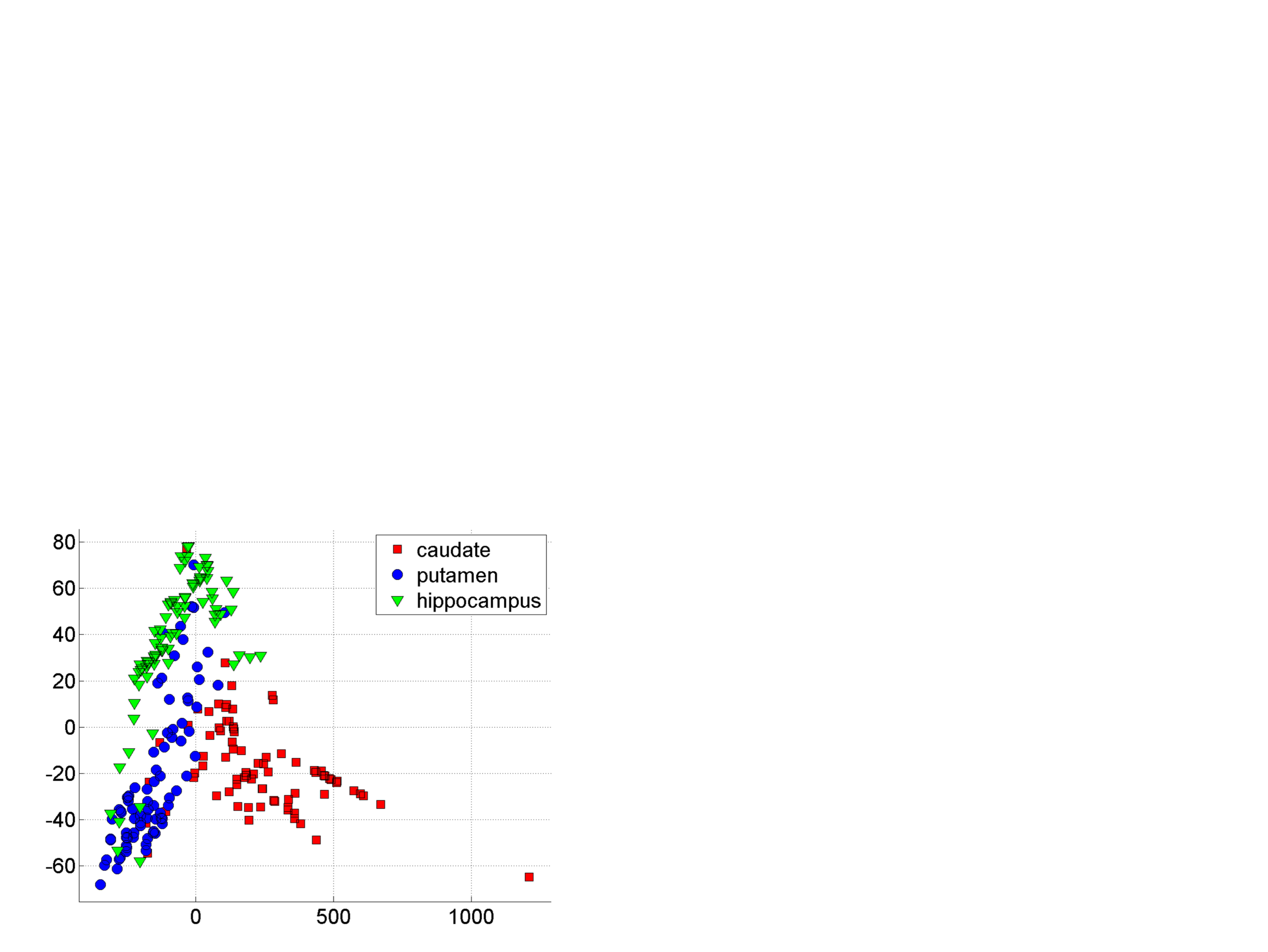}}
\subfigure[WESD - surface smoothing]{\includegraphics[width=0.48\linewidth]{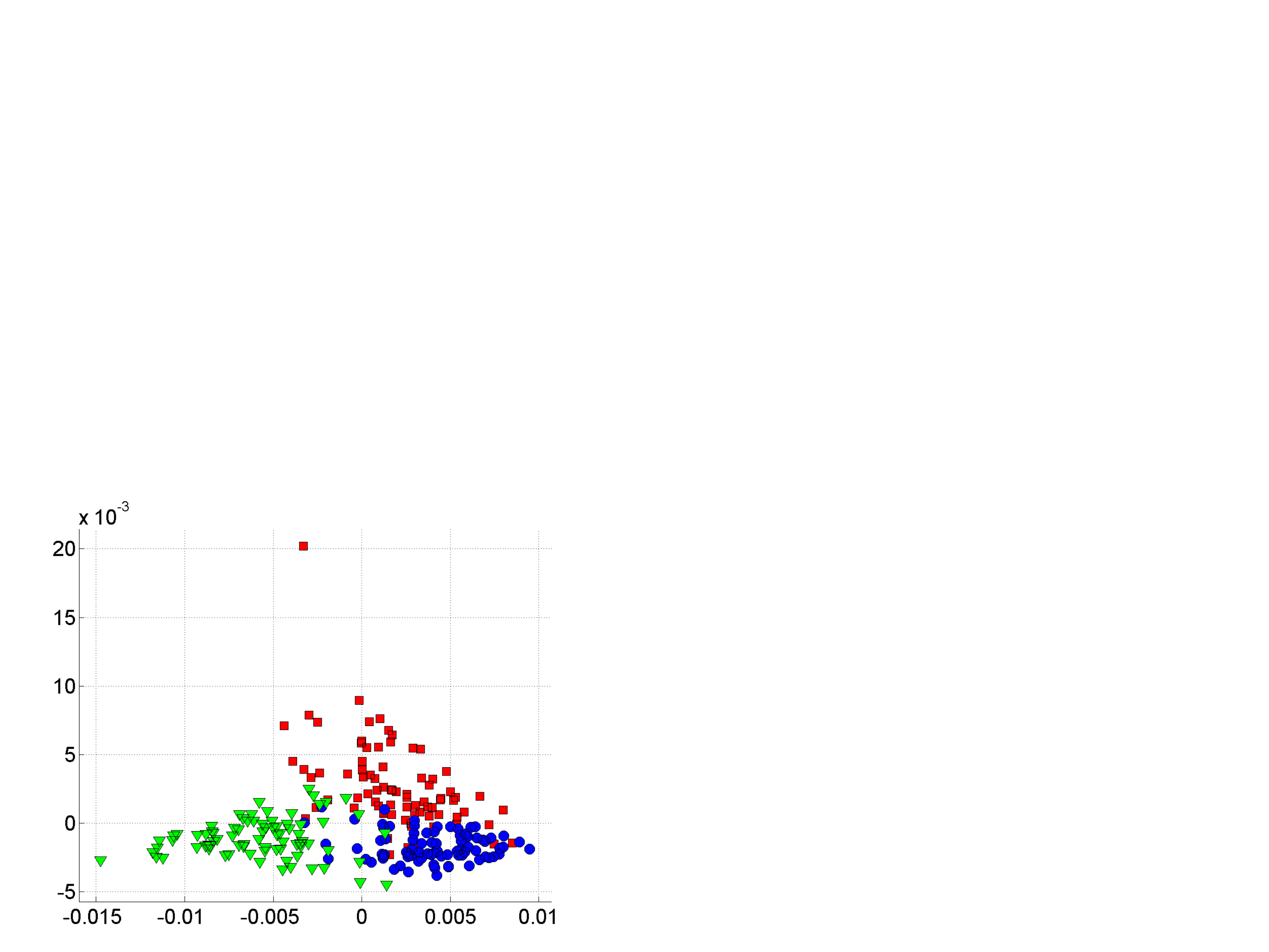}}
\end{minipage}
\caption{\label{fig:subcortical}2D embedding of subcortical structures: 240 structures (80 caudate nucleus, 80 putamen and 80 hippocampus) are extracted from MR scans of 40 different individuals. (a),(b) and (c) show some example structures from this dataset. Note the high intra-class variability and the artefacts due to finite resolution and manual segmentations. (d) and (e) plot 2D embeddings of these 240 structures obtained based on the affinity matrices computed via $\rho_{SD}^N$ and WESD respectively. These embeddings are computed without any preprocessing applied to the structures. The embedding obtained with WESD distinctly clusters the objects with respect to the anatomical structures. The embedding in (d) however, shows some ambiguities in the separation. Graphs in (f) and (g) plot the similar embeddings obtained after smoothing the surfaces of the structures to remove artefacts. The embedding obtained by $\rho_{SD}^N$, although better than (d), still suffer from similar problems. The embedding based on WESD on the other hand, now even between better separates the groups.}
\end{figure*}
The experiments on real data are conducted on segmentations of 3D structures obtained from magnetic resonance images (MRI). First, we apply WESD to subcortical brain structures. The experiment demonstrates WESD's capabilities to differentiate categories of objects even in the presence of high intra-class variability. In the second experiment, we focus on temporal analysis of cardiac images. We apply nWESD to delineations of the blood pool of the left ventricle obtained from 3D + time cardiac MRI. The experiment shows that the shape dissimilarity measurements between time points correlates with the dynamic processes of the beating heart.

\subsubsection{Clustering Sub-Cortical Structures}~\label{sec:subcortical}
Medical research frequently relies on morphometric studies analysing anatomical shapes from medical images~\cite{Bookstein2001a}. In this experiment we construct a low dimensional embedding of subcortical structures extracted from Magnetic Resonance Image (MRI) scans of different individuals based on WESD as well as shape-DNA based distance, $\rho_{SD}^N$, as proposed in~\cite{Reuter2006}. 

For this experiment, we use the publicly available LPBA40 dataset~\cite{Shattuck2008}\footnote{website:http://www.loni.ucla.edu/Atlases/LPBA40}. The dataset contains manual segmentations of various subcortical structures from MRI brain scans of 40 healthy subjects. Figures~\ref{fig:subcortical}(a), (b) and (c) show some examples from these structures. The are two main difficulties associated with such datasets. First, the structures have very large intra-class (inter-subject) variability, i.e. the shape of an anatomical structure is often very different across subjects. Second, the segmentations were obtained by manually delineating the 3D objects on successive 2D slices. This creates inconsistencies between segmentations in two successive slices. Such inconsistencies in the end manifest themselves as local artefacts on the object. The protrusion that can be seen on the top of the second hippocampus in Figure~\ref{fig:subcortical}(a) is an example of such an artefact. These artefacts can influence shape distances negatively. 

We select six structures for each patient: left/right caudate nucleus, left/right putamen and left/right hippocampus, resulting in 240 structures in total. We then create pairwise affinity matrices of the 240 structures first using $\rho_{SD}^N$ with $N = 200$, as proposed in~\cite{Reuter2006}, and then WESD ($\rho^N$ with $p=2$ and $N=200$). Finally, we use the ISOMAP algorithm~\cite{Tenenbaum2000} to construct 2D embeddings of the structures. Figures~\ref{fig:subcortical}(d) and (e) show the resulting embeddings. We observe that the embedding obtained via WESD well clusters the data with respect to the anatomical structures. The separation of the clusters for the SD case, however, is more ambiguous, especially between putamen and hippocampus.

The embeddings presented above were obtained by directly using the manual segmentations without any preprocessing. A natural question is how do these embeddings change if the effects of various artefacts are reduced say via surface smoothing. To answer this question, we smooth the surface of the anatomical 3D models and recomputed the embeddings, which are shown in Figures~\ref{fig:subcortical}(f) and (g). The embedding obtained with $\rho_{SD}^N$, although to a lesser extent, still suffers from similar ambiguity as in Figure~\ref{fig:subcortical}(d). The new embedding based on WESD on the other hand, compared to Figure~\ref{fig:subcortical}(e), even more clearly separates different anatomical structures. However, we also note that this type of preprocessing can also produce undesirable artefacts such as altering the topology of the anatomical object. This is the case for one caudate in Figures~\ref{fig:subcortical}(f) and (g), which ends up as an outlier that is clearly separated from the other data points. Considering this, the fact that WESD is able to produce visually pleasing embeddings (see Figure~\ref{fig:subcortical}(e)) without the need of preprocessing is an advantage. 
%---------------%
\subsubsection{Analysing Heart Function in 4D MRI}~\label{sec:cardiac}
Four-dimensional imaging of patient anatomy is gaining interest in the medical community. The temporal analysis of anatomical structures is used to extract the characteristics of related dynamic processes, which often indicate certain pathologies~\cite{Gotardo2006,Mansi2011,Bernardis2012}. In this section, we apply nWESD to the shapes of the hearts extracted from four dimensional cardiac images of five different patients. The scan of each patient captures a full cycle of one heartbeat as a series of 20 3D images. Each image shows the left ventricle (LV) at a specific point in the cycle, from which we manually segment the corresponding blood pool. Our reference is the blood pool extracted from the first frame (diastole). We compute the nWESD scores between this reference and all other shapes extracted from the series of images. Here, we do not normalise the eigenvalues with respect to the global scale since size change is an important aspect of the heartbeat dynamics. The graph given in Figure~\ref{fig:cardiac} shows the results of these measurements over time across the five patients. The figure also shows some exemplary images and shapes. We observe that the symmetry of the heartbeat along the systolic (as the blood pumps out of the LV pool) and the diastolic phases (as the blood fills in the pool) is well captured with the nWESD score. Furthermore, the end-systolic phase (the time point with the largest distance w.r.t. the reference) is at different time points for different patients, which is to be expected since the different patient scans are not synchronized in time. In summary, WESD well captures the dynamics of the beating heart, which is to be expected given the continuous link between the differences in eigenvalues and the difference in shape (see Section \ref{sec:solo}).  
\begin{figure}[!h]
\center
\subfigure{
 \includegraphics[width=0.13\linewidth]{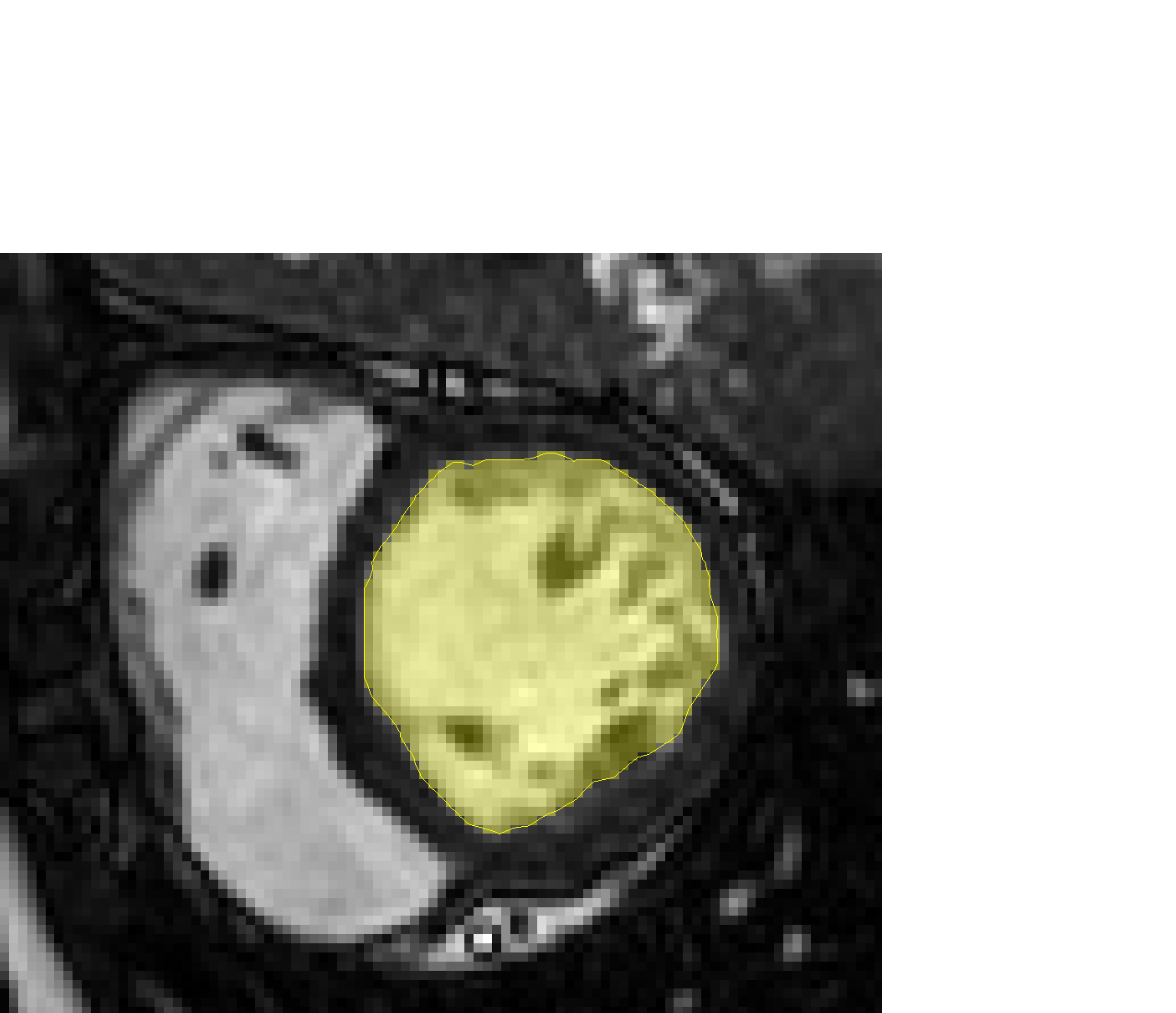}}
\subfigure{
 \includegraphics[width=0.13\linewidth]{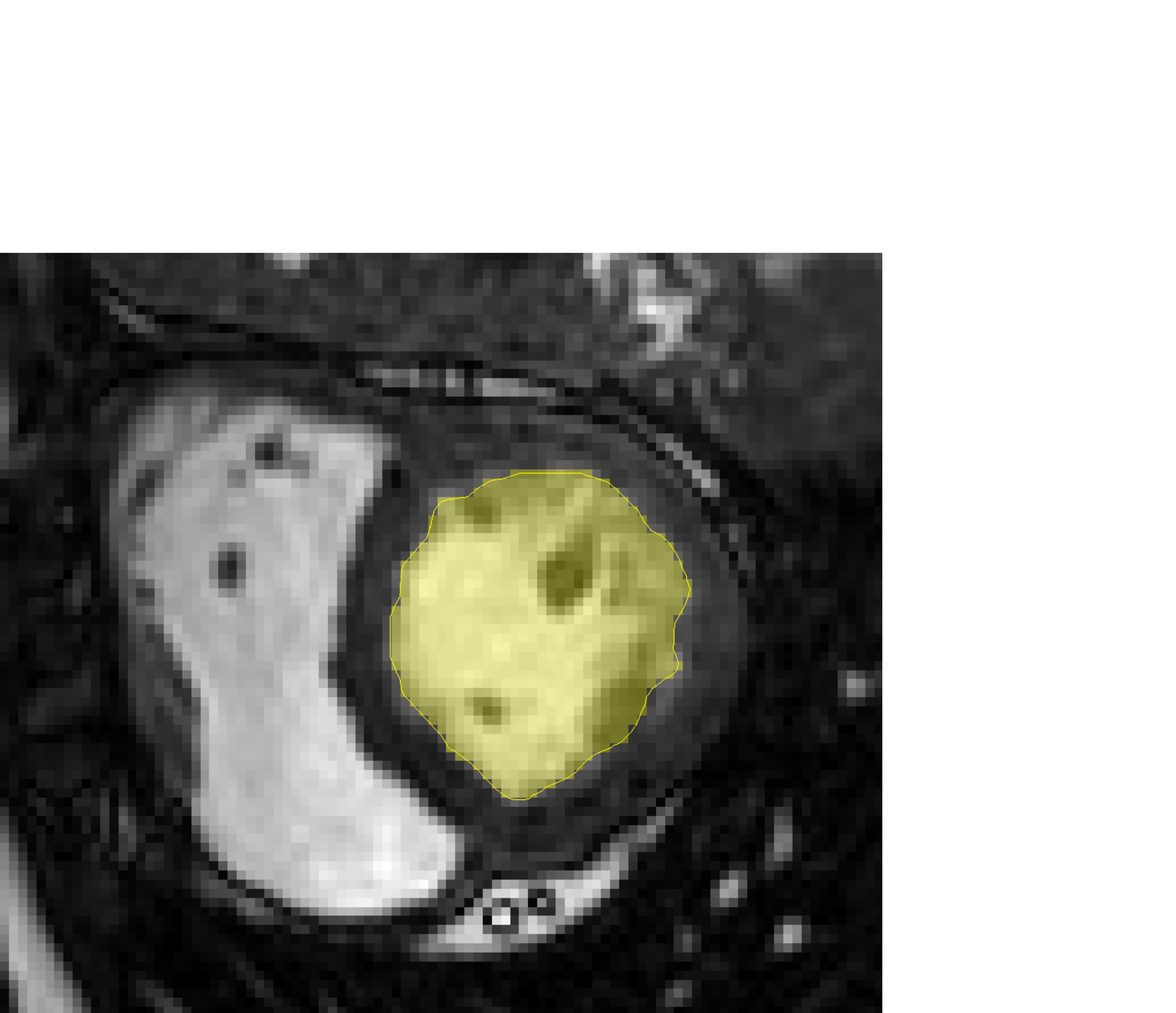}}
\subfigure{
 \includegraphics[width=0.13\linewidth]{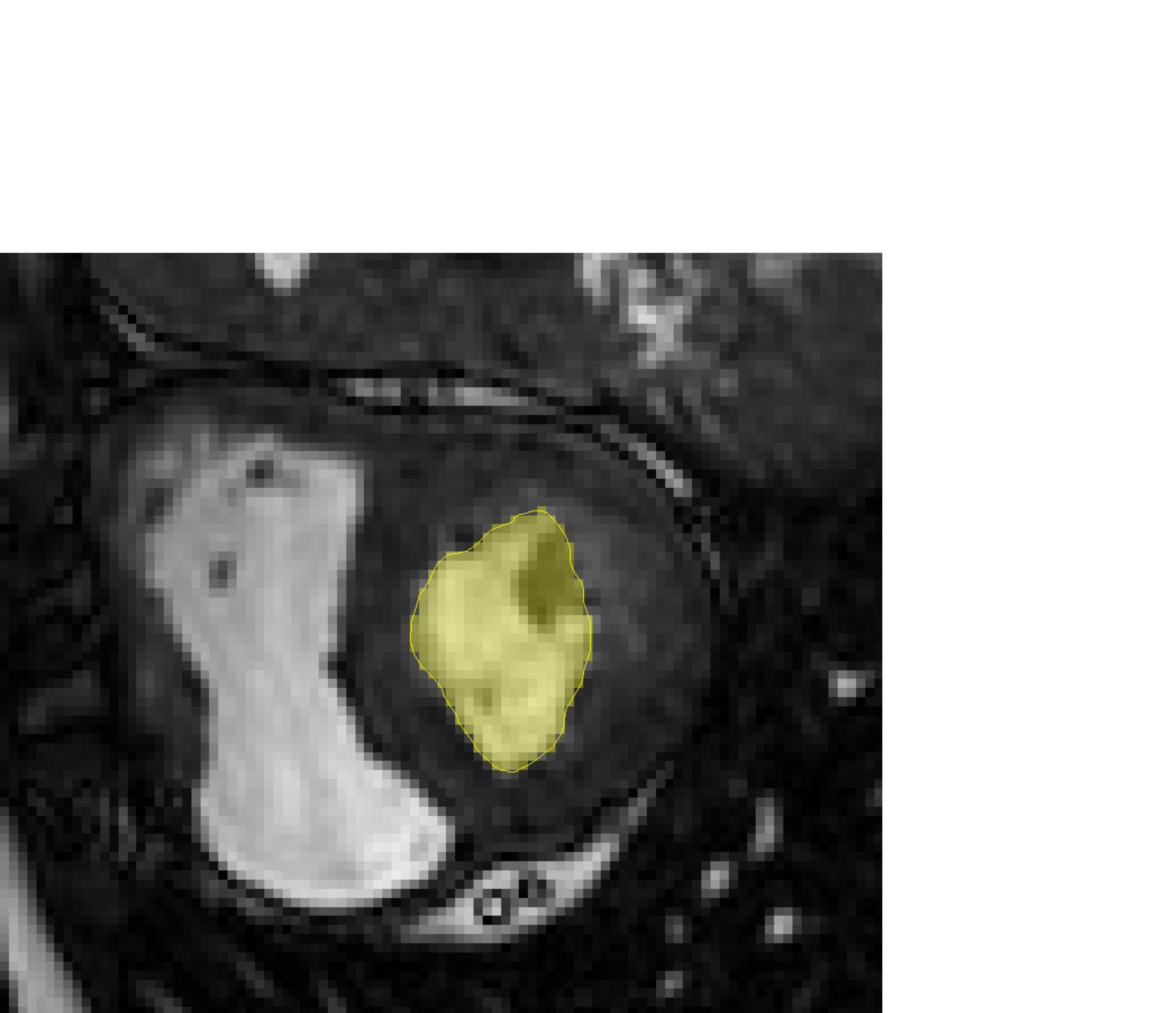}}
\subfigure{
 \includegraphics[width=0.13\linewidth]{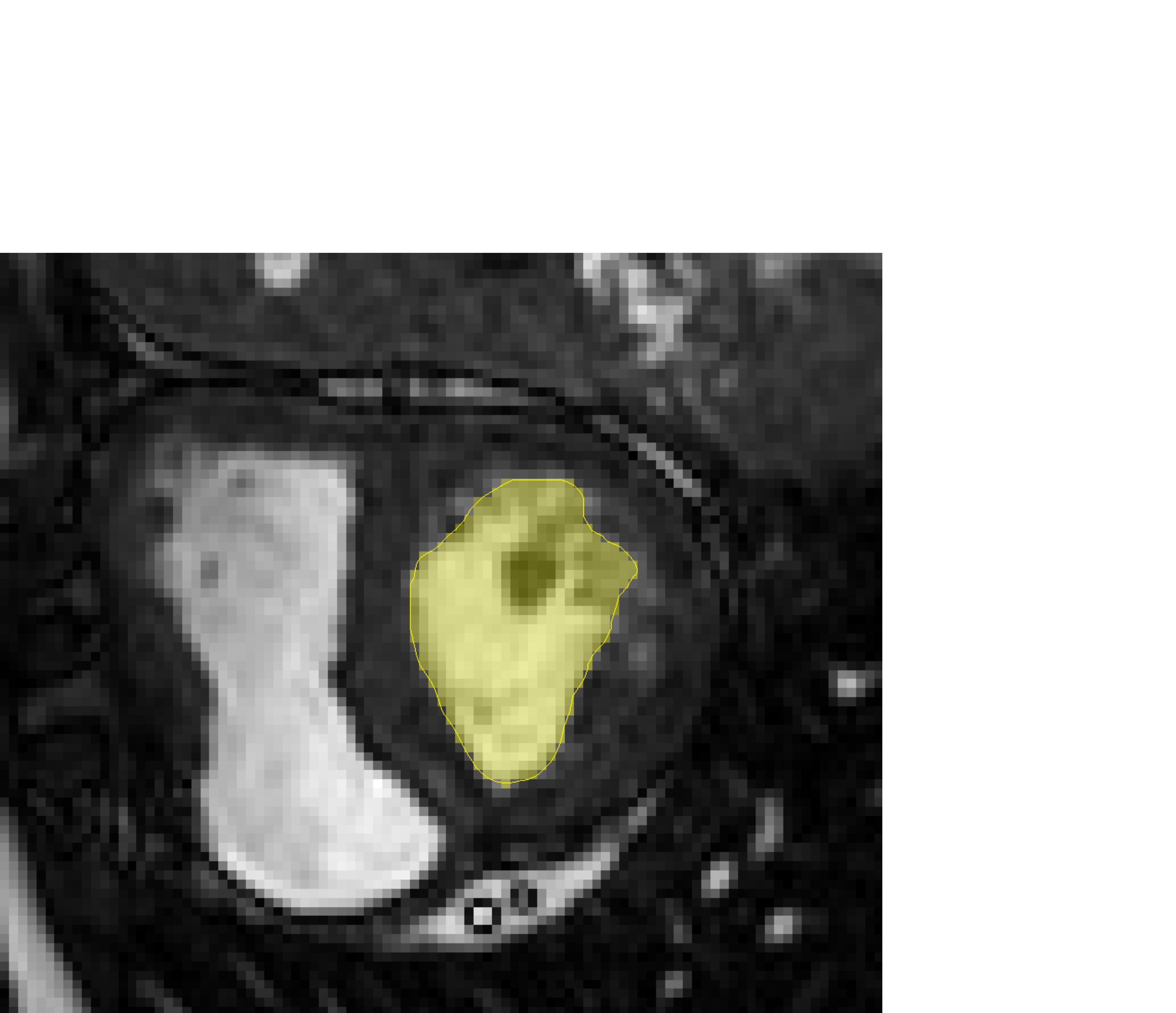}}
\subfigure{
 \includegraphics[width=0.13\linewidth]{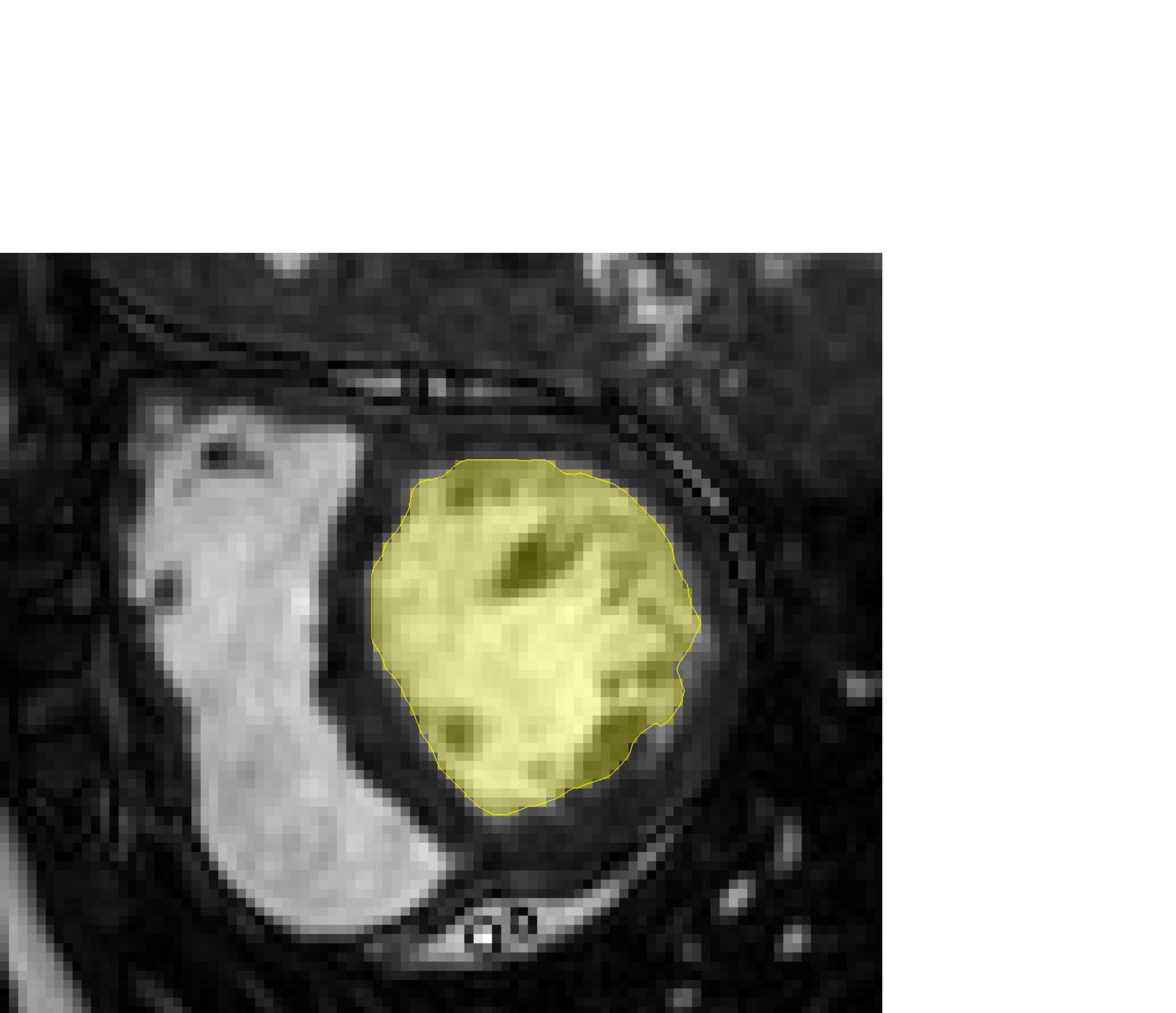}}\\
\subfigure{
 \includegraphics[width=0.13\linewidth]{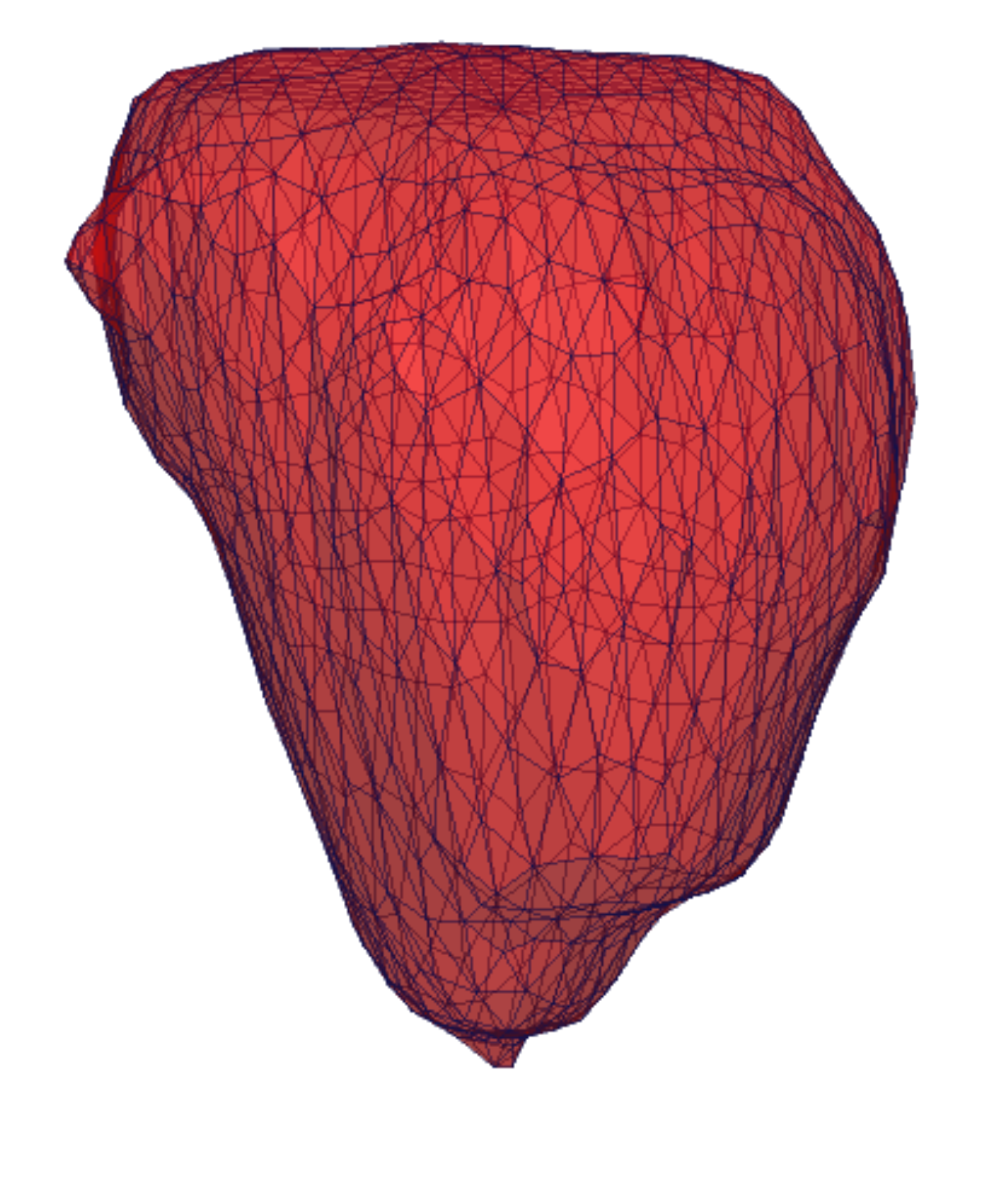}}
\subfigure{
 \includegraphics[width=0.13\linewidth]{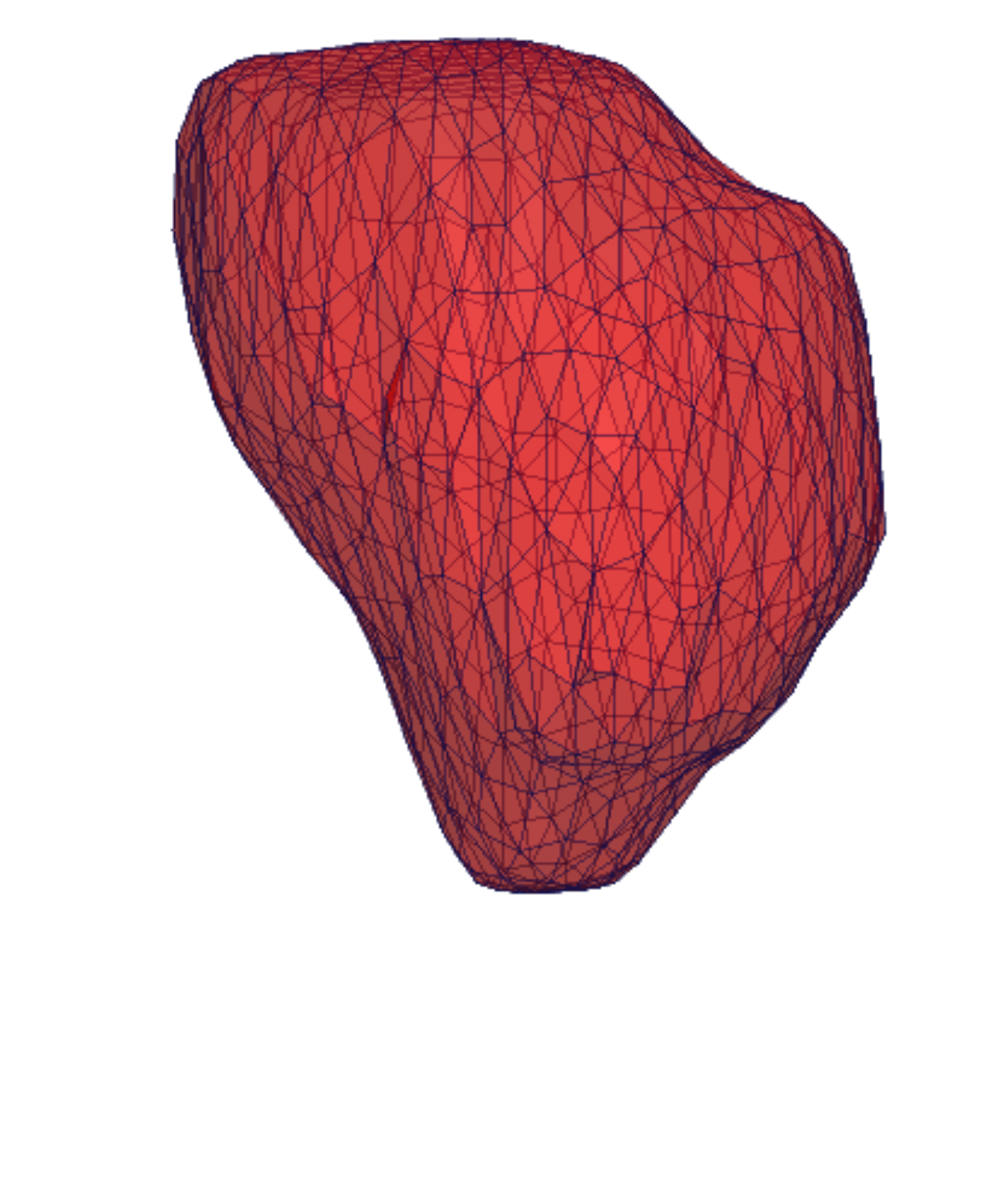}}
\subfigure{
 \includegraphics[width=0.13\linewidth]{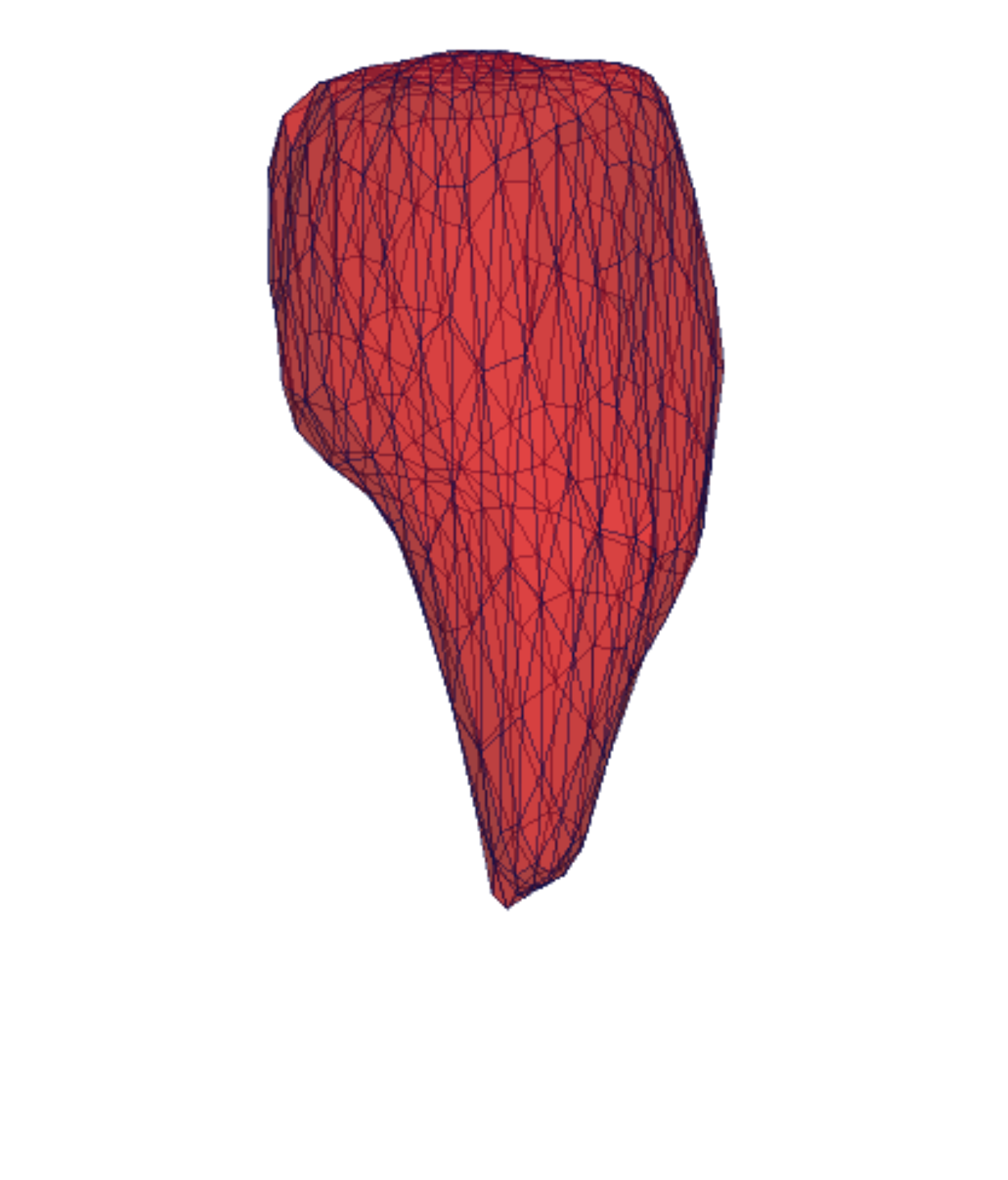}}
\subfigure{
 \includegraphics[width=0.13\linewidth]{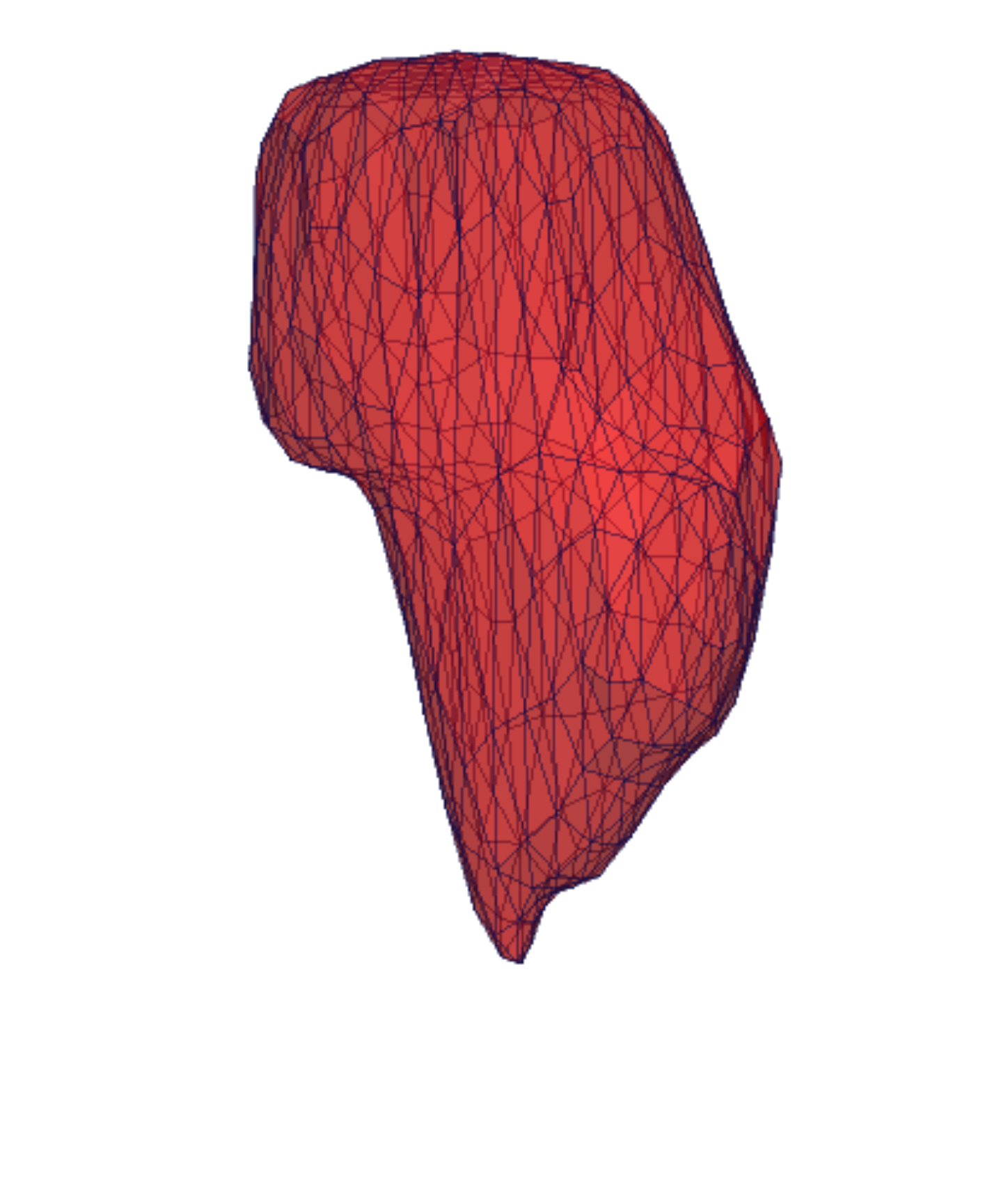}}
\subfigure{
 \includegraphics[width=0.13\linewidth]{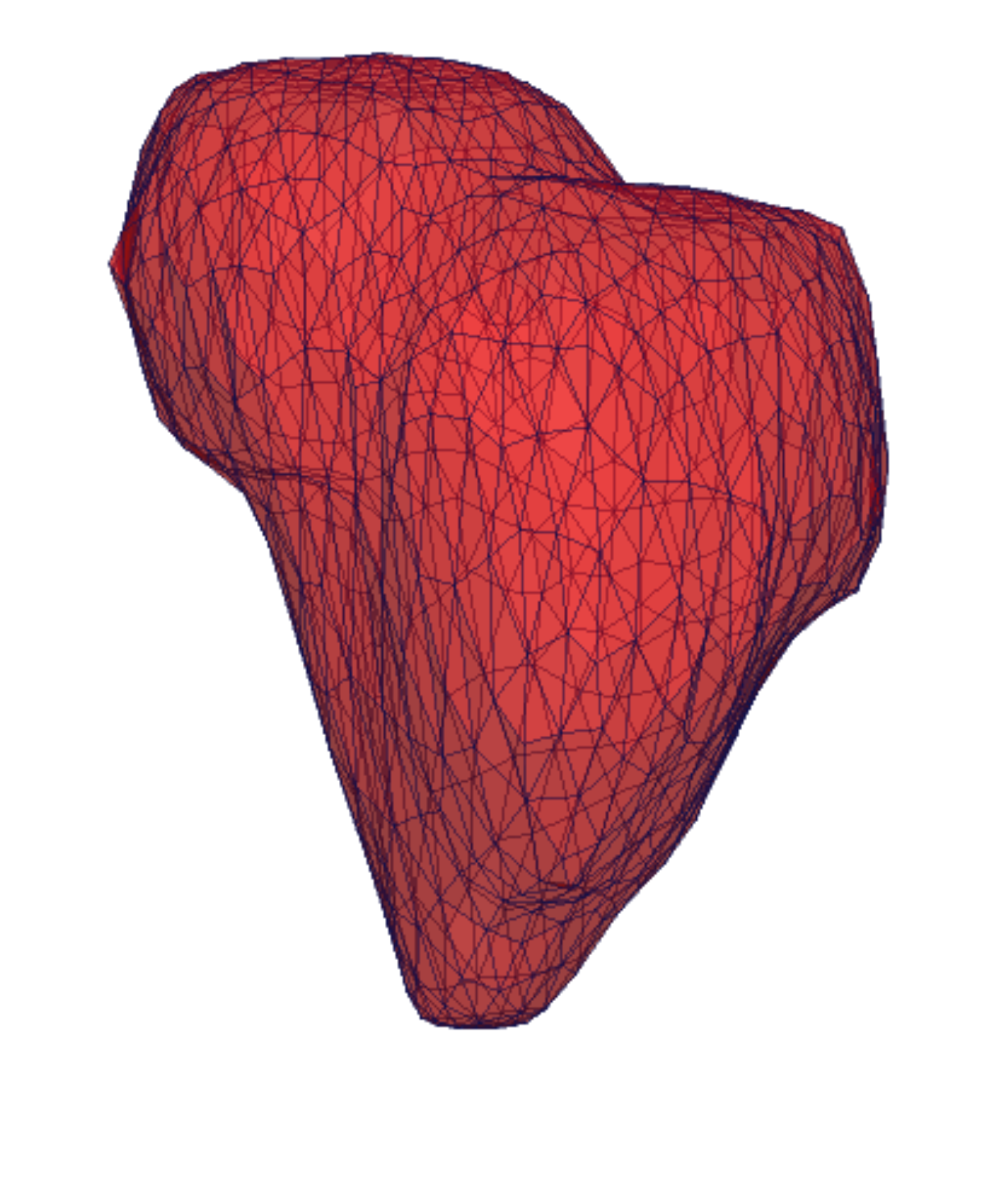}}
\subfigure{
 \includegraphics[width=0.70\linewidth]{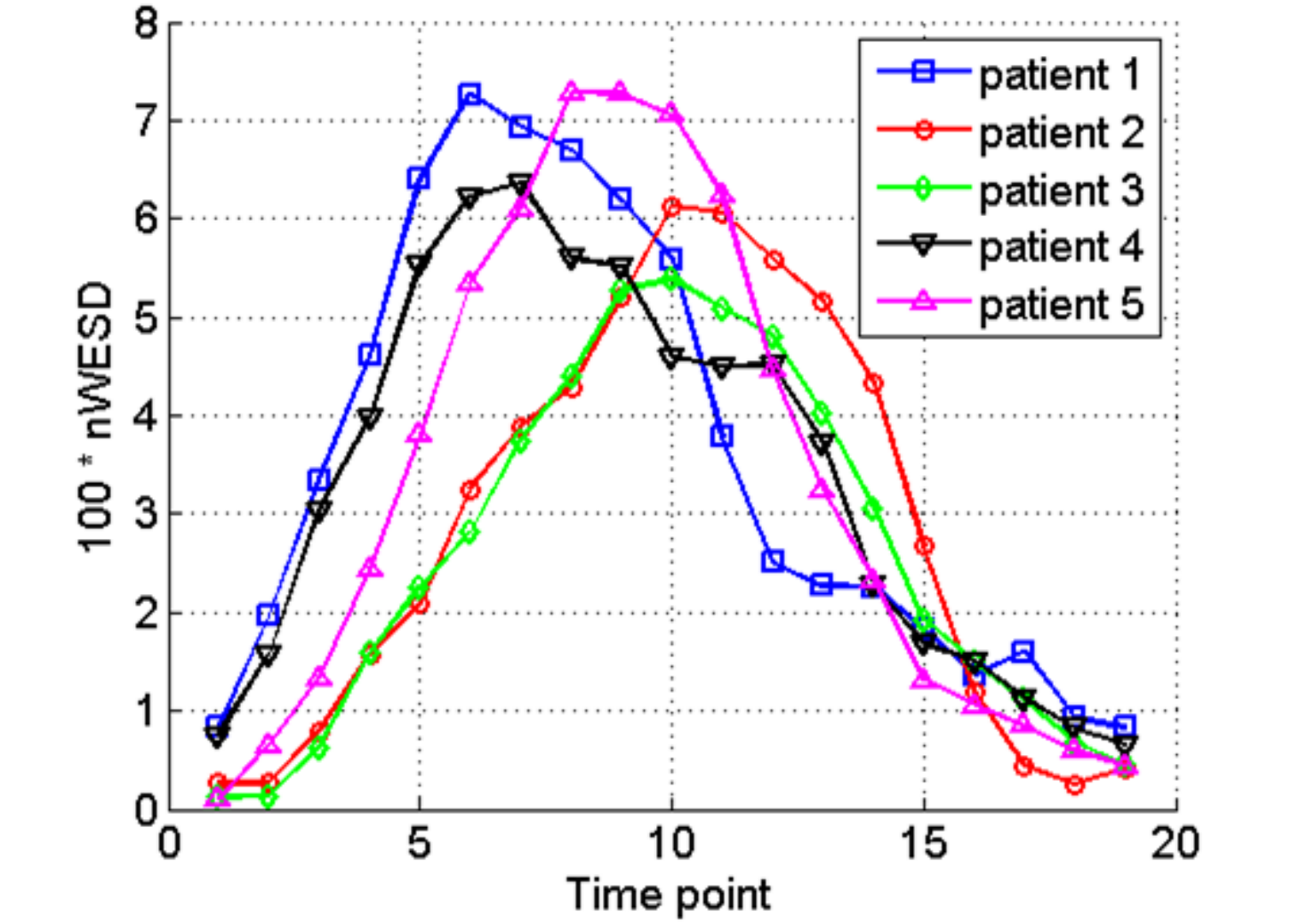}}
\caption{\label{fig:cardiac}Analysing 3D + time (4D) cardiac images: Top row shows corresponding 2D slices of a 4D MRI dataset at time points $t=\{0,3,6,9,12\}$. In the middle row, 3D shapes extracted at each of the time points. For five patients, we compute the nWESD shape dissimilarity score of the LV blood pool at each time point with respect to its shape at $t = 0$. The graph plots these scores. We note that the proposed shape distance is able to capture the dynamic process of the LV shape changes and furthermore, the symmetry between the two phases of an heart beat: diastole and systole.}
\end{figure}
\section{Conclusion}\label{sec:conc}
This article proposed WESD, a new spectral shape distance defined over the eigenvalues of the Laplace operator. WESD is a theoretically sound shape metric that is derived from the heat-trace. The theoretical analysis given in this article presented and proved the properties of WESD related to its existence, computability and multi-scale aspect. The presented experiments showed that the theoretical properties of WESD have many practical advantages over previous works. These experiments further highlighted that WESD is beneficial for various applications. 
\begin{appendices}
\section{Bounds on the Laplace spectrum}\label{sec:bounds}
Li and Yau in~\cite{li_1983} proved that the Laplace spectrum has the following universal lower bound 
\begin{equation}
\lambda_n \geq \frac{d}{d+2}4\pi^2\left(\frac{n}{B_dV}\right)^{2/d} ~~~\forall n > 0.
\end{equation}
We notice that this lower bound does not depend on the shape. 
\subsubsection{Upper bounds}\label{sec:upp_bound}
Several authors have investigated the upper bounds and the relative growth rate of the eigenvalues of the Dirichlet spectrum~\cite{Protter1987}. In~\cite{yang_1991}, Yang provides an upper bound for the growth rate of the components for $n\geq 1$ as
\begin{equation}
\lambda_{n+1} < \left[ 1 + \frac{4}{d}\right]\frac{1}{n}\sum_{m=1}^n \lambda_m.
\end{equation}This equation can be transformed into a sequence of upper bounds by only knowing the first eigenvalue $\lambda_1$. Although sharp for the first few eigenvalues, the upper bound is too relaxed for the remaining modes. Cheng and Yang in~\cite{cheng_2007} provides a much sharper upper bound for larger values of $n$ and is valid for $n\geq 2$. 
\begin{eqnarray}
\lambda_{n+1} \leq C_0(d,n)n^\frac{2}{d}\lambda_1,
\end{eqnarray}where
\begin{eqnarray}
\nonumber C_0(d,n) = 1 + \frac{a(\min(d,n-1))}{d}\\
\nonumber a(1) \leq 2.64\ a(2)\leq2.27\\ 
\nonumber \textrm{and}\ a(p) = 2.2 - 4 \log(1 + \frac{p-3}{50})\ \textrm{for}\ p\geq 3,
\end{eqnarray}where the bound only depends on the first eigenvalue and furthermore it is consistent with Weyl's asymptotic growth law. 
%----------------------------------%
\section{Proofs of Lemmas and Corollaries}\label{sec:proofs}
Before presenting the proofs for the corollaries let us provide a lemma that will be useful throughout this section. 
\begin{lemma}\label{lemma}
For any $a,b\in\mathbb{R}$ such that $a>b>0$ the function $f(a,b) = \frac{a - b}{ab}$ increases monotonously with increasing $a$ and decreases monotonously with increasing $b$.
\end{lemma}
\begin{proof}Since $f$ is differentiable it suffices to look at its partial derivatives $\frac{\partial f} {\partial a} = \frac{1}{ a^2}$ and $\frac{\partial f}{\partial b} = - \frac{1}{b^2}$.\end{proof}
%----------------%
\begin{repcorollary}{corr:conv}
Let $\Omega_\lambda\subset\mathbb{R}^d$ and $\Omega_\xi\subset\mathbb{R}^d$ be any two closed domains with piecewise smooth boundaries and $\{\lambda\}_{n=1}^\infty$ and $\{\xi\}_{n=1}^\infty$ be their Laplace spectrum. Then the weighted spectral distance
\begin{equation}
\nonumber\rho(\Omega_\lambda, \Omega_\xi) = \left[ \sum_{n=1}^\infty  \left(\frac{|\lambda_n - \xi_n|}{\lambda_n \xi_n}\right)^p\right]^{1/p}
\end{equation}converges for $p>\frac{d}{2}$. Furthermore, 
\begin{equation}
\rho(\Omega_\lambda, \Omega_\xi) < \left\{ C + K\cdot\left[ \zeta\left(\frac{2p}{d}\right) - 1 - \left(\frac{1}{2}\right)^\frac{2p}{d}\right] \right\}^\frac{1}{p},
\end{equation}where $\zeta(\cdot)$ is the Riemann zeta function and the coefficients $C$ and $K$ are given as
\begin{eqnarray}
\nonumber C &\triangleq& \sum_{i=1,2} \left[\frac{d+2}{d\cdot4\pi^2}\cdot\left(\frac{B_d\hat{V}}{i}\right)^\frac{2}{d} - \frac{1}{\mu}\cdot\left(\frac{d}{d+4}\right)^{i-1}\right]^p\\
\nonumber K &\triangleq& \left[ \frac{d+2}{d\cdot4\pi^2}\cdot\left(B_d\hat{V}\right)^\frac{2}{d} - \frac{1}{\mu}\cdot\frac{d}{d + 2.64}\right]^p\\
\nonumber \hat{V} &\triangleq& \max(V(\Omega_\lambda),V(\Omega_\xi)),\ \ \mu\triangleq\max(\lambda_1,\xi_1),
\end{eqnarray}
where $V(\cdot)$ denotes the volume (or area in 2D) of an object.
\end{repcorollary}
\begin{proof}
The following inequality results from combining the bounds specified in Section~\ref{sec:bounds} with Lemma~\ref{lemma} 
\begin{eqnarray}
\nonumber\frac{|\lambda_n - \xi_n|}{\lambda_n\xi_n} < \frac{d+2}{d\cdot4\pi^2}\cdot\left(\frac{B_d\hat{V}}{n}\right)^\frac{2}{d}-\frac{1}{\mu}\cdot\left(\frac{d}{d + 4}\right)^{n-1}
\end{eqnarray}for $n=1,2$ and for $n \geq 3$
\begin{eqnarray}
\nonumber\frac{|\lambda_n - \xi_n|}{\lambda_n\xi_n} &<& \frac{d+2}{d\cdot4\pi^2}\cdot\left(\frac{B_d\hat{V}}{n}\right)^\frac{2}{d} - \frac{1}{\mu}\cdot\frac{1}{C_0(d,n)n^\frac{2}{d}}\\
\nonumber&\leq&\frac{d+2}{d\cdot4\pi^2}\cdot\left(\frac{B_d\hat{V}}{n}\right)^\frac{2}{d} - \frac{1}{\mu}\cdot\frac{d}{(d + 2.64)n^\frac{2}{d}},
\end{eqnarray}Based on this component-wise bound we can write the infinite sum without the first two terms as
\begin{equation}
\nonumber\sum_{n=3}^\infty\left(\frac{|\lambda_n - \xi_n|}{\lambda_n\xi_n}\right)^p < K\sum_{n=3}^\infty \left(\frac{1}{n}\right)^\frac{2p}{d},
\end{equation}which for $p>\frac{d}{2}$ converges to 
\begin{equation}
\nonumber K\sum_{n=3}^\infty \left(\frac{1}{n}\right)^\frac{2p}{d} = \zeta\left(\frac{2p}{d}\right) - 1 - \left(\frac{1}{2}\right)^\frac{2p}{d}
\end{equation} and diverges for $p\leq \frac{2}{d}$. Consequently, $\rho(\Omega_\lambda, \Omega_\xi)$ converges for $p>\frac{d}{2}$. Furthermore, extending the sum with the upper bounds for $n=1,2$ the following upper bound for the distance between $\Omega_\lambda$ and $\Omega_\xi$ holds
\begin{equation}
\nonumber\rho(\Omega_\lambda, \Omega_\xi) < \left\{ C + K\cdot\left[ \zeta\left(\frac{2p}{d}\right) - 1 - \left(\frac{1}{2}\right)^\frac{2p}{d}\right] \right\}^\frac{1}{p}.\end{equation}
\end{proof}
%-------------------------%
\begin{repcorollary}{corr:pseudometric}
$\rho(\Omega_\lambda,\Omega_\xi)$ is a pseudometric for $d\geq2$.
\end{repcorollary}
To ease notation we define
\begin{equation}\nonumber\varrho_n(\Omega_\lambda,\Omega_\xi) \triangleq \frac{|\lambda_n - \xi_n|}{\lambda_n\xi_n}.\end{equation} This leads to \begin{equation}\nonumber\rho(\Omega_\lambda,\Omega_\xi) = \left[\sum_{n=1}^\infty\varrho^p_n(\Omega_\lambda,\Omega_\xi)\right]^\frac{1}{p}.\end{equation}
\begin{proof}
$\forall \Omega_\lambda\subset\mathbb{R}^d,\ \Omega_\xi\subset\mathbb{R}^d$
The first three points for this proof are trivial:
\begin{enumerate}
\item[-] $\rho(\Omega_\lambda,\Omega_\xi) > 0$ since $\varrho_n(\Omega_\lambda,\Omega_\xi) >0\ \forall n$.
\item[-] $\rho(\Omega_\lambda, \Omega_\lambda) = 0$ since $|\lambda_n - \lambda_n| = 0\ \forall n$.
\item[-] $\rho(\Omega_\lambda, \Omega_\xi) = \rho(\Omega_\xi,\Omega_\lambda)$ since $|\lambda_n-\xi_n| = |\xi_n - \lambda_n|\ \forall n$
\end{enumerate}
In order to prove the triangle inequality let us proceed with the case $\lambda_n \geq \xi_n$. The inverse case follows exactly the same way. Now, let $\Omega_\eta\subset\mathbb{R}^d$ be an arbitrary closed domain with piecewise smooth boundaries whose spectrum is given as $\{\eta_n\}_{n=1}^\infty$. Investigating \begin{equation}\nonumber \frac{|\lambda_n - \eta_n|}{\lambda_n\eta_n} + \frac{|\eta_n - \xi_n|}{\eta_n\xi_n}\end{equation} we notice that for each $n$ there are only three possible cases: 
\begin{enumerate}
\item $\lambda_n \geq \eta_n \geq \xi_n$, for which 
\begin{equation*}
\varrho_n(\Omega_\lambda,\Omega_\eta) + \varrho_n(\Omega_\eta,\Omega_\xi) = \varrho_n(\Omega_\lambda,\Omega_\xi),
\end{equation*} 
\item $\lambda_n \geq \xi_n \geq \eta_n$, for which $\varrho_n(\Omega_\lambda,\Omega_\eta) \geq \varrho_n(\Omega_\lambda,\Omega_\xi)$ as a result of the Lemma~\ref{lemma}. Due  $\varrho_n(\Omega_\xi,\Omega_\eta) \geq 0$: 
\begin{equation*}
\varrho_n(\Omega_\lambda,\Omega_\eta) + \varrho_n(\Omega_\eta,\Omega_\xi) \geq \varrho_n(\Omega_\lambda,\Omega_\xi)
\end{equation*}
\item $\eta_n \geq\lambda_n \geq \xi_n$, for which $\varrho_n(\Omega_\eta,\Omega_\xi) \geq \varrho_n(\Omega_\lambda,\Omega_\xi)$ as a result of the Lemma~\ref{lemma} once again. And as in the previous case, due to $\varrho_n(\Omega_\eta,\Omega_\lambda)\geq 0$ we have
\begin{equation*}
\varrho_n(\Omega_\lambda,\Omega_\eta) + \varrho_n(\Omega_\eta,\Omega_\xi) \geq \varrho_n(\Omega_\lambda,\Omega_\xi). 
\end{equation*}
\end{enumerate}
Thus $\forall n\ \varrho_n(\Omega_\lambda,\Omega_\eta) + \varrho_n(\Omega_\eta,\Omega_\xi) \geq \varrho_n(\Omega_\lambda,\Omega_\xi)$.  Since \mbox{$p > 1$} as $p>\frac{d}{2}$ for $d\geq2$ the Minkowski Inequality states
\begin{eqnarray}
\nonumber\rho(\Omega_\lambda,\Omega_\eta) &+& \rho(\Omega_\eta,\Omega_\xi) \\
\nonumber &=&\left[\sum_{n=1}^\infty\varrho^p_n(\Omega_\lambda,\Omega_\eta)\right]^\frac{1}{p} + \left[\sum_{n=1}^\infty\varrho^p_n(\Omega_\eta,\Omega_\xi)\right]^\frac{1}{p}\\
\nonumber &\geq& \left[ \sum_{n=1}^\infty \left(\varrho_n(\Omega_\lambda,\Omega_\eta) + \varrho_n(\Omega_\eta,\Omega_\xi)\right)^p\right]^{1/p}.
\end{eqnarray}When combined with the previous results, the outcome is the triangle inequality:
\begin{eqnarray}
\nonumber \nonumber\rho(\Omega_\lambda,\Omega_\eta) + \rho(\Omega_\eta,\Omega_\xi) \geq \rho(\Omega_\lambda,\Omega_\xi)
\end{eqnarray}
\end{proof}
%-----------------------------%
\begin{replemma}{lemma:influenceratio}
Let $\Omega_\lambda\mathbb{R}^d$ represent an object with piecewise smooth boundary and $\mathcal{D}(l,t)\triangleq \frac{e^{-\lambda_l t}}{Z(t)}$ be the corresponding influence ratio with respect to mode $l$ and $t$. Then for any two spectral indices $m>n>0$
\begin{equation}\nonumber\mathcal{D}(n,t) > \mathcal{D}(m,t),\ \ \forall t>0\end{equation}
and particularly for two $t$ values such that $t_1 > t_2$
\begin{equation} \nonumber \frac{\mathcal{D}(m,t_1)}{\mathcal{D}(n,t_1)} < \frac{\mathcal{D}(m,t_2)}{\mathcal{D}(n,t_2)}.\end{equation}\end{replemma}
\begin{proof}
The proof follows the properties of the exponential function and the properties of the spectrum of the Laplace operator. For $n<m$ we know that $\lambda_n < \lambda_m$ which leads to \mbox{$e^{-\lambda_n t} > e^{-\lambda_m t}\ \forall t>0$.} Since the denominators are the same for both  $\mathcal{D}(n,t)$ and $\mathcal{D}(m,t)$ then  
\begin{equation*}
 \mathcal{D}(n,t) > \mathcal{D}(m,t)\ \forall t>0.
\end{equation*} 

For the second part of the lemma, we first compute the ratio
\begin{equation}\nonumber\frac{\mathcal{D}(m,t)}{\mathcal{D}(n,t)} = e^{-(\lambda_m - \lambda_n)t}.\end{equation}Now based on $\lambda_m > \lambda_n$ and $e^{-(\lambda_m - \lambda_n)t}$ is monotonously decreasing with increasing $t$, it follows for $t_1 > t_2$ that 
\begin{equation*}
\frac{\mathcal{D}(m,t_1)}{\mathcal{D}(n,t_1)} = e^{-(\lambda_m - \lambda_n)t_1} < e^{-(\lambda_m - \lambda_n)t_2} = \frac{\mathcal{D}(m,t_2)}{\mathcal{D}(n,t_2)}. 
\end{equation*}
\end{proof}
%------------------------------%
\begin{repcorollary}{corr:multires}
Let $\Omega_\lambda$ and
  $\Omega_\xi$ be two objects with piecewise smooth boundaries. Then
  for any two scalars with $p>d/2$, $q>d/2$, $p \geq q$ and for all $n$ with $|\lambda_n - \xi_n|>0$ there exists a $M>n$ so that $\forall m\geq
  M$
  \begin{equation}
    \nonumber \frac{\left(\frac{|\lambda_m - \xi_m|}{\lambda_m\xi_m}\right)^p}{\left(\frac{|\lambda_n - \xi_n|}{\lambda_n\xi_n}\right)^p} \leq \frac{\left(\frac{|\lambda_m - \xi_m|}{\lambda_m\xi_m}\right)^q}{\left(\frac{|\lambda_n - \xi_n|}{\lambda_n\xi_n}\right)^q}
  \end{equation}
\end{repcorollary}
\begin{proof}
From Corollary~\ref{corr:conv} we know that the series 
\begin{equation}
\nonumber\sum_{m=1}^\infty \left(\frac{|\lambda_m - \xi_m|}{\lambda_m\xi_m}\right)^q
\end{equation}converges. Then based on Cauchy's convergence criterion for series
\begin{equation}
\nonumber\lim_{n\rightarrow\infty}\left(\frac{|\lambda_m - \xi_m|}{\lambda_m\xi_m}\right)^q=0.\end{equation}In other words, $\forall \epsilon>0$ there exists a $M$ such that
\begin{equation}
\nonumber\left(\frac{|\lambda_m - \xi_m|}{\lambda_m\xi_m}\right)^q < \epsilon,\ \forall m>M.
\end{equation}Let $n$ be an arbitrary index such that $|\lambda_n - \xi_n|>0$. Consequently, also for $|\lambda_n - \xi_n|$, there exists a $M$ such that $\forall m>M$
\begin{equation}
\nonumber\frac{\left(\frac{|\lambda_m - \xi_m|}{\lambda_m\xi_m}\right)^q}{\left(\frac{|\lambda_n - \xi_n|}{\lambda_n\xi_n}\right)^q} <1.
\end{equation}Since $p\geq q$ we can find a $k\geq 1$ such that $p = kq$. Then based on the above inequality $\forall m>M$
\begin{eqnarray}
\nonumber\frac{\left(\frac{|\lambda_m - \xi_m|}{\lambda_m\xi_m}\right)^p}{\left(\frac{|\lambda_n - \xi_n|}{\lambda_n\xi_n}\right)^p} &=& \left[\frac{\left(\frac{|\lambda_m - \xi_m|}{\lambda_m\xi_m}\right)^q}{\left(\frac{|\lambda_n - \xi_n|}{\lambda_n\xi_n}\right)^q}\right]^k \\
\nonumber&\leq& \frac{\left(\frac{|\lambda_m - \xi_m|}{\lambda_m\xi_m}\right)^q}{\left(\frac{|\lambda_n - \xi_n|}{\lambda_n\xi_n}\right)^q}
\end{eqnarray}
\end{proof}
%-------------------------------$
\begin{repcorollary}{corr:trunc}
Let $\rho^N(\Omega_\lambda,\Omega_\xi)$ and $\overline{\rho}^N(\Omega_\lambda, \Omega_\xi)$ be the truncated approximations of $\rho(\Omega_\lambda,\Omega_\xi)$ and $\overline{\rho}(\Omega_\lambda,\Omega_\xi)$ respectively, using the first $N$ modes. Then 
\begin{equation}\nonumber\lim_{N\rightarrow\infty}|\rho -\rho^N| = 0\end{equation}and
\begin{equation}\nonumber\lim_{N\rightarrow\infty}|\overline{\rho} - \overline{\rho}^N| = 0.\end{equation} Furthermore, for a given $N\geq 3$ the truncation errors $|\rho-\rho^N|$ and $|\overline{\rho}-\overline{\rho}^N|$ can be bounded by 
\begin{eqnarray}
\nonumber\left|\rho - \rho^N\right| &<& \left\{C + K\cdot\left[ \zeta\left(\frac{2p}{d}\right) - 1 - \left(\frac{1}{2}\right)^\frac{2p}{d}\right]\right\}^\frac{1}{p}\\
\nonumber & & - \left\{C + K\cdot\left[ \sum_{n=3}^N \left(\frac{1}{n}\right)^\frac{2p}{d} \right]\right\}^\frac{1}{p}\\
\nonumber\left|\overline{\rho} - \overline{\rho}_N\right| &<& 1 - \left\{\frac{ C + K\cdot\left[ \sum_{n=3}^N \left(\frac{1}{n}\right)^\frac{2p}{d} \right] }{C + K\cdot\left[ \zeta\left(\frac{2p}{d}\right) - 1 - \left(\frac{1}{2}\right)^\frac{2p}{d}\right]}\right\}^\frac{1}{p}
\end{eqnarray}
\end{repcorollary}
\begin{proof}
As before, to ease notation, let us again define 
\begin{equation}
\varrho_n \triangleq \frac{|\lambda_n - \xi_n|}{\lambda_n\xi_n}.
\end{equation} Then based on Corollary~\ref{corr:conv} we know that the sum $\sum_{n=1}^{\infty}\varrho_n^p$ exists and thus also the partial sums converge
\begin{eqnarray}
\nonumber \lim_{N\rightarrow\infty}\left|\sum_{n=1}^\infty\varrho_n^p - \sum_{n=1}^N\varrho_n^p\right|
\nonumber = \lim_{N\rightarrow\infty}\sum_{n=1}^\infty\varrho_n^p - \sum_{n=1}^N\varrho_n^p = 0.
\end{eqnarray}
Based on  
\begin{equation*}
\begin{split}
&\forall a,b,d\in\mathbb{R} \text{ with } a,b,d \geq 0 \text{ and } a^d - b^d \rightarrow 0 \Rightarrow a-b \rightarrow 0 \\
&\text{and } \\
& \varrho_n \succeq 0 \\
& \text{we reach} \\
& \lim_{N\rightarrow\infty} \left|\rho - \rho_N\right| = \lim_{N\rightarrow\infty}\left[\sum_{n=1}^\infty\varrho_n^p\right]^\frac{1}{p} - \left[\sum_{n=1}^N\varrho_n^p\right]^\frac{1}{p} = 0 
\end{split}
\end{equation*}
As the denominators for both $\overline{\rho}_N$ and $\overline{\rho}$ are the same, the above limit also yields $\lim_{N\rightarrow\infty}|\overline{\rho} - \overline{\rho}_N| = 0$. 

The upper bounds for the truncation errors now is a direct result of Corollary~\ref{corr:conv} as
\begin{eqnarray}
\nonumber \left|\rho - \rho^N\right| &=& \left[\sum_{n=1}^\infty \varrho_n^p\right]^\frac{1}{p} - \left[\sum_{n=1}^N\varrho_n^p\right]^\frac{1}{p}\\
\nonumber &<& \left\{C + K\cdot\left[ \zeta\left(\frac{2p}{d}\right) - 1 - \left(\frac{1}{2}\right)^\frac{2p}{d}\right]\right\}^\frac{1}{p}\\
\nonumber & & - \left\{C + K\cdot\left[ \sum_{n=3}^N \left(\frac{1}{n}\right)^\frac{2p}{d} \right]\right\}^\frac{1}{p}\\
\nonumber\left|\overline{\rho} - \overline{\rho}_N\right| &=& \frac{  \left[\sum_{n=1}^\infty\varrho^p\right]^{1/p} - \left[\sum_{n=1}^N\varrho^p\right]^{1/p}}{ \left\{ C + K\left( \zeta(2p/d) -1 - 1/2^{2p/d}\right)\right\}^{1/p} }\\
\nonumber &<& 1 - \left\{\frac{ C + K\cdot\left[ \sum_{n=3}^N \left(\frac{1}{n}\right)^\frac{2p}{d} \right] }{C + K\cdot\left[ \zeta\left(\frac{2p}{d}\right) - 1 - \left(\frac{1}{2}\right)^\frac{2p}{d}\right]}\right\}^\frac{1}{p}
\end{eqnarray}
\end{proof}
\begin{proposition}\label{prop:residual}
Let $x,y\in\mathbb{R}$ be positive real values such that $y>x$. Then $\forall A,B \in \mathbb{R}$, $A,B>0$
\begin{equation}
\nonumber\frac{A+Bx}{A+By} > \frac{x}{y}.
\end{equation}
\end{proposition}
\begin{proof}
Since $A,x,y>0$, we can find two positive real values $k_1>0$ and $k_2>0$ such that $A = k_1x$ and $A = k_2y$. Furthermore, $y>x$ simply implies $k_1 > k_2$. Using $k_1$ and $k_2$ now we can write
\begin{equation}
\nonumber\frac{A+Bx}{A+By} = \frac{k_1x+Bx}{k_2y+By} > \frac{k_2x+Bx}{k_2y+By} = \frac{x}{y}.
\end{equation}
\end{proof}
\end{appendices}
\bibliographystyle{plain}
\bibliography{cites}
\end{document}